%% file: main.tex
\documentclass{article}

\usepackage[nonatbib, final]{neurips_2022}

\usepackage[utf8]{inputenc} 
\usepackage[T1]{fontenc}    
\usepackage{url}            
\usepackage{booktabs}       
\usepackage{amsfonts}       
\usepackage{nicefrac}       
\usepackage{microtype}      
\usepackage{xcolor}         

\usepackage[normalem]{ulem}

\usepackage{microtype}
\usepackage{graphicx}
\usepackage{subcaption} 
\usepackage{booktabs} 

\usepackage{fancyhdr}

\usepackage{algorithm}
\usepackage[noend]{algorithmic}
\usepackage[numbers]{natbib}
\usepackage{eso-pic} 
\usepackage{forloop}
\usepackage{url}

\usepackage{amsmath}
\usepackage{amssymb}
\usepackage{mathtools}
\usepackage{amsthm}
\usepackage{wrapfig}

\usepackage{setspace}

\usepackage{amsthm}

\usepackage{thmtools}
\usepackage{thm-restate}

\usepackage{url}

\usepackage{color}

\usepackage{colortbl}

\usepackage{hyperref}       
\usepackage[capitalize,noabbrev]{cleveref}

\theoremstyle{plain}
\newtheorem{theorem}{Theorem}[section]

\newtheorem{lemma}[theorem]{Lemma}

\theoremstyle{definition}

\theoremstyle{remark}

\newcommand{\eg}{\textit{e.g., }}
\newcommand{\ie}{\textit{i.e., }}

\creflabelformat{equation}{#2\textup{(\textcolor{red}{#1})}#3}
\creflabelformat{theorem}{#2\textcolor{red}{#1}#3}
\creflabelformat{table}{#2\textcolor{red}{#1}#3}
\creflabelformat{figure}{#2\textcolor{red}{#1}#3}
\creflabelformat{section}{#2\textcolor{red}{#1}#3}
\creflabelformat{appendix}{#2\textcolor{red}{#1}#3}
\creflabelformat{algorithm}{#2\textcolor{red}{#1}#3}

\usepackage[colorinlistoftodos,bordercolor=orange,backgroundcolor=orange!20,linecolor=orange,textsize=scriptsize]{todonotes}

\definecolor{DarkCoral}{rgb}{0.8, 0.36, 0.27}

\title{Make Some Noise: Reliable and Efficient Single-Step Adversarial Training}

\author{
  Pau de Jorge \thanks{Correspondence to \texttt{pau@robots.ox.ac.uk}} \\
  University of Oxford \\
  NAVER LABS Europe \\
  \And
  Adel Bibi \\
  University of Oxford \\
  \And
  Riccardo Volpi \\
  NAVER LABS Europe \\
  \And
  Amartya Sanyal \\
  ETH Z\"urich \\
  ETH AI Center \\
  \And
  Philip H. S. Torr \\
  University of Oxford \\
  \And
  Gr\'egory Rogez \\
  NAVER LABS Europe \\
  \And
  Puneet K. Dokania \\
  University of Oxford \\
  Five AI Ltd. \\
}

\begin{document}

\maketitle

\begin{abstract}
Recently, \citet{RS-FGSM} showed that adversarial training with single-step FGSM leads to a characteristic failure mode named \textit{Catastrophic Overfitting} (CO), in which a model becomes suddenly vulnerable to multi-step attacks. Experimentally they showed that simply adding a random perturbation prior to FGSM (RS-FGSM) could prevent CO. However, \citet{grad_align} observed that RS-FGSM still leads to CO for larger perturbations, and proposed a computationally expensive regularizer (GradAlign) to avoid it.
In this work, we methodically revisit the role of noise and clipping in
single-step adversarial training. Contrary to previous intuitions, we find
that using a \textit{stronger noise} around the clean sample combined with \textit{not clipping} is highly effective in avoiding CO for
large perturbation radii. We then propose \textit{Noise-FGSM} (N-FGSM) that, while providing the benefits of single-step adversarial training, does not suffer from CO. Empirical analyses on
a large suite of experiments show that N-FGSM is able to match or surpass
the performance of previous state-of-the-art GradAlign, while achieving 3$\times$ speed-up. Code can be found in \url{https://github.com/pdejorge/N-FGSM}
\end{abstract}


\section{Introduction}
\label{introduction}

Deep neural networks have achieved remarkable performance on a variety of tasks \citep{he2015delving, go, bert}. However, it is well known that they are vulnerable to small worst-case perturbations around the input data  -- commonly referred to as \textit{adversarial examples}~\citep{szegedy2013intriguing}. The existence of such adversarial examples poses a security threat to deploying models in sensitive environments \citep{biggio2018wild, DBLP:journals/corr/abs-1903-05157}.
This has motivated a large body of work towards improving the \textit{adversarial robustness} of neural networks~\citep{fgsm, papernot2016distillation, tramer2018ensemble, DBLP:journals/corr/abs-1804-07090, parseval}.

The most popular family of 
methods for learning
robust neural networks is based on the concept of \textit{adversarial training}~\citep{fgsm,PGD}. In a nutshell, adversarial training can be posed as a min-max problem where instead of minimizing some loss over a dataset of \textit{clean} samples, we augment the inputs with worst-case perturbations that are generated online during training. However, obtaining such perturbations is NP-hard~\citep{weng2018towards} and hence, 
different \textit{adversarial attacks}
have been suggested that approximate them. In their seminal work,~\citet{fgsm} proposed the {\em Fast Gradient Sign Method}~(FGSM), that generates adversarial attacks by performing a gradient ascent step on the loss function. Yet, while FGSM-based adversarial training provides robustness against single-step FGSM adversaries,~\citet{tramer2018ensemble} showed that these models are still vulnerable to multi-step attacks, namely those allowed to perform multiple gradient ascent steps. 
%
%
Given their better (robust) performance, multi-step
attacks
such as  \textit{Projected Gradient Descent} (PGD) \cite{PGD}
have now become the de facto standard for adversarial training. 

The main downside of multi-step adversarial training
is that the cost of these attacks
increases linearly with the number of steps, making their applicability often computationally prohibitive. 
For this reason,
several works have focused on reducing the cost of adversarial training by approximating the worst-case perturbations with single-step attacks~\citep{RS-FGSM,free, dropout_fgsm}. In particular,~\citet{RS-FGSM} studied FGSM adversarial training and discovered that
it suffers from a characteristic failure mode, in which a model  suddenly becomes vulnerable to multi-step attacks despite remaining robust to single-step attacks. This phenomenon is referred to as \textit{Catastrophic Overfitting} (CO). 
As a solution,
they argued that adding a random perturbation prior to FGSM~(RS-FGSM) seemed sufficient to prevent CO and produce robust models. Yet, \citet{grad_align} recently observed that RS-FGSM still leads to CO as one increases the perturbation radii of the attacks. They suggested a regularizer (GradAlign) that
can avoid CO in 
the settings they considered, but 
at the expense of computing
a double derivative -- significantly increasing the computational cost with respect to RS-FGSM. 

\begin{figure*}
\centering
\begin{subfigure}[b]{.27\textwidth}
\includegraphics[width=\linewidth]{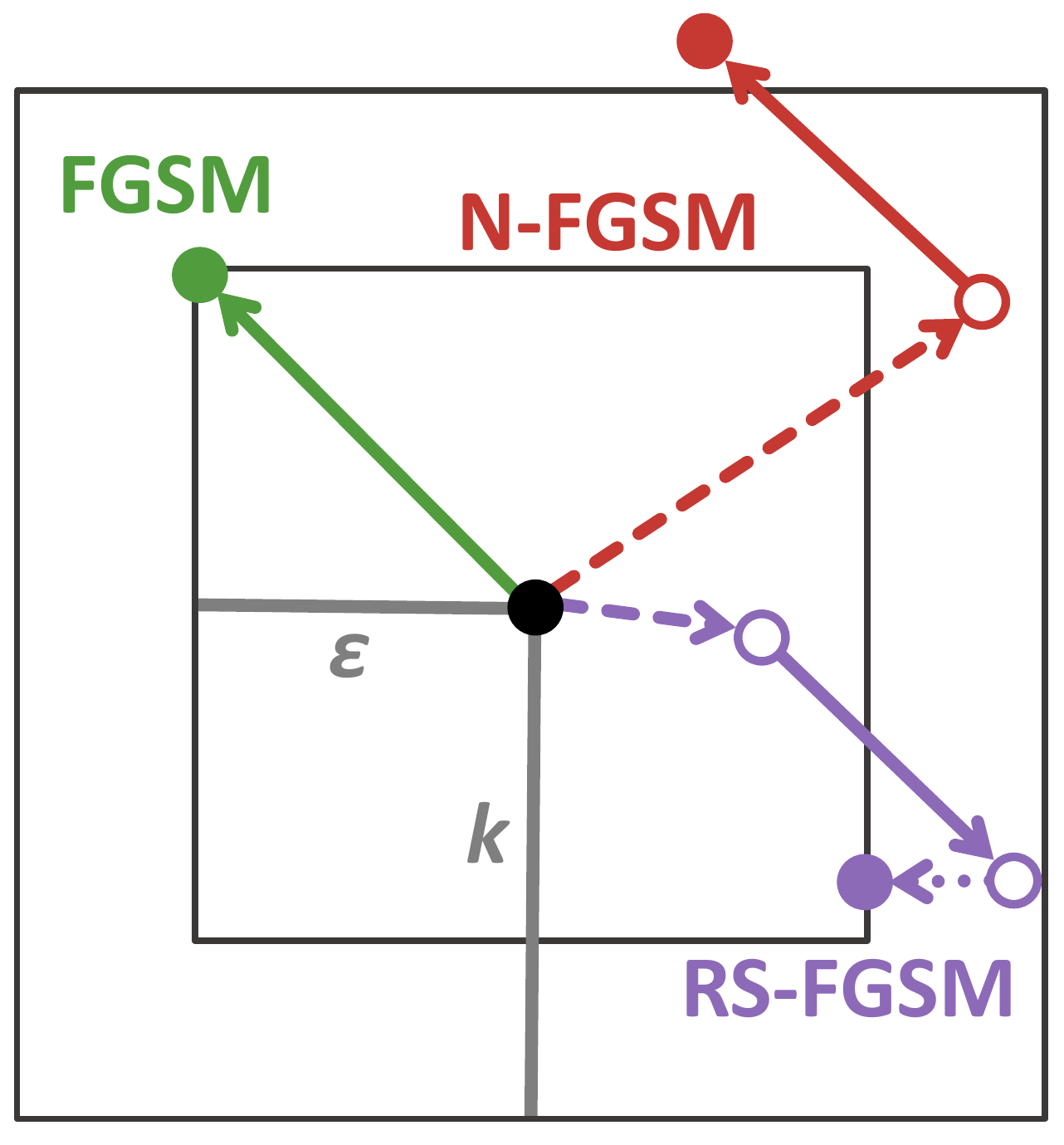}
\end{subfigure}
\hspace{5pt}
\begin{subfigure}[b]{.41\textwidth}
\includegraphics[width=\linewidth]{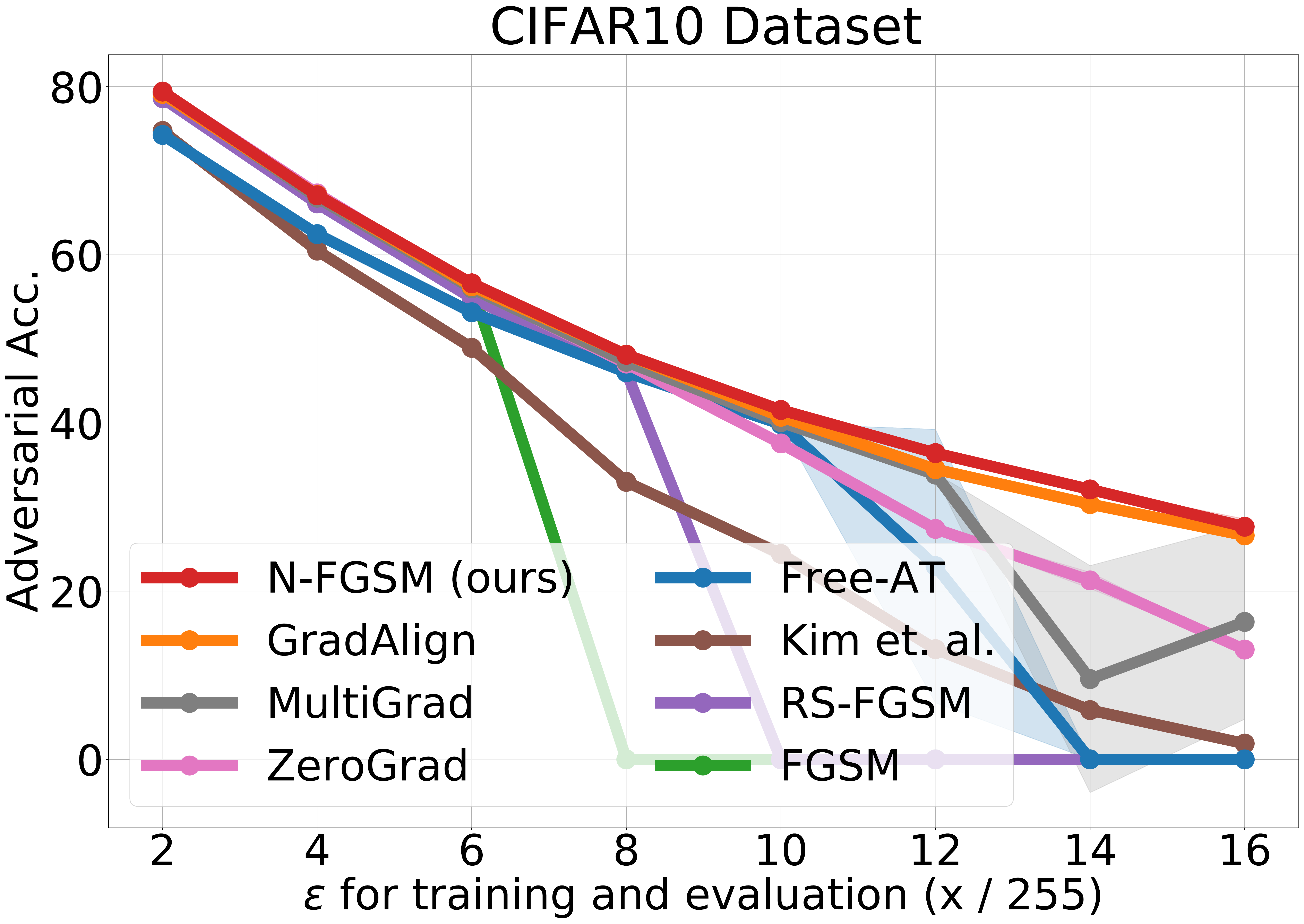}
\end{subfigure}
\hspace{5pt}
\begin{subfigure}[b]{.196\textwidth}
\includegraphics[width=\linewidth]{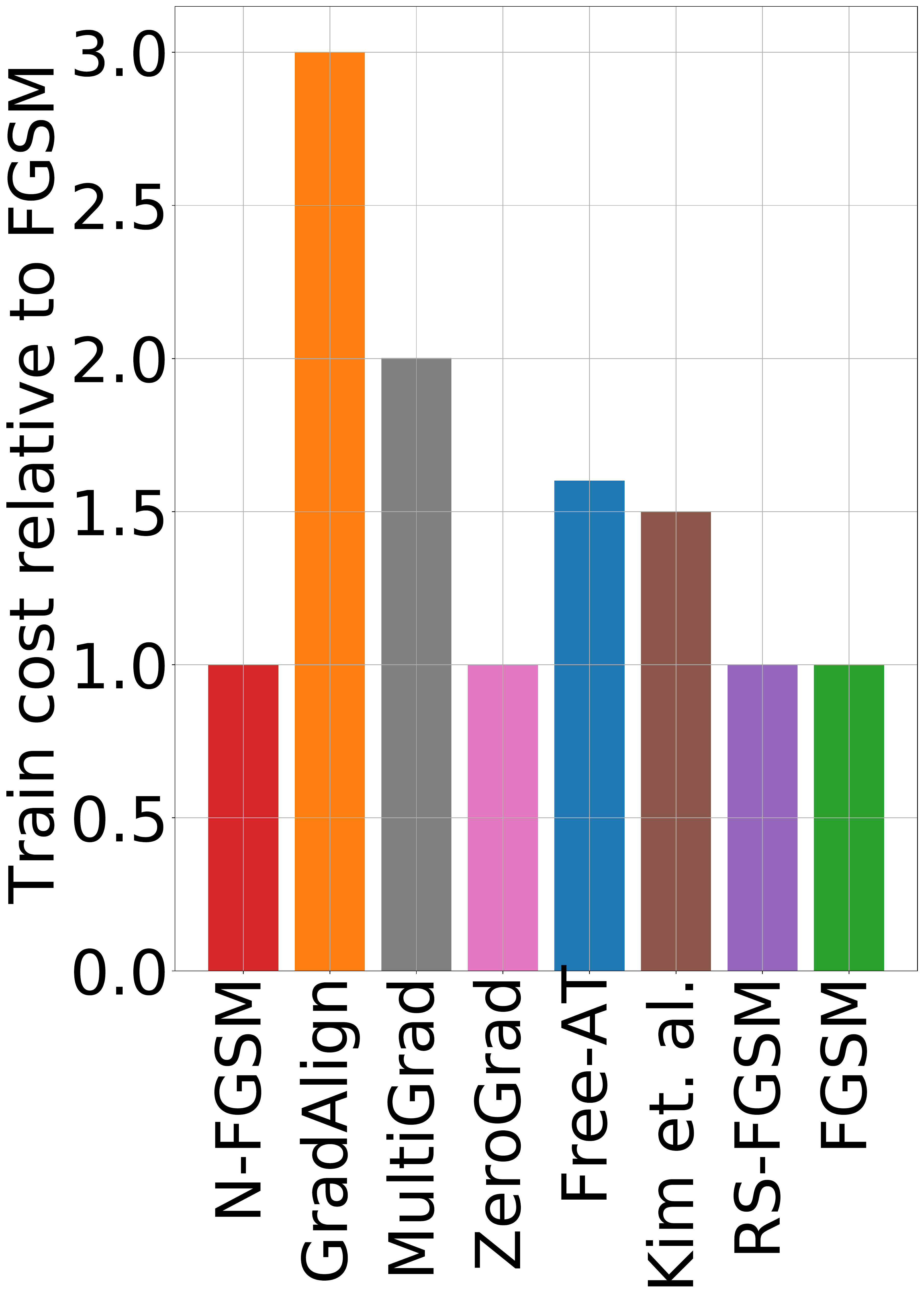}
\end{subfigure}

\caption{\textbf{Left:} Visualization of FGSM~\citep{fgsm}, RS-FGSM~\citep{RS-FGSM} and N-FGSM (ours) attacks. While RS-FGSM is limited to noise in the $\epsilon-l_\infty$ ball, N-FGSM draws noise from an arbitrary $k-l_\infty$ ball. Moreover, N-FGSM does not clip the perturbation around the clean sample. \textbf{Middle:} Comparison of single-step methods on CIFAR-10 with PreactResNet18 over different perturbation radii ($\epsilon$ is divided by 255).  Our method, N-FGSM, can match or surpass state-of-the-art results while \textit{reducing the cost by a $3\times$ factor}. Adversarial accuracy is based on PGD-50-10 and experiments are averaged over $3$ seeds. \textbf{Right}: Comparison of training costs relative to FGSM baseline based on the number of Forward-Backward passes, see~\cref{sec:train_cost} for details.
}
\label{figure:splash}
\end{figure*}

In this paper, we revisit the idea of including noise in single-step attacks. Differently from previous methods that consider the noise as part of the attack, we propose an adversarial training procedure where the noise is used as a form of \textit{data augmentation}. As we 
detail
in~\cref{sec:noise_and_fgsm}, this motivates us to introduce two main changes with respect to previous methods: 1) We center adversarial perturbations with respect to noise-augmented samples and therefore, unlike previous RS-FGSM, we do not clip around the clean samples. 2) We use noise perturbations larger than the $\epsilon-$ball, since they are not restricted by the strength of the attack anymore.
Our experiments 
show that 
performing data augmentation with sufficiently {\em strong noise}
and removing the {\em clipping step} improves model robustness and prevents CO, even against large perturbation radii. 
Our new method, termed N-FGSM, matches, or even surpasses, the robust accuracy of the regularized FGSM introduced by \citet{grad_align} (GradAlign), while \textit{providing a 3$\times$ speed-up}.

To corroborate the effectiveness of our solution, we present an experimental survey of recently proposed single-step attacks and empirically demonstrate that N-FGSM trades-off robustness and computational cost better than any other single-step approach, evaluated over a large spectrum of perturbation radii (see~\cref{figure:splash}, middle and right panels), over several datasets (CIFAR-10, CIFAR-100, and SVHN) and architectures (PreActResNet18 and WideResNet28-10). We will release our code to reproduce the experiments.

\section{Related Work}
\label{related_work}\looseness=-1
Since the discovery of adversarial examples, many defense mechanisms have been proposed, 
\textit{adversarial training} being one of the most popular and empirically validated. We can categorise adversarial training methods based on how they approximate the perturbations applied to training samples. \textit{Multi-step} approaches approximate an inner maximization problem to find the worst-case perturbation with several gradient ascent steps \citep{TRADES, adv_ml_scale, PGD}. While this provides a better approximation, it is also more expensive. At the other end of the spectrum, \textit{single-step} methods only use one gradient step to approximate the worst case perturbation. \citet{fgsm} first proposed FGSM; \citet{tramer2018ensemble} proposed a new variant with an additional random step (R+FGSM), but observed that both methods were vulnerable to multi-step attacks.
\citet{free} proposed \textit{Free Adversarial Training} (Free-AT), which successfully reduced the computational cost of training by using 
a single backward pass to compute both 
weight update and attack. Motivated by this, \citet{RS-FGSM} explored a 
variant
of R+FGSM, namely RS-FGSM, that uses a less restrictive form of noise and showed this can improve 
robustness for moderate perturbation radii at the same cost as FGSM. Recently, \citet{grad_align} proposed the GradAlign regularizer. Combining FGSM with GradAlign results in robust models at even larger perturbation radii. However, GradAlign suffers from a~\textit{threefold} increase in the training cost to  as compared to FGSM.
%
The need for more efficient solutions has motivated a growing body of work whose goal is the design of computationally lighter single-step methods~\citep{zero_grad,AAAI,dropout_fgsm,SLAT,towards}.

In this work, we revisit the idea of combining noise with the FGSM attack. 
Our method builds upon FGSM and intuitions from R+FGSM and RS-FGSM to combine it with random perturbations, however, we consider the noise step as data augmentation rather than part of the attack. This motivates us to use a stronger noise \textit{without} clipping.
As opposed to~\cite{kang2021understanding}, our thorough study leads to a practically effective approach 
that yields robustness also against large perturbation radii.


\section{Preliminaries on Single-Step Adversarial Training} \label{sec:preliminaries}
\looseness=-1
Given a classifier $f_\theta : \mathcal{X} \rightarrow \mathcal{Y}$ parameterized by $\theta$ and a perturbation set $\mathcal{S}$, $f_\theta$ is
defined as
\textit{robust} at $x \in \mathcal{X}$ on the set $\mathcal{S}$ if for all \(\delta \in\mathcal{S}\) we have \(f_\theta(x+\delta) = f_\theta(x)\). One of the most popular definitions for $\mathcal{S}$ is the $\epsilon-\ell_\infty$ ball, \ie $\mathcal{S} =  \{\delta : \|\delta\|_\infty \leq \epsilon\}$. This is known as the $l_\infty$ threat model which we adopt throughout this work. 

To train networks that are robust against $\ell_\infty$ threat models, adversarial training modifies the classical training procedure of minimizing a loss function over a dataset $\mathcal{D} = \{(x_i,y_i)\}_{i=1:N}$ of images $x_i \in \mathcal{X}$ and labels $y_i \in \mathcal{Y}$. In particular, adversarial training instead minimizes the worst-case loss over the perturbation set $\mathcal{S}$, \ie training is on the adversarially perturbed samples
$\{(x_i+\delta_i,y_i)\}_{i=1:N}$. Under the $l_\infty$ threat model, we can formalize adversarial training as solving the following problem:
\begin{equation}
\begin{aligned} \label{eq:adversarial_training}
    \min_\theta \sum_{i=1}^N \max_{\delta} \mathcal{L}(f_\theta(x_i + \delta), y_i) \ \ \textrm{s.t.} \ \|\delta\|_\infty \leq \epsilon,
\end{aligned}
\end{equation}
where $\mathcal{L}$ is typically the cross-entropy loss.
Due to the difficulty of finding the exact inner maximizer, the most common procedure for adversarial training is to approximate the worst-case perturbation through several PGD steps \citep{PGD}. While PGD has been shown to yield robust models, 
its cost increases linearly with the number of steps.
As a result, several works have focused on reducing the cost of adversarial training by approximating the inner maximization with a single-step.

If the loss function is linear with respect to the input, the inner maximization of~\cref{eq:adversarial_training} will enjoy a closed form solution.~\citet{fgsm} leveraged this to propose FGSM, where the adversarial perturbation follows the direction of the sign of the gradient.~\citet{tramer2018ensemble} proposed adding a random initialization prior to FGSM. However, both methods were later shown to be vulnerable against multi-step attacks, such as PGD. Contrary to prior intuition, recent work from \citet{RS-FGSM} observed that combining a random step with FGSM can actually lead to a promising robustness performance. In particular, most recent single-step methods approximate the worst-case perturbation solving the inner maximization problem in~\cref{eq:adversarial_training} with the following general form:
\begin{equation}
\begin{aligned}\label{eq:general_update_rule}
    \delta = \psi\Big( \eta + \alpha \cdot \textrm{sign} \big( \nabla_{x_i} \mathcal{L}(f_\theta(x_i + \eta), y_i)  \big) \Big),
\end{aligned}
\end{equation}
where $\eta$ is drawn from a distribution $\Omega$.
For example, when $\psi$ is the projection operator onto the $\ell_\infty$ ball and $\Omega$ is the uniform distribution $[-\epsilon, \epsilon]^d$, where $d$ is the dimension of $\mathcal{X}$, this recovers RS-FGSM. 
Under a different noise setting where $\Omega = (\epsilon - \alpha) \cdot \textrm{sign} \left( \mathcal{N}(\mathbf{0}_d, \mathbf{I}_d) \right)$ and by choosing the step size $\alpha$ to be in $[0, \epsilon]$, we recover R+FGSM by \citet{tramer2018ensemble}. This was among the first works to explore the application of noise to FGSM, but did not report improvements over 
it. If we consider $\Omega$ to be deterministically $0$ and $\psi$ to be the identity map, we recover FGSM. Finally, if we take FGSM and add a gradient alignment regularizer, this recovers GradAlign.

\section{Noise and FGSM}\label{sec:noise_and_fgsm}
\looseness=-1
Previous methods that combined noise with FGSM, \eg R+FGSM~\citep{tramer2018ensemble} and RS-FGSM~\citep{RS-FGSM}, have considered the noise as a \textit{random step} integrated within the attack. Since it is a common practice to restrict adversarial perturbations to the $\epsilon-$ball, we argue that this introduces a trade-off between the magnitude of the random step and that of the attack. For illustration, consider the purple lines corresponding to RS-FGSM in~\Cref{figure:splash}~(left). If the initial random step is significantly larger than the $\epsilon-$ball, then the final clipping step will have a noticeable impact on the perturbation, possibly removing a considerable portion of the FGSM step (middle arrow). To prevent this from happening, R+FGSM and RS-FGSM restrict the random step to lie within the $\epsilon-$ball, thereof implicitly entangling the noise magnitude and the attack strength.

Contrary to previous methods, we note that adding noise to the clean sample can be considered as a form of \textit{data augmentation} to be applied independently from the attack. We make two considerations from this perspective 1) When one performs data augmentation during adversarial training, the input after the corresponding transformation is the starting point to compute the adversarial perturbations, therefore, we argue that adversarial attacks should be centered around the noise-augmented samples. This motivates us to \textit{avoid clipping} around the clean sample. 2) We do not need to restrict the noise augmentation to lie inside an $\epsilon-$ball, since its strength is disentangled from that of the attack. Thus, we can use \textit{stronger noise-augmentations} than previous methods.

These modifications lead to a novel adversarial training method that combines noise-based data augmentations with FGSM. We denote it as Noise-FGSM (N-FGSM). Following the notation introduced in \cref{sec:preliminaries}, we define the noise augmented sample as $x_{\textrm{aug}} = x + \eta$ where $\eta$ is sampled 
from a uniform distribution on $[-k, k]^d$ (where we can have $k>\epsilon$). Then the adversarially perturbed samples have the following form:
\begin{equation}
\begin{aligned}\label{equation:n-fgsm}
    x_{\text{N-FGSM}} = x_{\textrm{aug}} + \alpha \cdot \text{sign} \big(\nabla_{ x_{\textrm{aug}}} \mathcal{L}(f_\theta(x_{\textrm{aug}}), y)  \big). 
\end{aligned}
\end{equation}\looseness=-1
This construction corresponds to augmenting the clean sample $x$ with the perturbation defined in~\cref{eq:general_update_rule}~where $\psi$ is the identity map and $\Omega$ is the uniform ditribution spanning $[-k, k]^d$. We detail our full adversarial training procedure in~\cref{alg:n-fgsm}. In what follows, we analyse the effect of treating the noise as data augmentation as opposed to treating it as a random step within the attack. In particular, we show that clipping around the clean sample $x$ (as done in RS-FGSM) can strongly degrade the robustness of the network. Moreover, we show that as we increase the $\epsilon$ radii of adversarial attacks, we need stronger noise perturbations than previously used to prevent CO. \looseness=-1

\begin{figure*}
\centering
\begin{subfigure}[b]{.31 \linewidth}
\includegraphics[width=\linewidth]{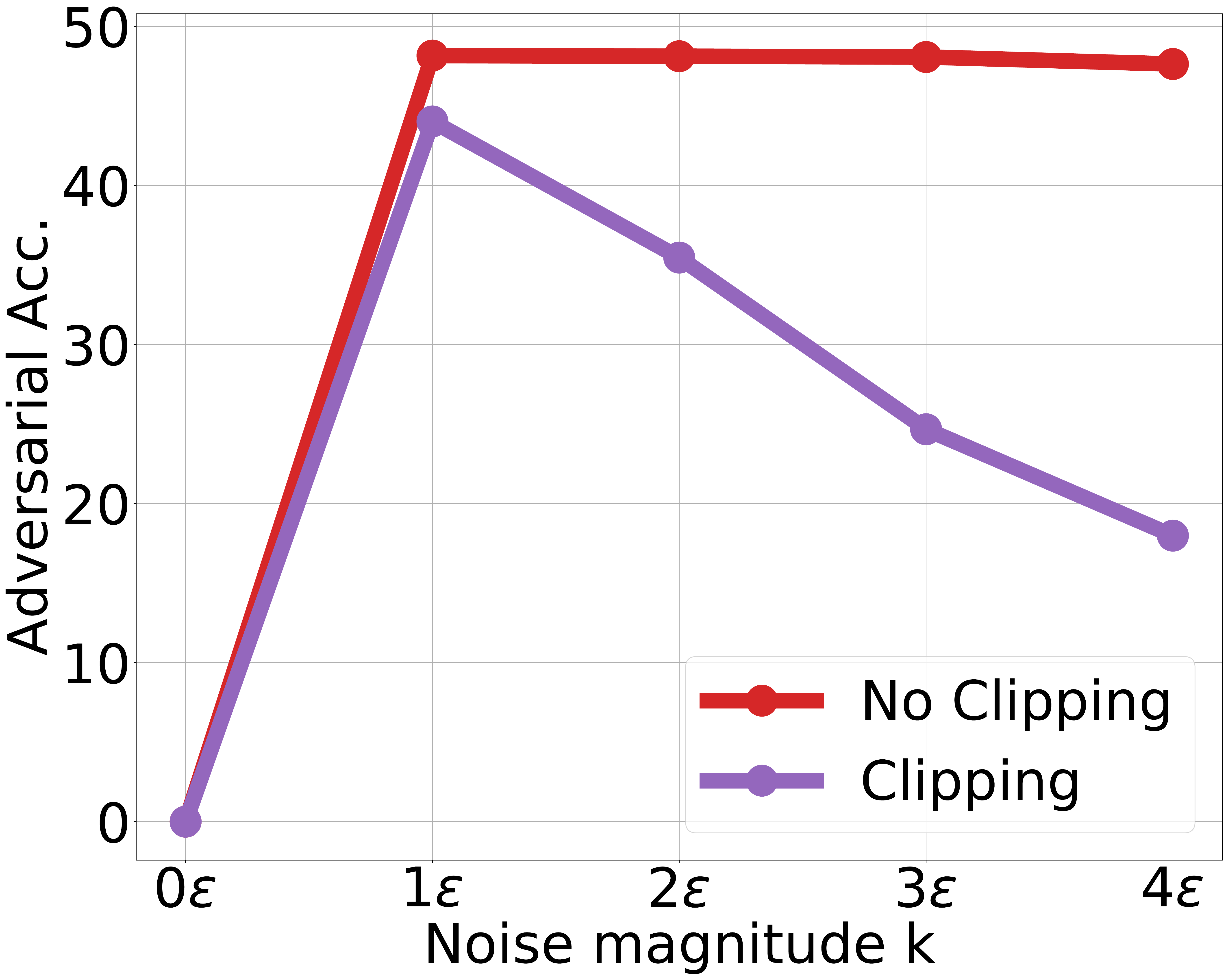}
\end{subfigure}
\hspace{5pt}
\begin{subfigure}[b]{.303 \linewidth}
\includegraphics[width=\linewidth]{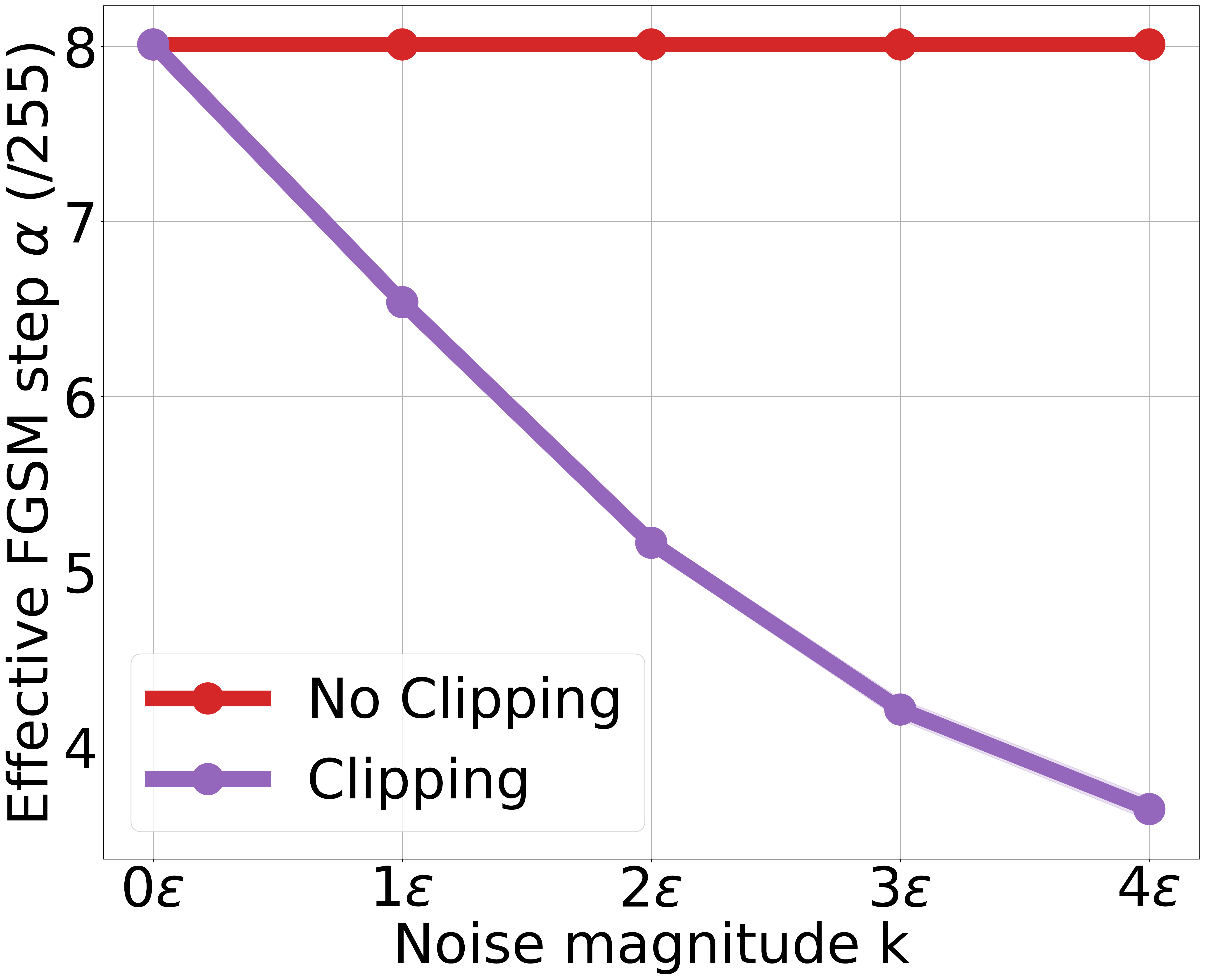}
\end{subfigure}
\hspace{5pt}
\begin{subfigure}[b]{.31 \linewidth}
\includegraphics[width=\linewidth]{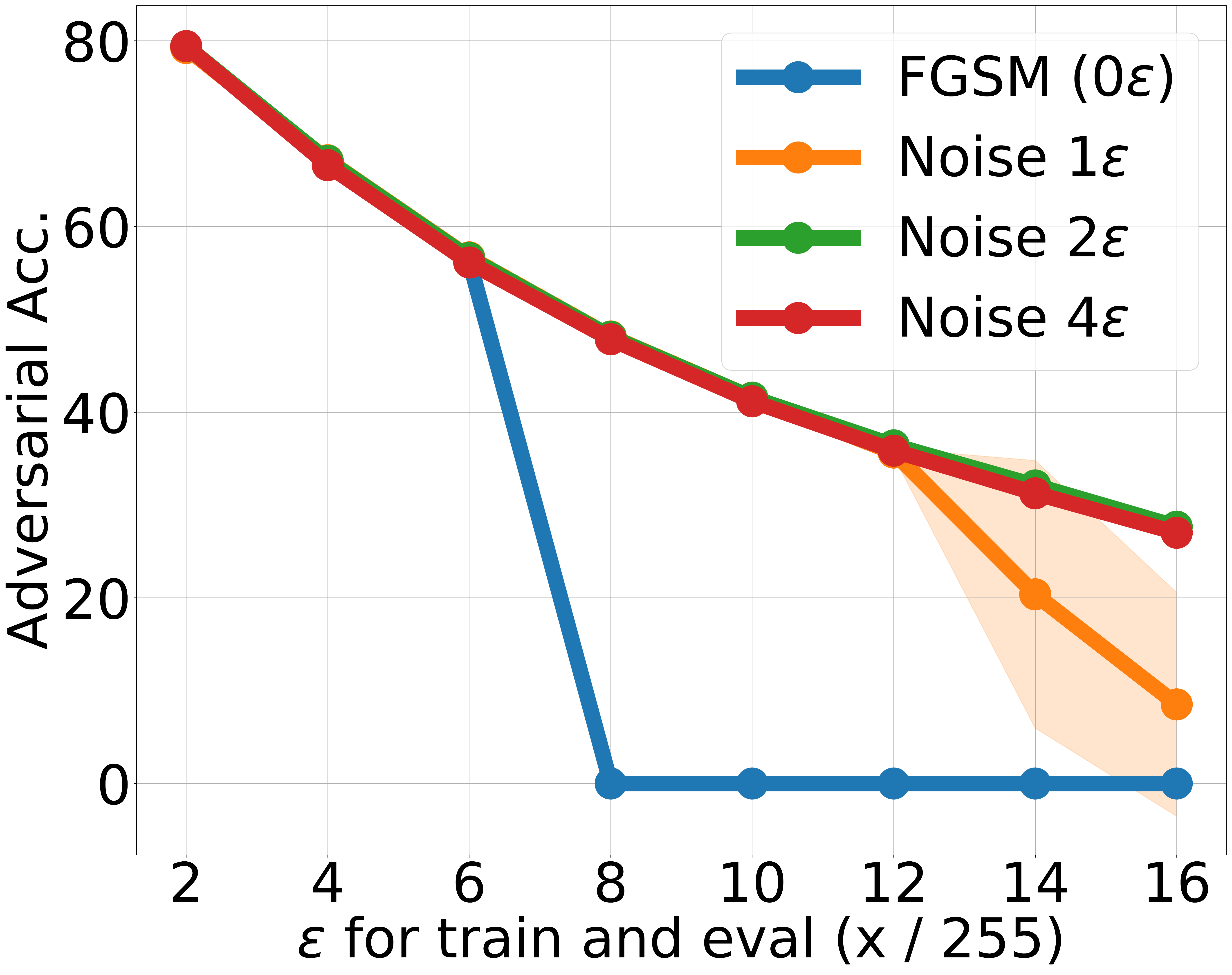}
\end{subfigure}
\caption{\textbf{Left:} Ablation of clipping vs not clipping around the clean sample $x$ for $\epsilon=\nicefrac{8}{255}$. Clipping leads to a significant drop in robustness which increases with the strength of the noise augmentations.  \textbf{Middle:} Analysis of the effective FGSM step size after clipping. We observe that clipping leads to a decrease in the effective FGSM step size, thus, adversarial perturbations will be more similar to random noise. \textbf{Right}: N-FGSM (ours) when varying the noise magnitude $k$ ($\epsilon$ is divided by 255). Increasing the amount of noise is key to avoiding CO. For (left) and (right) plots, adversarial accuracy is based on PGD-50-10 and experiments are averaged over 3 seeds.}
\label{figure:ablation}
\end{figure*} 

\textbf{Clipping around clean sample $x$ hinders the effectiveness of perturbations.}  \label{sec:analyses-noise} 
We analyse two variants, one where perturbations are clipped around the clean sample $x$ (as done in previous methods) and another where no clipping is applied. In~\cref{figure:ablation} (left), we report the robust accuracy using PGD-50-10 (\ie PGD attack with 50 iterations and 10 restarts) with $\epsilon=\nicefrac{8}{255}$ and observe that clipping significantly degrades the effectiveness of FGSM training. To understand this drop, consider the following perturbations; (\textbf{1}) a baseline perturbation where we only use noise $\delta_{\textrm{random}} = \psi(\eta)$ and (\textbf{2}) a perturbation that combines noise with FGSM $\delta_{\textrm{full}} = \psi(\eta + \alpha \cdot \text{sign} \big( \nabla_{x} \mathcal{L}(f_\theta(x + \eta), y)  \big))$. Moreover, we consider two cases in which we either define $\psi$ as a clipping operator or as the identity. We define the effective FGSM step size as the magnitude corresponding to the ratio\footnote{The denominator $\|x\|_2$ is simply to normalize the $\ell_2-$norm and 
be
comparable to the FGSM step size 
$\alpha$.} $\alpha_{\textrm{effective}}= \nicefrac{\| \delta_{\textrm{full}} - \delta_{\textrm{random}} \|_2}{\|x\|_2}$ which measures the contribution of the FGSM step in the final perturbation compared to simply following the noise direction $\eta$. In~\cref{figure:ablation} (middle), we observe that the clipping operator reduces the effective magnitude of FGSM, thus, perturbations become more similar to only using random noise. On the other hand, without clipping we always take the full step in the FGSM direction. 
This highlights the trade-off between noise magnitude and attack strength discussed above. \looseness=-1

\textbf{Larger noise is also necessary to prevent CO.} 
As discussed above, previous work did not investigate the effects of using noise perturbations potentially larger than the attack strength. However, we empirically find that increasing the noise magnitude is key to avoiding CO. In particular, as seen in \cref{figure:ablation} (right), when no clipping is performed, it is crucial that we augment with larger noise magnitude in order to prevent CO in all settings. We find the noise magnitude of $k = 2\epsilon$ to work well in most of our experiments, however, a more extensive hyperparameter tuning might improve our results further. \looseness=-1

Note that these results are contrary to previous intuitions: \citet{grad_align} suggested that the random step in RS-FGSM is not important per se, 
arguing that
its main role is 
reducing
the $\ell_2$ norm of the perturbations, so that the loss remains to be approximately locally linear. In contrast, N-FGSM perturbations are larger on expectation than those of RS-FGSM, while they do not suffer from CO (refer~\cref{sec:l2_norm}). We believe that our findings will lead to a better understanding of the role of noise in avoiding CO in future work. Moreover, in \cref{sec:increased_perturb} we conduct extensive analyses to show that, despite N-FGSM obtains larger perturbations, clean accuracy does not degrade and other methods do not benefit from simply increasing the strength of their attacks.



\textbf{Why does noise augmentation avoid CO?} \citet{grad_align} found that after CO, the gradients of the loss with respect to the input around clean samples became strongly misaligned, which is a sign of non-linearity. Moreover, \citet{AAAI} showed that the loss surface of models suffering from CO appears distorted, \ie there is a sharp peak in the loss surface along the FGSM direction, which seems to render FGSM ineffective (observe from~\cref{figure:grad_norm_delta_rank} how after after CO, visually, FGSM perturbations change drastically). In order to prevent CO, GradAlign explicitly regularizes the loss surface so it remains linear. To investigate further, we plot the loss surface at the end of training for different methods (see~\cref{figure:loss_surface} in~Appendix) and find that, while FGSM or RS-FGSM lead to a distorted loss, N-FGSM obtains a non-distorted loss surface similar to that obtained by GradAlign regularizer. Thus, it seems that adding strong noise-augmentations implicitly regularizes the loss landscape, leading to more effective single-step attacks. This aligns with previous work that theoretically link noise augmentations with a regularizer that encourages Lipschitzness \citep{bishop1995training}.

\begin{algorithm}[t]
	\caption{N-FGSM adversarial training} 
	\label{alg:n-fgsm}
	\begin{spacing}{1.3}

	\begin{algorithmic}[1]

	\STATE \textbf{Inputs:} epochs $T$, batches $M$, radius $\epsilon$, step-size $\alpha$ (default: $\epsilon$), noise magnitude $k$ (default: $2\epsilon$).
	\FOR {$t=1,\dots, T$}
	    \FOR{$i=1, \dots, M$}
    	    \STATE $\eta \sim \textrm{Uniform}[-k, k]^d$
    	    \STATE $x^i_\text{aug} = x^i + \eta$ // \textit{Augment sample with additive noise.}
    	    \STATE $x^i_\text{N-FGSM} = x^i_\text{aug} + \alpha \cdot \textrm{sign}\big(\nabla_{x^i_\text{aug}}\mathcal{L}(f_{\theta}(x^i_\text{aug}), y^i)\big)$ // \textit{N-FGSM augmented sample.}
    	    \STATE $\nabla_{\theta} = \nabla_{\theta} \mathcal{L}(f_{\theta}(x^i_\text{N-FGSM}), y^i)$ // \textit{Compute gradients of model's weights}
    		\STATE $\theta = \textrm{optimizer}(\theta, \nabla_{\theta})$ // \textit{Standard weight update, (\eg  SGD)}
        \ENDFOR
	\ENDFOR
	
	\end{algorithmic} 
	
	\end{spacing}

\end{algorithm}
\vspace{-3pt}

\section{Robustness Evaluations and Comparisons}
\vspace{-2pt}
\label{sec:experiments}
We compare N-FGSM against several
adversarial training methods, on a broad range of $\epsilon-l_\infty$ radii. Following \citet{RS-FGSM}, we evaluate adversarial robustness on CIFAR-10/100 \citep{cifar} and SVHN \citep{svhn}
with PGD-50-10 attacks, using both PreactResNet18 \citep{preact} and WideResNet28-10 \citep{zagoruyko2016wide}. Evaluations with AutoAttack\cite{croce2020reliable} are also in \cref{sec:autoattack}. 

\vspace{-2pt}
\subsection{Comparison against Single-Step Methods} \label{sec:single-step}
\vspace{-1pt}
We start by comparing N-FGSM against
other single-step methods.
Note that not all single-step methods are equally expensive, since they may involve more or less
computationally demanding operations. For instance, GradAlign uses a regularizer that is considerably expensive, while MultiGrad requires evaluating input gradients on multiple random points. For a comparison of training costs of different single-step methods, we refer the reader to~\cref{figure:splash} (right).
\begin{figure}
\centering
\begin{subfigure}[b]{.47\linewidth}
\includegraphics[width=\linewidth]{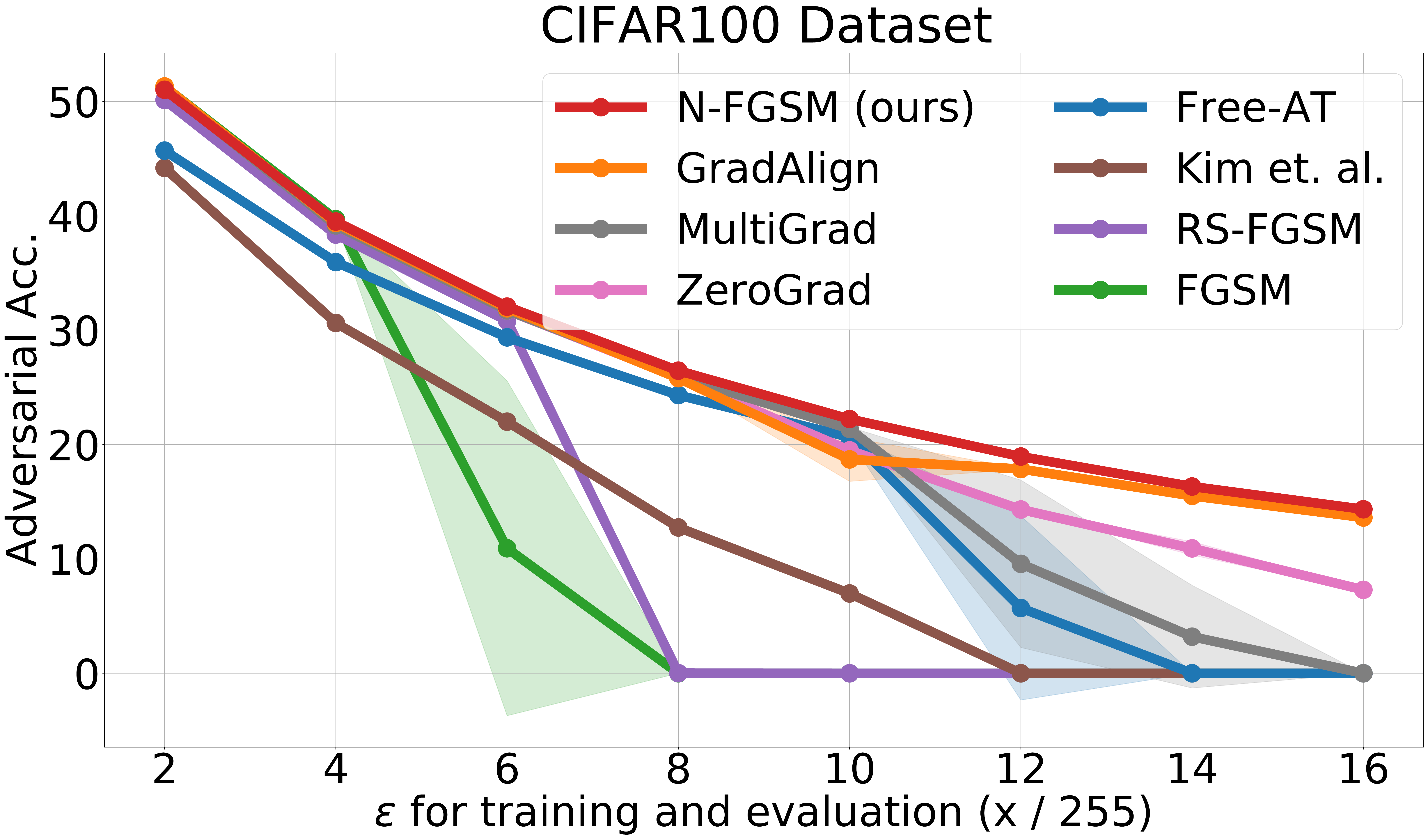}
\end{subfigure}
\hspace{10pt}
\begin{subfigure}[b]{.47\linewidth}
\includegraphics[width=\linewidth]{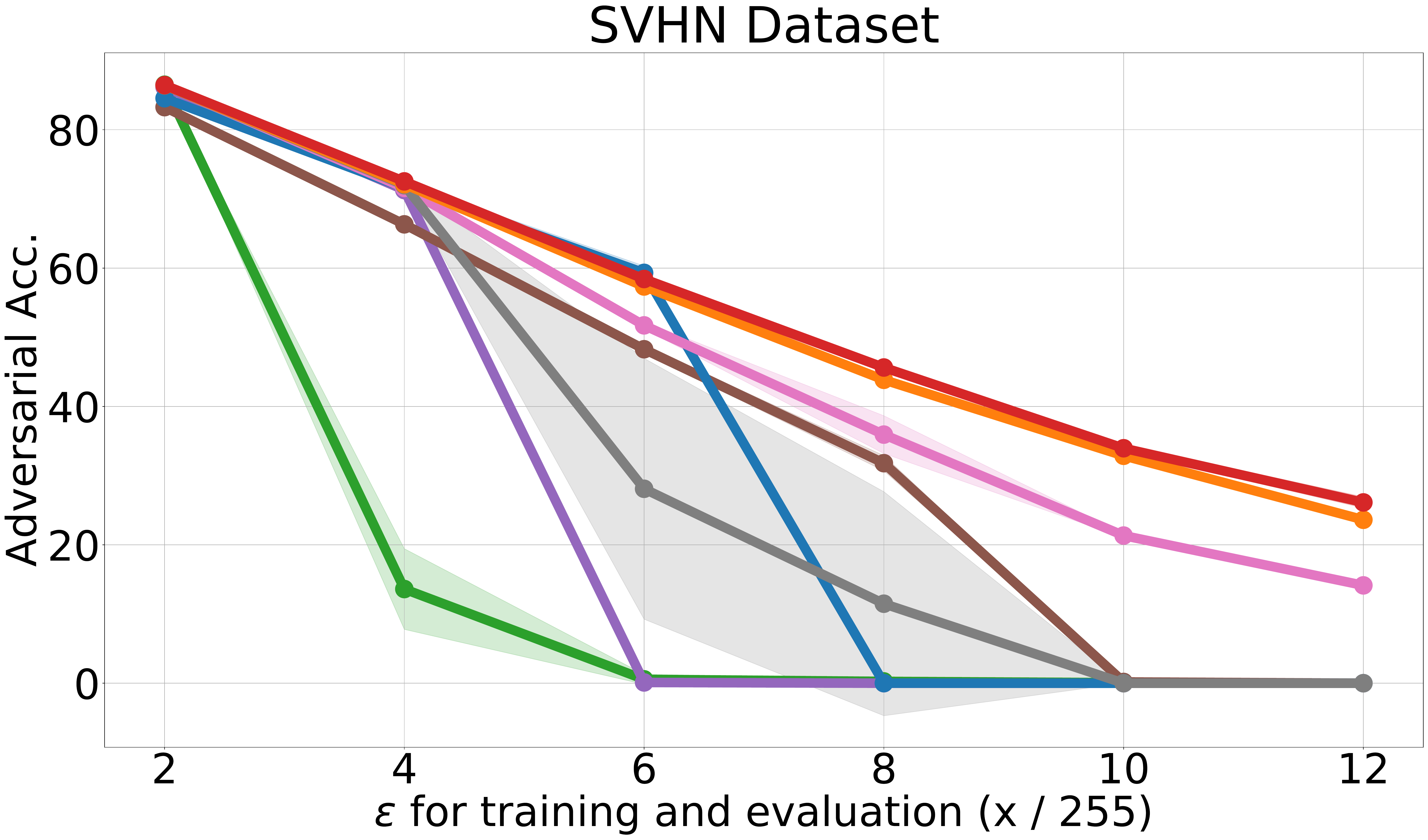}
\end{subfigure}
\caption{Comparison of single-step methods on CIFAR-100 (left) and SVHN (right) with PreactResNet18 over different perturbation radius ($\epsilon$ is divided by 255). Our method, N-FGSM, can match or surpass prior art results while \textit{reducing the cost by a $3\times$ factor}. Adversarial accuracy is based on PGD-50-10 and experiments are averaged over 3 seeds. Legend is shared among plots.}
\label{figure:comparison_single_step}
\end{figure} 
We use RS-FGSM and Free-AT with the settings recommended by \citet{RS-FGSM}. We apply GradAlign with hyperparameters reported in the official repository\footnote{\href{https://github.com/tml-epfl/understanding-fast-adv-training/}{https://github.com/tml-epfl/understanding-fast-adv-training/}}. ZeroGrad and \citet{AAAI} do not have a recommended set of hyperparameters; for a fair comparison we ablate them and select the ones with highest adversarial accuracy (for every $\epsilon$ and dataset). We train on CIFAR-10/100 for 30 epochs and on SVHN for 15 epochs with a cyclic learning rate. Only for Free-AT, we use 96 and 48 epochs for CIFAR-10/100 and SVHN, respectively, to obtain comparable results following \citet{RS-FGSM}. CIFAR-10 results are in~\cref{figure:splash} (middle), whereas CIFAR-100 and SVHN are in~\cref{figure:comparison_single_step}.

As observed in~\cref{figure:splash} and~\cref{figure:comparison_single_step}, FGSM and RS-FGSM suffer from CO for larger $\epsilon$ attacks on all reported datasets. For instance, RS-FGSM fails against attacks with $\epsilon = \nicefrac{8}{255}$ on CIFAR-10 and CIFAR-100 and against $\epsilon = \nicefrac{6}{255}$ on SVHN. With appropriate hyperparameters, ZeroGrad is able to consistently avoid CO. However, it obtains sub-par robustness compared to N-FGSM and GradAlign, especially against large $\epsilon$ attacks. Neither MultiGrad nor \citet{AAAI} avoid CO in all settings despite being more expensive. Free-AT also suffers from CO on all three datasets as also observed by \citet{grad_align}. In contrast, N-FGSM avoids CO on all datasets, achieving comparable or superior robustness to GradAlign \textit{while being 3 times faster.}



\subsection{Comparison against Multi-Step Attacks} \label{sec:multi-step}
In~\cref{sec:single-step}, we compared the performance of single-step methods and observed that N-FGSM is able to match or surpass the state-of-the-art method, \ie GradAlign, while reducing the computational cost by a factor of \(3\). In this section, we compare the performance of N-FGSM against multi-step attacks. In particular, we compare against PGD-2 with $\alpha = \nicefrac{\epsilon}{2}$ and PGD-10 with $\alpha = \nicefrac{2}{255}$, keeping the same training settings as described in~\cref{sec:single-step}. PGD-x denotes x iterations and no restarts.

In~\cref{figure:multi-step}, we observe that PGD-2, despite being a multi-step method, still suffers from CO for larger $\epsilon$ as opposed to our proposed N-FGSM. On the other hand, despite achieving comparable clean accuracies, there is a gap in adversarial accuracies between PGD-10, and other single-step methods that grows with perturbation size. This can be partially expected since the search space for adversaries grows exponentially with $\epsilon$; and PGD, with more iterations, can explore it more thoroughly.
Nevertheless, \textit{computing a PGD-10 attack is $10\times$ more expensive to N-FGSM}. An important direction for future work would be addressing this gap and analysing, both theoretically and empirically, whether single-step methods can match the performance of their multi-step counterparts.

\begin {figure}
\centering
\begin{subfigure}[b]{.47\linewidth}
\includegraphics[width=\linewidth]{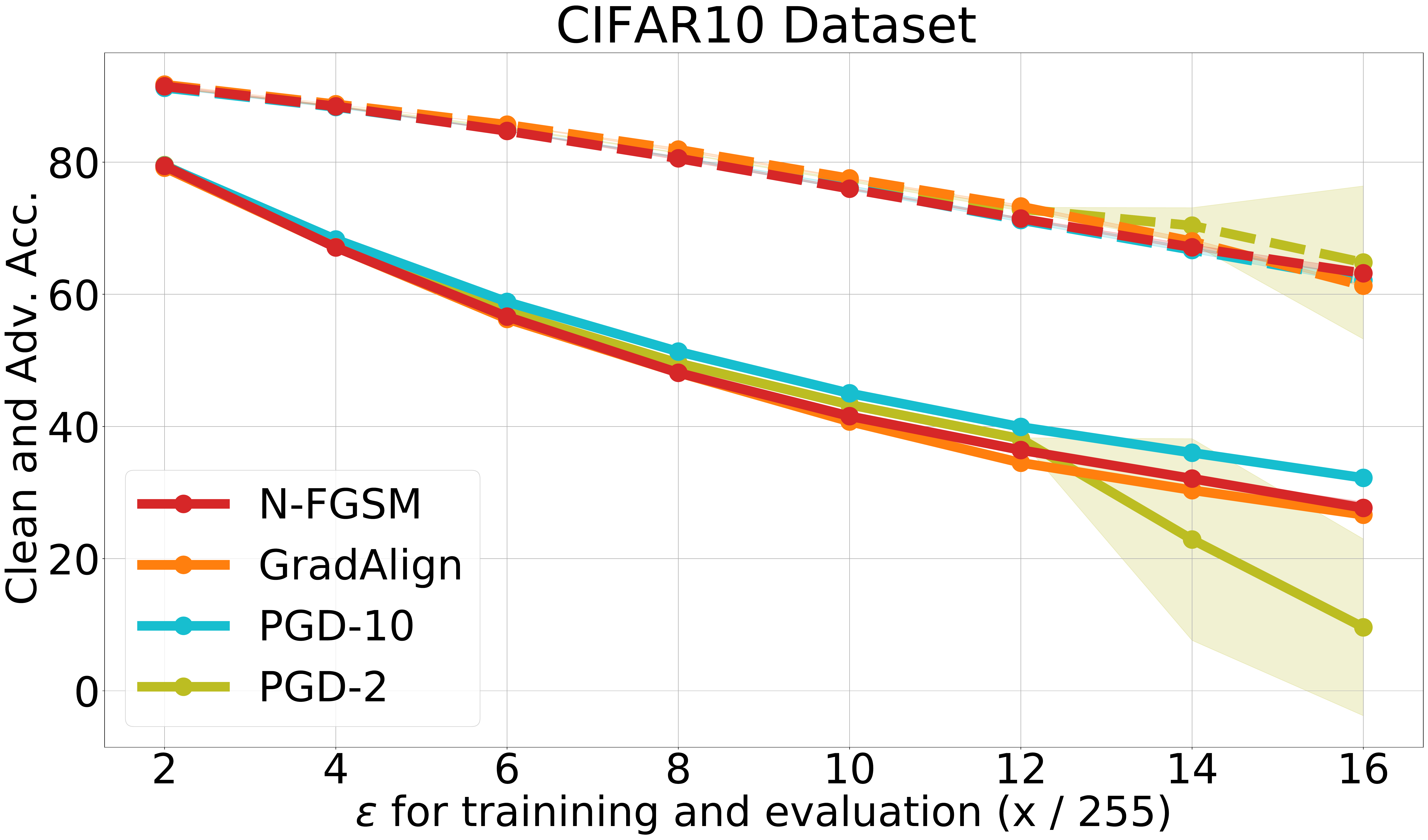}
\end{subfigure}
\hspace{10pt}
\begin{subfigure}[b]{.47\linewidth}
\includegraphics[width=\linewidth]{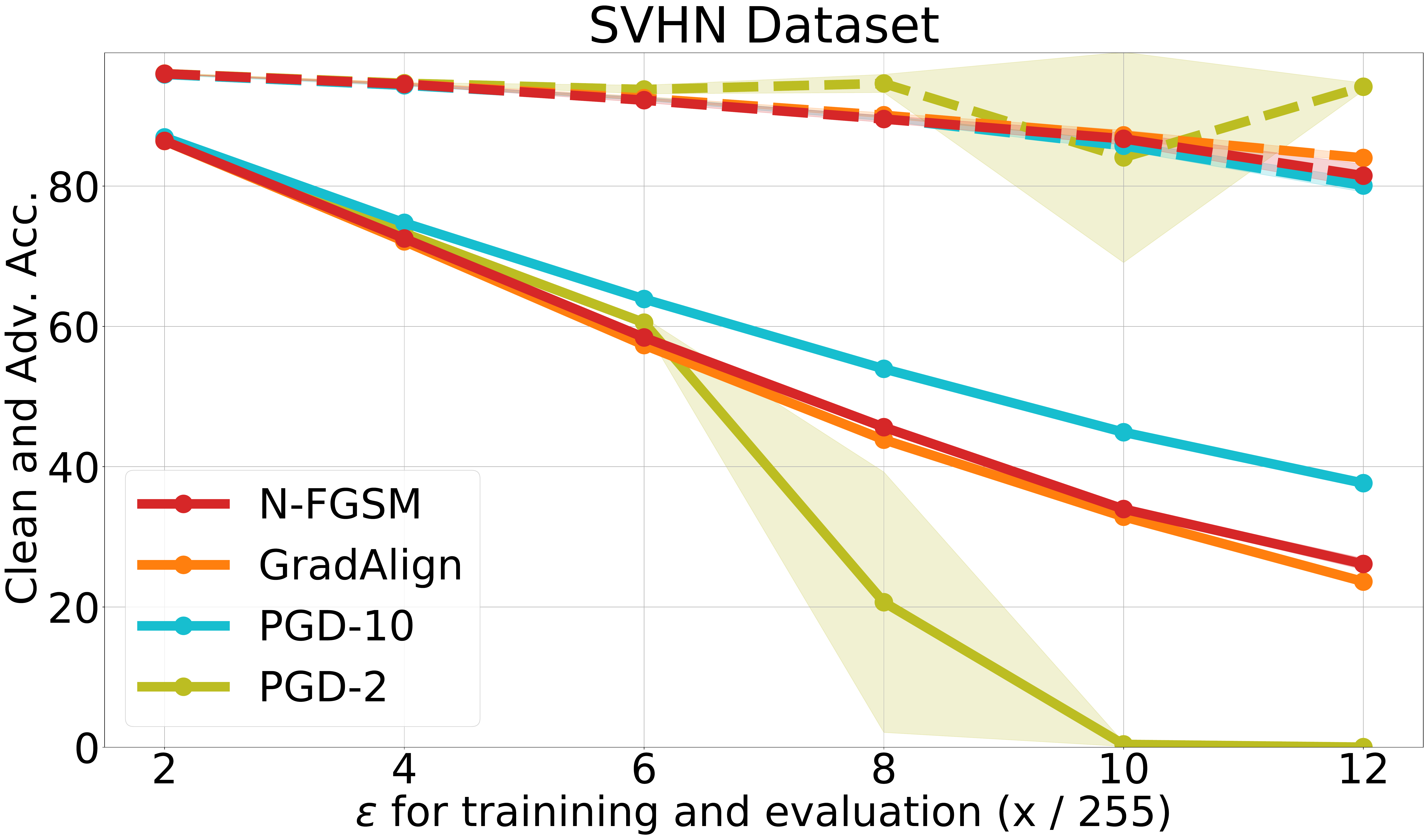}
\end{subfigure}
\caption{Comparison of N-FGSM and GradAlign with multi-step methods on CIFAR-10 (Left) and SVHN (Right) with PreactResNet18 over different perturbation radii ($\epsilon$ is divided by 255). Despite all methods achieving comparable clean accuracy (dashed lines), there is a gap in robust accuracy between PGD-10 and single-step methods. However, note that PGD-10 is $10\times$ more expensive than N-FGSM. Adversarial accuracy is based on PGD-50-10 and experiments are averaged over 3 seeds.}
\label{figure:multi-step}
\vspace{-5pt}
\end {figure} 

\subsection{Analysis of Gradients and Adversarial Perturbations} \label{sec:viz_perturbations} To gain further insights into CO, we visually explore the perturbations generated with FGSM, RS-FGSM, N-FGSM, and PGD-10 attacks. We show that N-FGSM generates perturbations that exhibit behavior similar to PGD-10. In particular, for a given test sample, we average the adversarial perturbations ($\delta$) and gradients across several epochs at the beginning of training~(Epoch 2 to 8) and at the end~(Epoch 24 to 30) and visualise them in \cref{figure:grad_norm_delta_rank} (see also~\cref{figure:deltas_grads_viz_extended} in Appendix for more examples). 
We observe that, during early stages in training all, methods generate consistent and interpretable $\delta$. However, after CO, FGSM and RS-FGSM generate $\delta$ that are harder to interpret, similarly to their gradients. On the other hand, we observe that N-FGSM provides consistent and interpretable $\delta$ throughout training, similar to those generated by PGD-10. This provides further evidence that N-FGSM enjoys similar properties to the more expensive PGD-10 training.

\cref{figure:grad_norm_delta_rank} analyzes the gradients and $\delta$ throughout the test set. Aside from loosing interpretability, post-CO the gradient norm increases by several orders of magnitude for FGSM and RS-FGSM while it remains low for N-FGSM and PGD-10. 
We also compute the effective rank\footnote{We compute effective rank as the number of singular vectors required to explain \(90\%\) of the variance.} of $\delta$ for each example before and after CO to measure the consistency of $\delta$ before and after CO. We consider three training intervals, (Epoch 2 to 8): before CO for all methods; (Epoch 16 to 22): after FGSM suffers CO but not RS-FGSM; (Epoch 24 to 30): after both FGSM and RS-FGSM suffer 
CO. Prior to CO, PGD-10 has a larger effective rank (\ie the perturbations span a larger subspace) than FGSM and RS-FGSM. N-FGSM has the highest effective rank, arguably due to the higher noise magnitude. Note that RS-FGSM, which has a smaller noise magnitude and clipping, also has a larger effective rank than FGSM, however, the difference is much lower. When either FGSM or RS-FGSM suffer from 
CO, the effective rank of their $\delta$ increases significantly above that of PGD-10 and N-FGSM. This would suggest that $\delta$ loose consistency after CO and is aligned with our visualizations in \cref{figure:grad_norm_delta_rank}. All of these show properties of $\delta$ and gradients that are consistent across methods (N-FGSM and PGD) that avoid CO and different from methods like RS-FGSM and FGSM, which suffer from CO.
 

\section{Increasing Adversarial Perturbations} \label{sec:increased_perturb}

\looseness=-1
In~\cref{sec:analyses-noise}, we observed that removing clipping and increasing the noise magnitude were both necessary for the improved performance of N-FGSM. However, as discussed in 
\cref{theorem:l2_norm} 
this will result in an increase of the squared norm of the training  perturbation \(\delta_{\mathrm{N-FGSM}}\) as compared to FGSM. In this section, we perform further ablations to corroborate that it is indeed the increase in noise magnitude
-- and not the mere increase of the perturbation's magnitude --
that helps to stabilize N-FGSM.

\begin{figure}[t]
\centering
\begin{subfigure}[b]{0.25\linewidth}
\includegraphics[width=\linewidth]{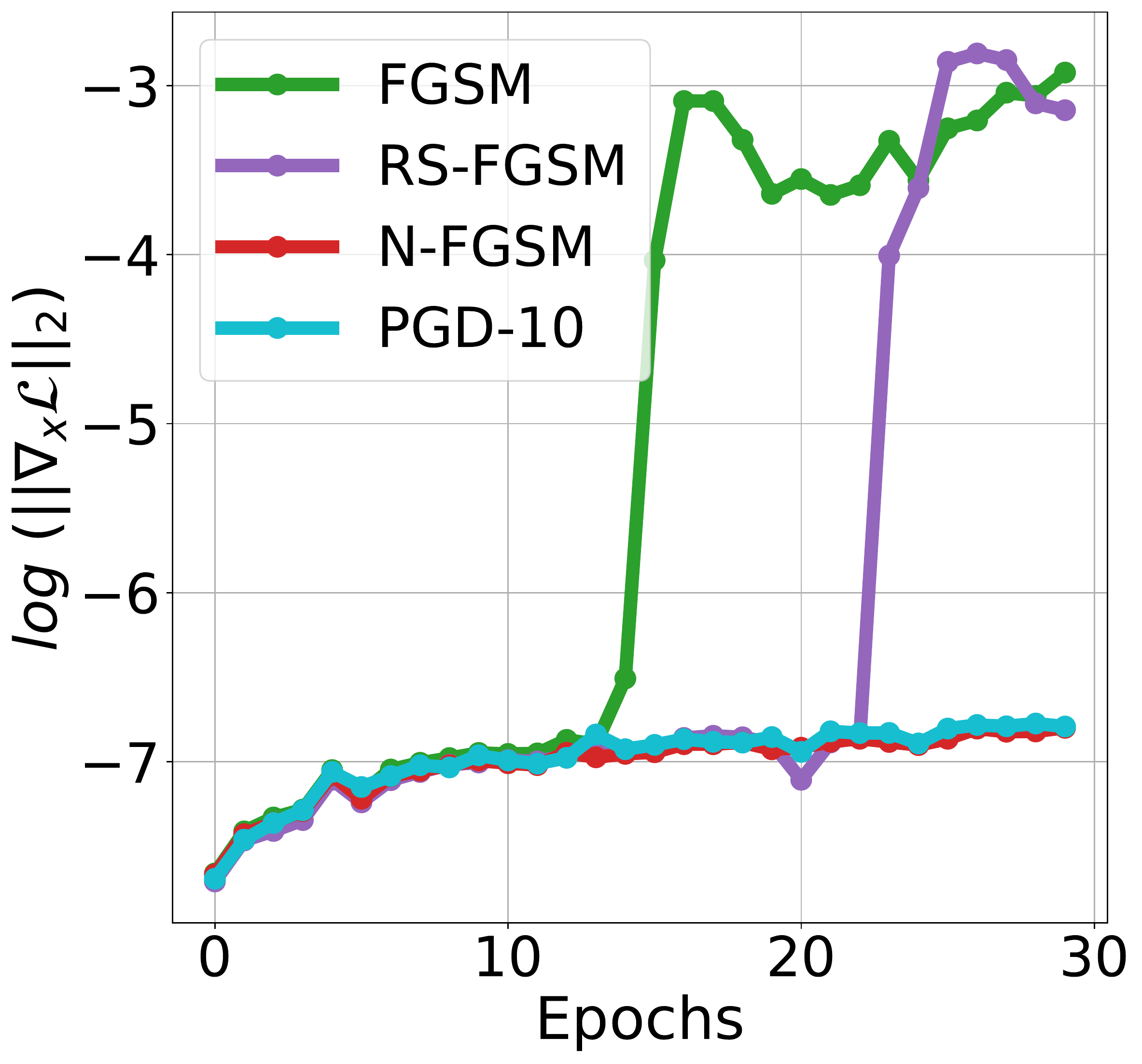}
\end{subfigure}
\hspace{1pt}
\begin{subfigure}[b]{0.25\linewidth}
\includegraphics[width=\linewidth]{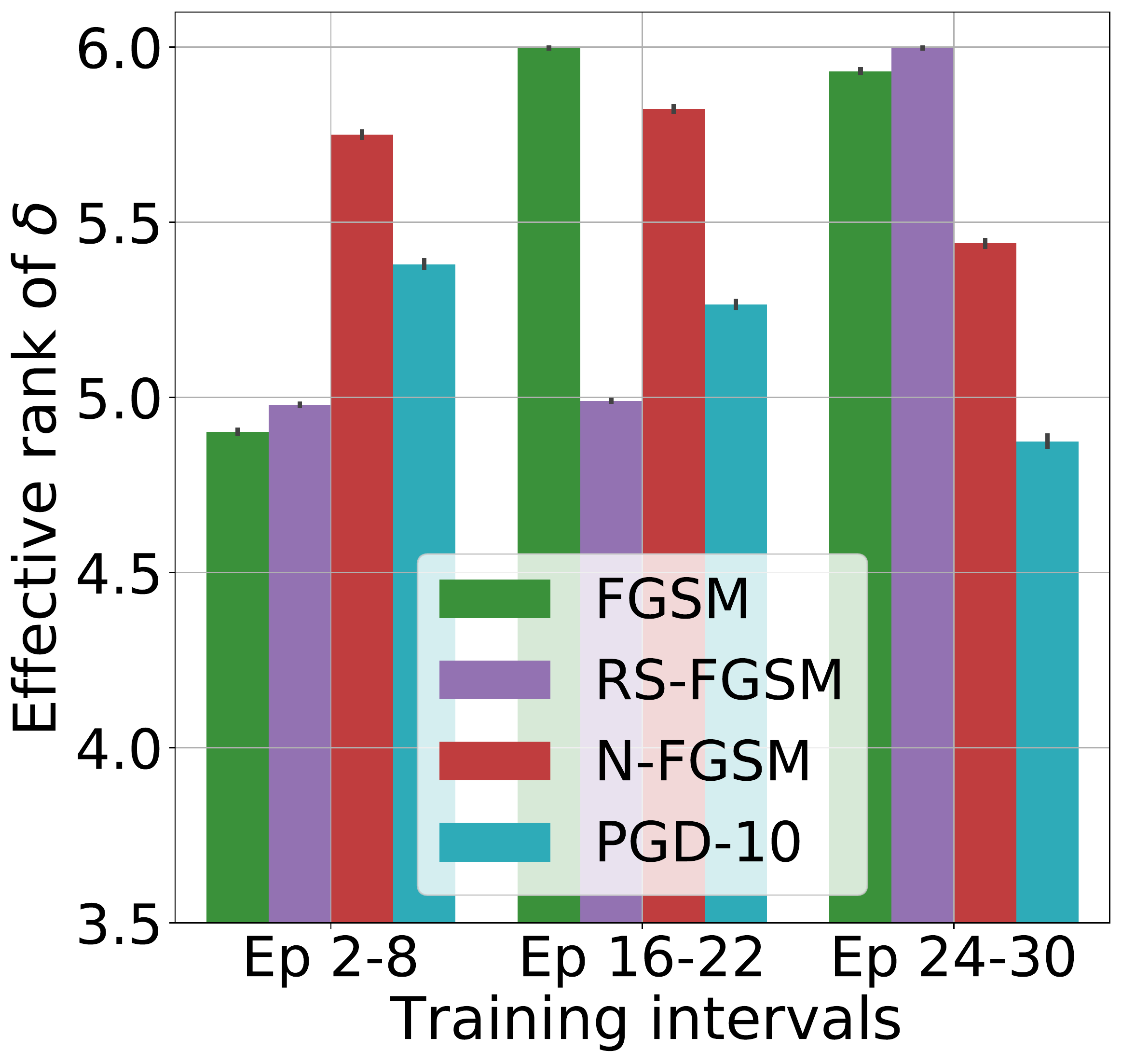}
\end{subfigure}
\hspace{1pt}
\begin{subfigure}[b]{0.45\linewidth}
\includegraphics[width=\linewidth]{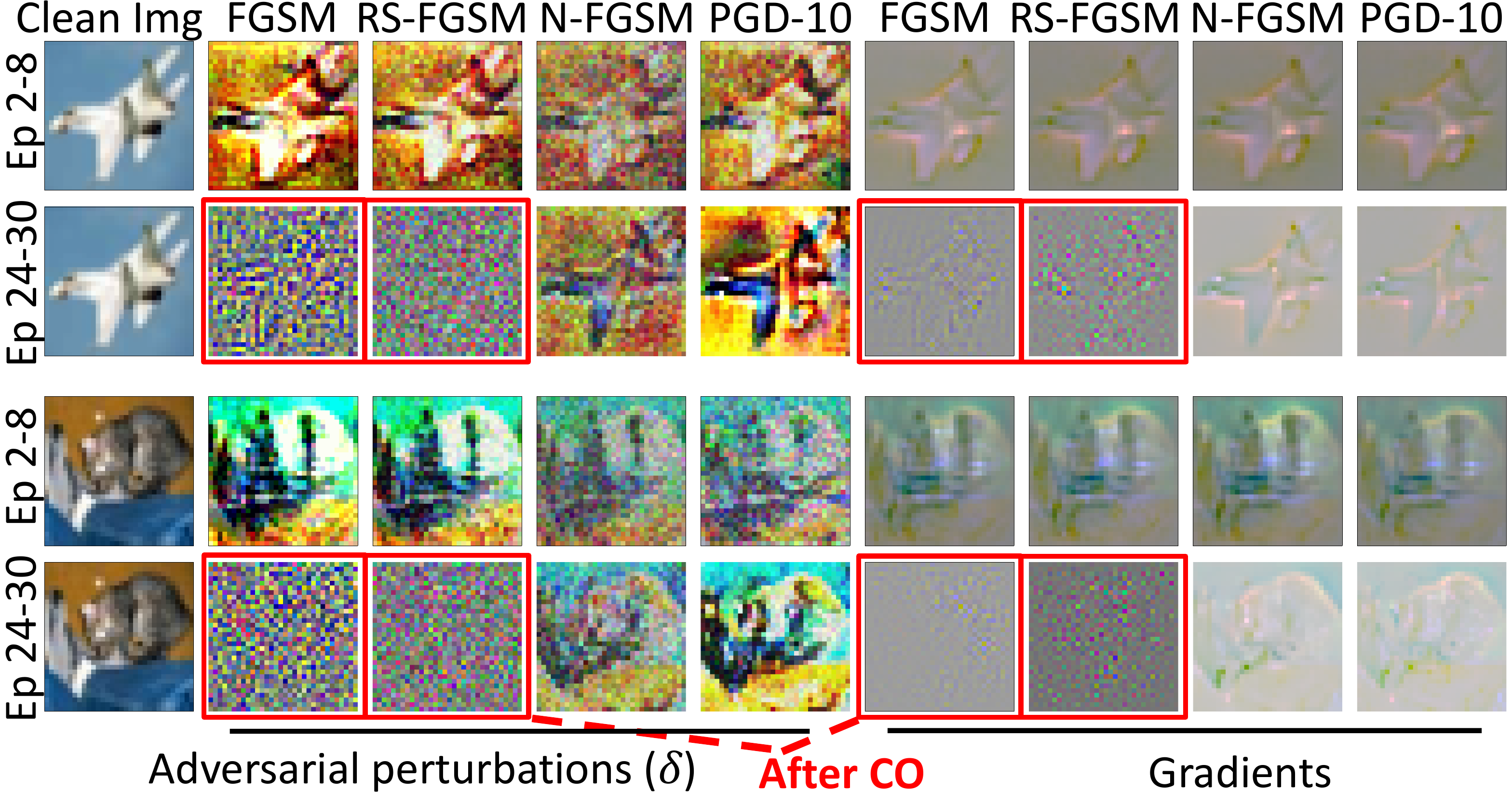}
\end{subfigure}
\caption{\textbf{Left:} Mean $\ell_2$ norm of per-sample-gradients across all test set samples. After CO, both FGSM and RS-FGSM gradients increase by several orders of magnitude. \textbf{Middle:} Effective rank of perturbations ($\delta$) across three training intervals \textit{Ep 2-8} before CO for all methods; \textit{Ep 16-22}: after FGSM presents CO but not RS-FGSM; \textit{Ep 24-30}: after both FGSM and RS-FGSM had CO. \textbf{Right:} Visualization of $\delta$ and gradients averaged across several epochs at the beginning (top) and end (bottom) of training. After CO, FGSM and RS-FGSM obtain $\delta$ and gradients that are hard to interpret.\looseness=-1}
\label{figure:grad_norm_delta_rank}
\vspace{-10pt}
\end{figure}



\looseness=-1
\textbf{Increasing $\alpha$ alone is not sufficient.} 
N-FGSM combines a noise perturbation with an FGSM step. Thus, we can increase the perturbation magnitude by increasing any of the two. This begs the question: Is it sufficient to increase the N-FGSM step-size $\alpha$ to avoid CO without adding any noise? We observe in~\cref{figure:exp2}~(A) that training without noise
(essentially, FGSM) leads to CO, with robust accuracy equal to zero, even for large values of $\alpha$.  This indicates that it is not an increase in the perturbation norm, but the combination with noise which plays an essential role in circumventing CO for N-FGSM.




\textbf{Increasing $\alpha$ requires adjusting the noise magnitude.} As observed in~\cref{figure:exp2}~(A), increasing $\alpha$ for N-FGSM leads to CO if the noise magnitude is not large enough. For example, while a noise magnitude 
\(k=1\epsilon\) and an adversarial step size \(\alpha=1.25\epsilon\)  
yield
a robust accuracy of $49.68\%$, increasing $\alpha$ to \(1.5\epsilon\) while keeping the same noise magnitude results in CO -- with robust accuracy equal to zero. This further suggests that an increase in the adversarial step-size $\alpha$ requires a commensurate increase in the noise magnitude. We find that setting the noise magnitude $k=2\epsilon$ works well for most settings.


\looseness=-1

\textbf{Larger noise perturbations preserve clean accuracy.} Increasing the norm of training perturbations by increasing $\alpha$ results in a drop in the clean accuracy (discussed later in \cref{sec:hyperparams}). 
This has also been observed in prior works \citep{RS-FGSM}. However, we 
show
in~\cref{figure:multi-step} that the clean accuracy for N-FGSM is similar to that of GradAlign, despite the magnitude of the perturbations being larger. 
We ablate the effects of adversarial and noise perturbations on the clean accuracy in~\cref{figure:exp2}~(B): we observe that augmenting training samples with noise alone (\ie $\alpha=0$) has a much milder effect on the clean accuracy than augmenting in an adversarial direction.
In general, increasing noise is more forgiving on the clean accuracy than increasing the adversarial step size.
This is not surprising, considering that moving in random directions along the input space has a significantly lower impact on the loss than moving along the FGSM direction (see~\cref{figure:loss_surface} in the Appendix) and that training with noise alone does not provide any significant robustness against larger attacks (for a more detailed ablation, see Appendix ~\cref{figure:noise_augmentation}).

\looseness=-1
\textbf{Other methods do not benefit from larger training $\epsilon$.} As previously mentioned, N-FGSM perturbations have $\ell_{\infty}-$norm larger than $\epsilon$. We have seen that the benefits of N-FGSM can not be reproduced by simply increasing $\alpha$ without increasing the noise. However, for the sake of completeness, we also ablate other single-step baselines by using a larger $\epsilon$ during training, while testing with a fixed $\epsilon=\nicefrac{8}{255}$ on CIFAR10. We observe that increasing $\epsilon_{\textrm{train}}$ seems to lead to a decrease in robustness for most methods;
for instance, PGD-50-10 accuracy for RS-FGSM 
drops
from $46.08 \pm 0.18$ when training with $\epsilon=\nicefrac{8}{255}$ to $0.0 \pm 0.0$ with $\epsilon=\nicefrac{12}{255}$. In two cases (GradAlign and MultiGrad), we observe a small increase, with the highest increase being for GradAlign, which improves from $48.14 \pm 0.15$ to $50.6 \pm 0.45$; yet, the clean accuracy drops from $81.9 \pm 0.22$ to $73.29 \pm 0.23$. This is similar to increasing $\alpha$ for N-FGSM (see \cref{figure:exp2}~(C)). However, this is tied to a significant degradation of clean accuracy.
All in all, taking into account both clean and robust accuracy, we conclude that all single-step baselines 
suffer from either CO or a severe degradation in their clean accuracy when increasing the training $\epsilon$. Full results are presented in \cref{table:increased_epsilon} in Appendix.

\vspace{-5pt}

\section{Additional Ablations}

\looseness=-1
\textbf{Hyperparameter selection.} \label{sec:hyperparams} While FGSM relies on a fixed step-size (\ie $\alpha = \epsilon$), \citet{RS-FGSM} explored different values of $\alpha$ for RS-FGSM, finding that an increase of the step-size improves the adversarial accuracy -- up to a point where CO occurs.
We also ablate the value of $\alpha$ for N-FGSM in~\cref{figure:exp2}~(C). We find that by increasing the noise magnitude, N-FGSM can use larger $\alpha$ values than RS-FGSM, without suffering from CO. This leads to an increase in the adversarial accuracy at the expense of a decrease in the clean accuracy. In light of this trade-off, 
we also use $\alpha = \epsilon$ for N-FGSM. Regarding the noise hyperparameter $k$, we find that $k=2\epsilon$ works in all but one SVHN experiment ($\epsilon=12$, in which we set $k=3\epsilon$). In comparison, GradAlign regularizer hyperparameter or ZeroGrad quantile value need to be modified for every radius with a noticeable shift between CIFAR-10 and SVHN hyperparameters, suggesting they may require additional tuning when applied to novel datasets.

\textbf{Long vs fast training schedules.}
Throughout our experiments, we used the RS-FGSM training setting introduced in \cite{RS-FGSM}. However, \citet{overfitting} suggest that a longer training schedule coupled with early stopping may lead to a boost in performance. \citet{AAAI} and \citet{towards} report that longer training schedules increase the chances of CO for RS-FGSM and that this limits its performance. We test the longer training schedule with N-FGSM and find that it does not suffer from CO. However, it does suffer from \textit{robust overfitting}, \ie adversarial accuracy on the training set is larger than on the test set as described in \cite{overfitting} for PGD-10. Notice the difference between the robust accuracy of the final and best models in~\cref{figure:exp2}~(D). Interestingly, although we observe a slight increase in performance when using the long training schedule, we find the fast training schedule to be remarkably competitive. 
See
more results in \cref{sec:long_schedule}, including a comparison to GradAlign.

\begin {figure}
\begin{minipage}[t]{.23\linewidth}
\centering
\includegraphics[width=\linewidth]{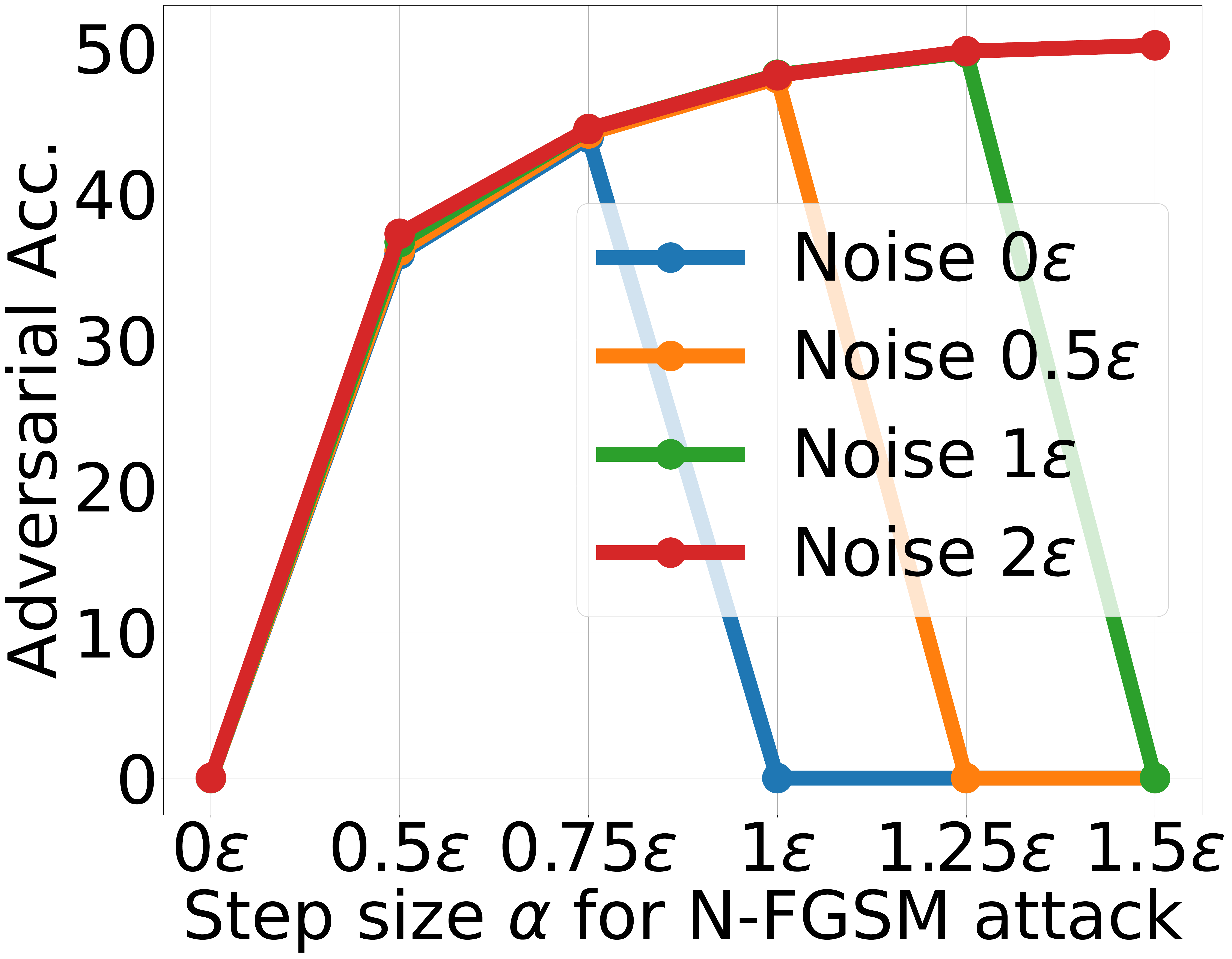}
\end{minipage}
\begin{minipage}[t]{.23\linewidth}
\centering
\includegraphics[width=\linewidth]{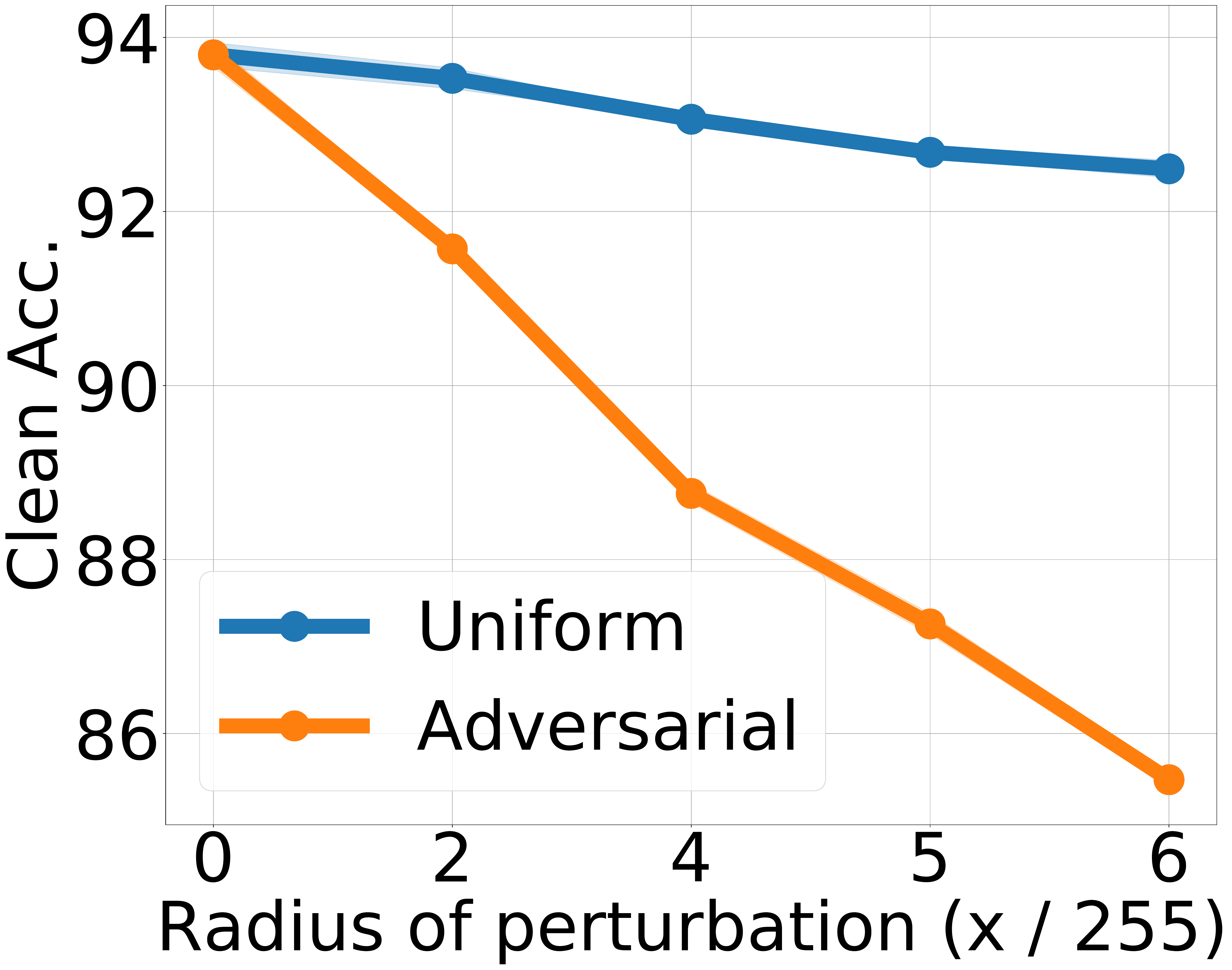}
\end{minipage}
\begin{minipage}[t]{.225\linewidth}
\centering
\includegraphics[width=\linewidth]{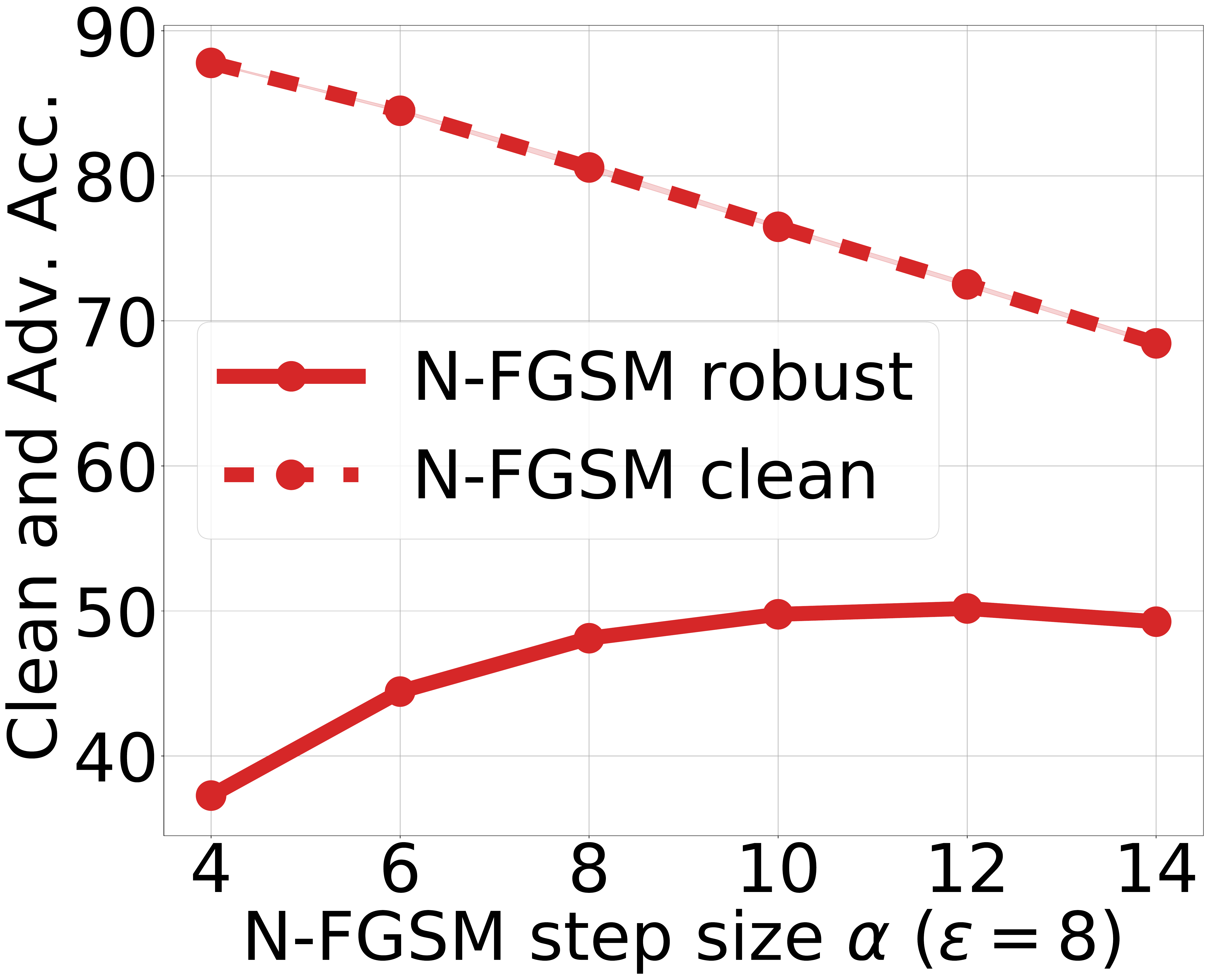}
\end{minipage}
\begin{minipage}[b]{.225\linewidth}
\renewcommand{\arraystretch}{1.2}
\scriptsize{
\begin{tabular}{c c}
    \multicolumn{2}{c}{\textbf{Comparison of Training Schedules}}\\
    Clean Acc & Robust Acc \\
    \hline
    \rowcolor{gray!20}
    \multicolumn{2}{c}{Long schedule: Final model}\\
     \textbf{83.18 $\pm$ 0.11} & 36.56 $\pm$ 0.26 \\
    \rowcolor{gray!20}
    \multicolumn{2}{c}{Long schedule: Best model}\\
     80.8 $\pm$ 0.36 & \textbf{48.48 $\pm$ 0.27}\\
    \rowcolor{gray!20}
     \multicolumn{2}{c}{Fast schedule: Final model}\\
     80.58 $\pm$ 0.22 & 48.12 $\pm$ 0.07 \\
\end{tabular}
}
\renewcommand{\arraystretch}{1.2}
\end{minipage}
\caption{Different ablations on N-FGSM parameters and training schedule. From left to right: \textbf{A:}~Adversarial accuracy when varying step-size $\alpha$ and noise magnitude $k$ ($\epsilon=8$). 
Increasing $\alpha$ does not suffice to prevent CO, we must also increase the noise magnitude.
\textbf{B:}~Clean accuracy after training with random or adversarial perturbations. With comparable radius, random perturbations have a much milder effect than adversarial. \textbf{C:}~Ablation of step size $\alpha$ in N-FGSM $\epsilon=8, k=2\epsilon$. As we increase the magnitude of the FGSM perturbation we observe an increase in robustness coupled with a drop on the clean accuracy. \textbf{D}:~Comparison of the ``fast'' training schedule from \cite{RS-FGSM} and ``long" training schedule described in \cite{overfitting}. N-FGSM shows robust oberfitting but not CO with the long schedule. Adversarial accuracy is based on PGD-50-10 and experiments are averaged over 3 seeds. \looseness=-1}
\label{figure:exp2}
\vspace{-5pt}
\end {figure} 

\textbf{Experiments with WideResNet28-10.} \label{sec:wideresnet}
We also compare the performance of all methods on WideResNet28-10 \citep{zagoruyko2016wide} architecture in \cref{figure:wideresnet_single_step} and \cref{figure:wideresnet_all} in Appendix. As in the experiments with PreActResNet18, N-FGSM obtains the best performance/cost trade-off. 
We had to increase the regularizer hyperparameter for GradAlign (compared to the settings for PreActResNet18) in order to prevent CO on CIFAR-100 and, to our surprise, \textit{we could not find a competitive hyperparameter setting for GradAlign on the SVHN dataset} for $\epsilon \geq 6$. We tried both increasing the regularizer hyperparameter and decreasing the step size $\alpha$, but some or all runs led to models close to a constant classifier for each setting. We do not claim that GradAlign will not work, but finding a good configuration might require further tuning. The default configuration for N-FGSM ($\alpha=\epsilon,\ k=2\epsilon$) works well in all settings except for $\epsilon=16$ on CIFAR-10 and $\epsilon = 10,\ 12$ on SVHN. For CIFAR-10, we increase the noise magnitude to $k = 4\epsilon$. For SVHN, we find that decreasing $\alpha$ 
works better than increasing the noise. 
In both cases, N-FGSM yields
nontrivial
adversarial accuracy.

\textbf{Experiments on Imagenet.} \label{sec:imagenet}
We present results on the Imagenet dataset \cite{krizhevsky2012imagenet} in~\cref{tab:imagenet}. Due to the high computational demands of Imagenet training and testing we focus on the main baselines of comparable cost to FGSM. Namely FGSM, RS-FGSM and N-FGSM.  We observe that FGSM presents CO for $\epsilon=\nicefrac{6}{255}$ while neither RS-FGSM nor N-FGSM present CO. However, N-FGSM has better robustness. For instance, at $\epsilon=\nicefrac{6}{255}$ N-FGSM obtains PGD50-10 accuracy of 17.12\% while RS-FGSM yields 16.5\% and FGSM 0.08\% (due to CO). Thus, N-FGSM also avoids CO in ImageNet, improving robustness over same-cost baselines. For experimental details refer to \cref{sec:imagenet_details}. 
 \begin{table}
\caption{Clean accuracy (top) and PGD50-10 accuracy (bottom) of N-FGSM and other same-cost baselines on Imagenet dataset. We observe that FGSM presents CO for $\epsilon=\nicefrac{6}{255}$ while both RS-FGSM and N-FGSM avoid CO. N-FGSM has consistently better robustness than baselines.}
\label{tab:imagenet}
\vspace{2pt}
\renewcommand{\arraystretch}{1.5}
\centering
\small{
\begin{tabular}{c|c|c|c}
\toprule
     & 
     $\epsilon=\nicefrac{2}{255}$ &
     $\epsilon=\nicefrac{4}{255}$ &
     $\epsilon=\nicefrac{6}{255}$                           
     \\ \hline \hline
FGSM & 
     \begin{tabular}[c]{@{}c@{}}54.72\\ 38.21\end{tabular} & 
     \begin{tabular}[c]{@{}c@{}}48.50\\ 25.86\end{tabular} & \begin{tabular}[c]{@{}c@{}}48.55\\ 0.08\end{tabular} \\ \hline
RS-FGSM &
  \begin{tabular}[c]{@{}c@{}}56.29\\ 36.86\end{tabular} &
  \begin{tabular}[c]{@{}c@{}}50.81\\ 25.12\end{tabular} &
  \begin{tabular}[c]{@{}c@{}}47.67\\ 16.49\end{tabular} \\ \hline
N-FGSM &
  \begin{tabular}[c]{@{}c@{}}54.39\\ 38.07\end{tabular} &
  \begin{tabular}[c]{@{}c@{}}47.56\\ 26.28\end{tabular} &
  \begin{tabular}[c]{@{}c@{}}47.70
\\ 17.12\end{tabular} \\ \hline
\end{tabular}
}
\end{table}

\textbf{Combining N-FGSM with additional regularizers.} \label{sec:gat}
Recent works \citep{GAT, NuAT} have been proposed to improve performance of single step methods at moderate perturbation radii $\epsilon = \nicefrac{8}{255}$. However, we observe that with the default settings (which use a version of RS-FGSM with Bernoulli noise) they lead to CO for larger $\epsilon$. Then we compare them with N-FGSM + Regularizer where we apply their proposed regularizers to N-FGSM. If we apply GAT or NuAT regularizers to N-FGSM then we do not observe CO and usually a boost in performance. For instance, at $\epsilon=\nicefrac{10}{255}$, GAT has a robust acc (with PGD50-10) of 43.34 $\pm$ 0.23 while N-FGSM+GAT regularizer obtains 44.97 $\pm$ 0.07, in comparison plain N-FGSM has 41.56 $\pm$ 0.16. This is extremely compelling as it suggests N-FGSM can be combined with other regularizers designed to improve FGSM performance and mutually benefit each other. Full results of the comparison are presented in \cref{tab:gat_nuat} in \cref{sec:app_gat}


\section{Conclusion}
In this work, we explore the role of noise and clipping in single-step adversarial training. Contrary to previous intuitions, we show 
that increasing the noise magnitude and removing the $\epsilon - \ell_\infty$ constraint leads to an improvement in adversarial robustness while maintaining a competitive clean accuracy. These findings led us to propose N-FGSM, a simple and effective approach that can match or surpass the performance of GradAlign \citep{grad_align}, while achieving a $3\times$ speed-up. 

We perform an extensive comparison with other relevant single-step methods, observing that all of them achieve sub-optimal performance and most of them are not able to avoid CO for larger $\epsilon$ attacks. Moreover, we also analyze gradients and adversarial perturbations during training and observe that they have a similar behaviour for N-FGSM and PGD-10 as opposed to other methods that present CO such as FGSM and RS-FGSM.
However, despite impressive improvements of single-step adversarial training methods, there is still a gap between single-step and multi-step methods such as PGD-10 as we increase the $\epsilon$ radius. Therefore, future work should put an emphasis on formally understanding the limitations of single-step adversarial training and explore how, if possible, this gap can be reduced.

\begin{ack}
We thank Guillermo Ortiz-Jim\'enez for the fruitful discussions and feedback. This work is supported by the UKRI grant: Turing AI Fellowship EP/W002981/1 and EPSRC/MURI grant: EP/N019474/1. We would also like to thank the Royal Academy of Engineering and FiveAI. A. Sanyal acknowledges support from the ETH AI Center postdoctoral fellowship. Pau de Jorge is fully funded by NAVER LABS Europe. 
\end{ack}


\bibliographystyle{plainnat}
\bibliography{main.bib}

\section*{Checklist}


\begin{enumerate}

\item For all authors...
\begin{enumerate}
  \item Do the main claims made in the abstract and introduction accurately reflect the paper's contributions and scope? 
    \answerYes{We provide experimental validation for all our claims.}
  \item Did you describe the limitations of your work?
    \answerYes{Yes, in conclusions and when appropriate along the paper.}
  \item Did you discuss any potential negative societal impacts of your work?
    \answerNA{We do not consider research on training robust networks can have negative impact.}
  \item Have you read the ethics review guidelines and ensured that your paper conforms to them?
    \answerYes{}
\end{enumerate}

\item If you are including theoretical results...
\begin{enumerate}
  \item Did you state the full set of assumptions of all theoretical results?
    \answerYes{}
        \item Did you include complete proofs of all theoretical results?
            \answerYes{}
\end{enumerate}

\item If you ran experiments...
\begin{enumerate}
  \item Did you include the code, data, and instructions needed to reproduce the main experimental results (either in the supplemental material or as a URL)?
    \answerYes{We include the URL to our code in the abstract.}
  \item Did you specify all the training details (\eg data splits, hyperparameters, how they were chosen)?
    \answerYes{When relevant in the text.}
        \item Did you report error bars (\eg with respect to the random seed after running experiments multiple times)?
            \answerYes{All plots or numbers have standard deviation. Although sometimes too small to be visible.}
        \item Did you include the total amount of compute and the type of resources used (\eg type of GPUs, internal cluster, or cloud provider)?
            \answerYes{Yes, see appendix \ref{sec:gpu-hours}.}
\end{enumerate}

\item If you are using existing assets (\eg code, data, models) or curating/releasing new assets...
\begin{enumerate}
  \item If your work uses existing assets, did you cite the creators?
    \answerYes{}
  \item Did you mention the license of the assets?
    \answerNA{}
  \item Did you include any new assets either in the supplemental material or as a URL?
    \answerYes{We include the URL to our code in the abstract.}
  \item Did you discuss whether and how consent was obtained from people whose data you're using/curating?
    \answerNA{}
  \item Did you discuss whether the data you are using/curating contains personally identifiable information or offensive content?
    \answerNA{}
\end{enumerate}

\item If you used crowdsourcing or conducted research with human subjects...
\begin{enumerate}
  \item Did you include the full text of instructions given to participants and screenshots, if applicable?
    \answerNA{}
  \item Did you describe any potential participant risks, with links to Institutional Review Board (IRB) approvals, if applicable?
    \answerNA{}
  \item Did you include the estimated hourly wage paid to participants and the total amount spent on participant compensation?
    \answerNA{}
\end{enumerate}

\end{enumerate}


\newpage
\appendix

\section{Additional plots for PreActResNet18 experiments}
In the main paper we compare N-FGSM with other single-step methods and multi-step methods separately and remove clean accuracies for better visualization. In this section we present the curves for all methods with both the clean and robust accuracy. The tendency in the three datasets is for N-FGSM PGD-50-10 accuracy to be slightly above that of GradAlign, while the opposite happens to the clean accuracy. We also observe that clean accuracy becomes significantly more noisy when CO happens. Exact numbers for all the curves are in~\cref{sec:detailed_results}.

\begin{figure}[h]
\centering
\begin{subfigure}[b]{.32\linewidth}
\includegraphics[width=\linewidth]{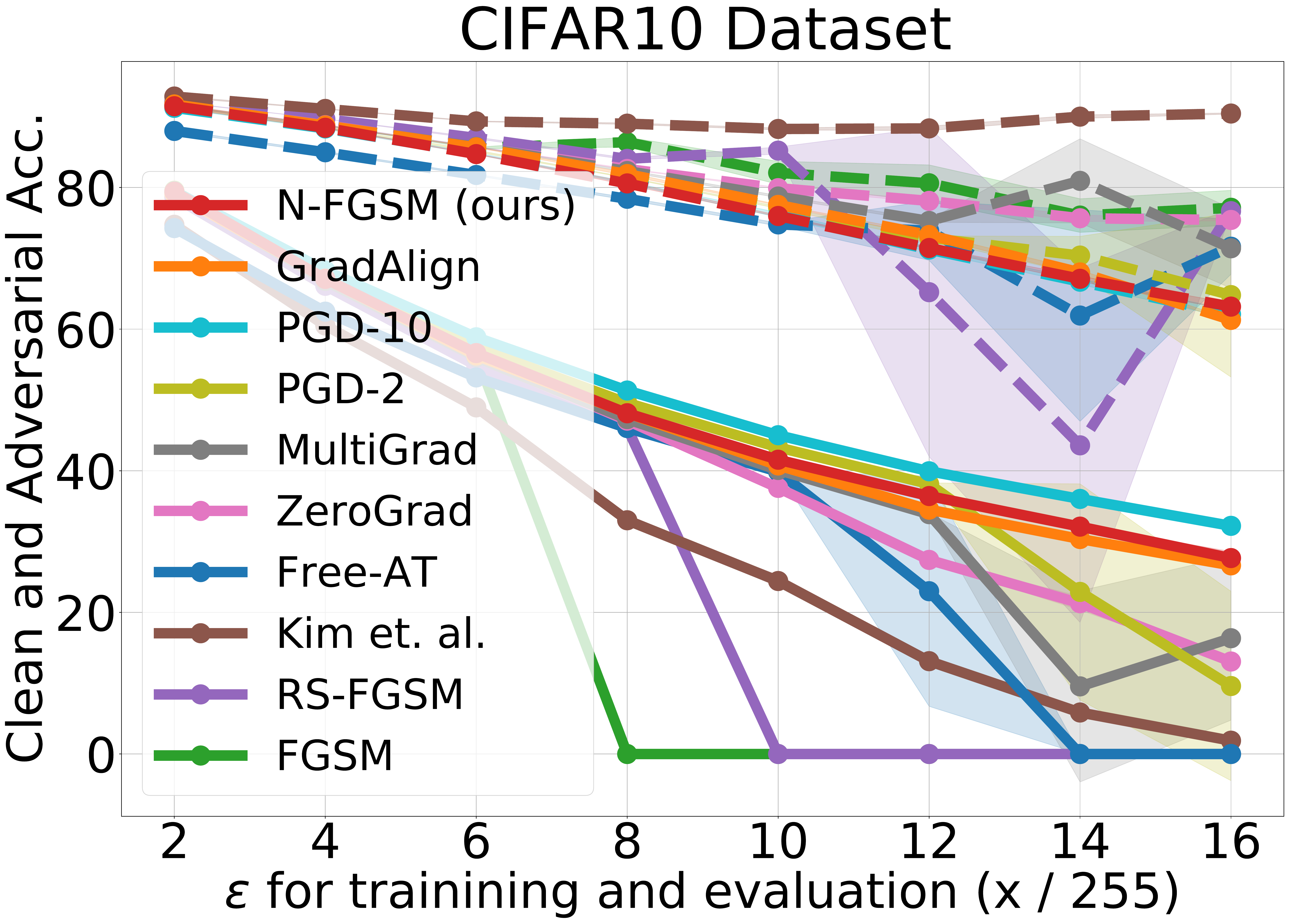}
\end{subfigure}
\begin{subfigure}[b]{.32\linewidth}
\includegraphics[width=\linewidth]{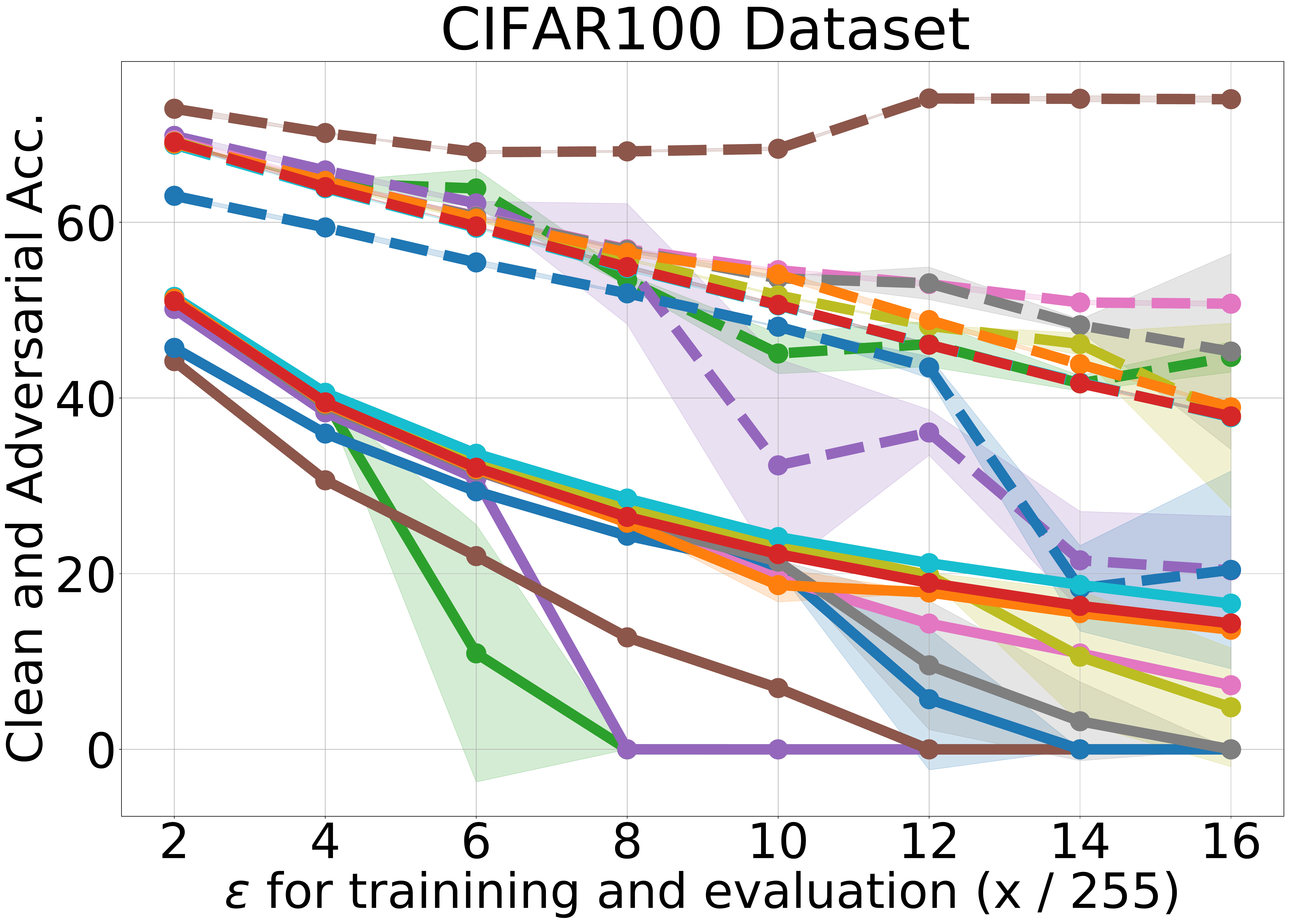}
\end{subfigure}
\begin{subfigure}[b]{.32\linewidth}
\includegraphics[width=\linewidth]{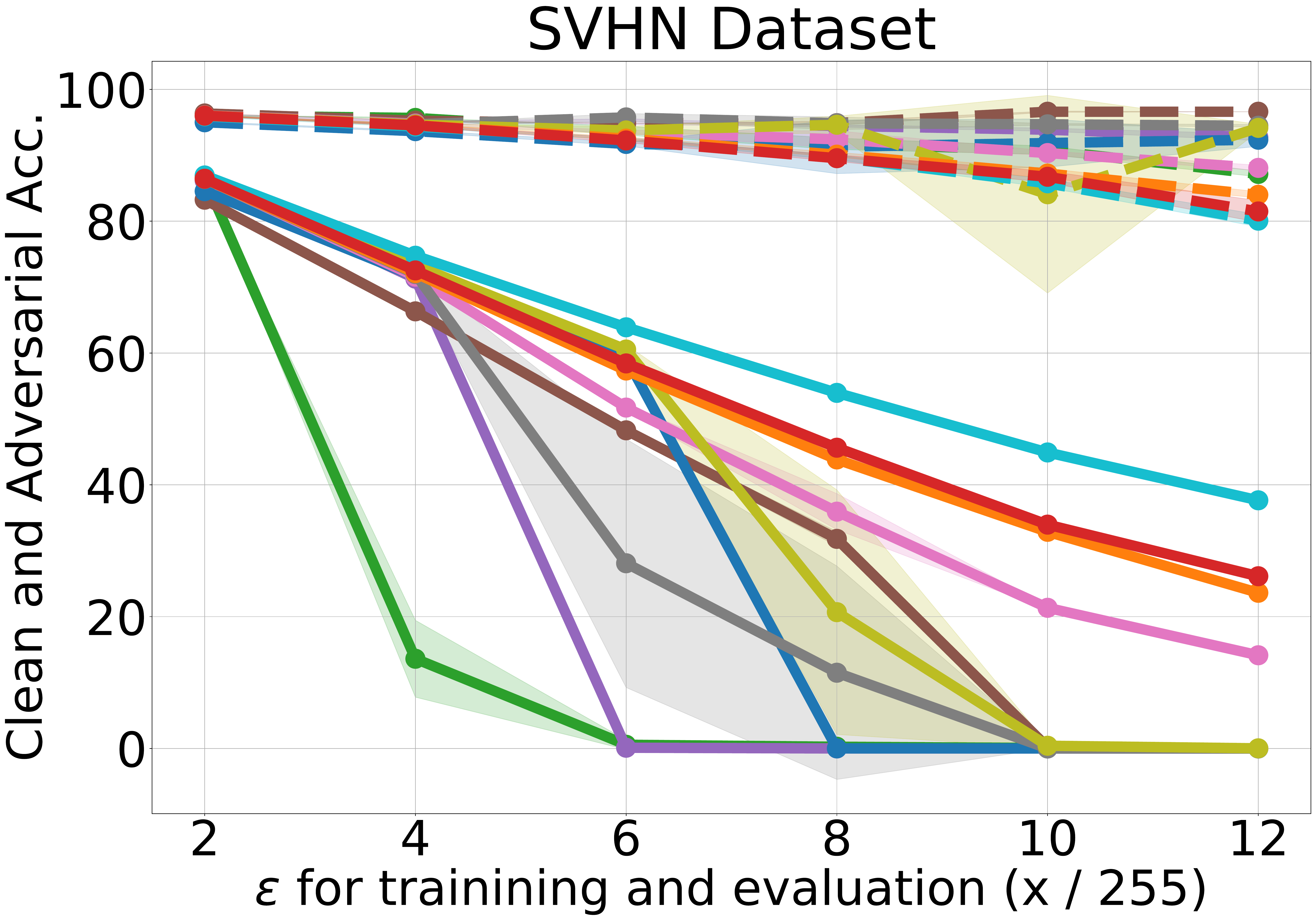}
\end{subfigure}
\caption{Comparison of all methods on CIFAR-10, CIFAR-100 and SVHN with PreactResNet18 over different perturbation radius ($\epsilon$ is divided by 255). We plot both the robust (solid line) and the clean (dashed line) accuracy for each method. Our method, N-FGSM, is able to match or surpass the state-of-the-art single-step method GradAlign while \textit{reducing the cost by a $3\times$ factor}. Adversarial accuracy is based on PGD-50-10 and experiments are averaged over 3 seeds. Legend is shared among all plots.}
\label{figure:comparison_all}
\end{figure}

\section{Experiments with WideResNet28-10 architecture} \label{sec:wideresnet_appendix}
In this section we present the plots of our experiments with WideResNet28-10. We report the results in two figures. In~\cref{figure:wideresnet_single_step} we compare all single-step methods and we do not plot the clean accuracy for better visualization. In~\cref{figure:wideresnet_all} we plot all methods, including multi-step methods, and report the clean accuracy as well with dashed lines. Since we observed that our baseline, RandAlpha, outperformed \cite{AAAI} in all settings for PreActResNet18, we only report RandAlpha for WideResNet. As mentioned in the main paper, we observe that CO seems to be more difficult to prevent for WideResNet. In particular, for GradAlign we observed the regularizer hyperparameter settings proposed by \cite{grad_align} for CIFAR-10 (searched for a PreActResNet18) worked well. However, those parameters led to CO for $6 \leq \epsilon \leq 12$ in CIFAR-100. Since $\epsilon = 14,\ 16$ did not show CO, we increased the GradAlign regularizer hyperparameter $\lambda$ for CIFAR-100 so that each $6 \leq \epsilon \leq 12$ would have the default value corresponding to $\epsilon + 2$, for instance, $\lambda$ for $\epsilon=6$ would be the default $\lambda$ in \cite{grad_align} for $\epsilon = 8$. 

\begin{figure}[hb]
\centering
\begin{subfigure}[b]{.32\linewidth}
\includegraphics[width=\linewidth]{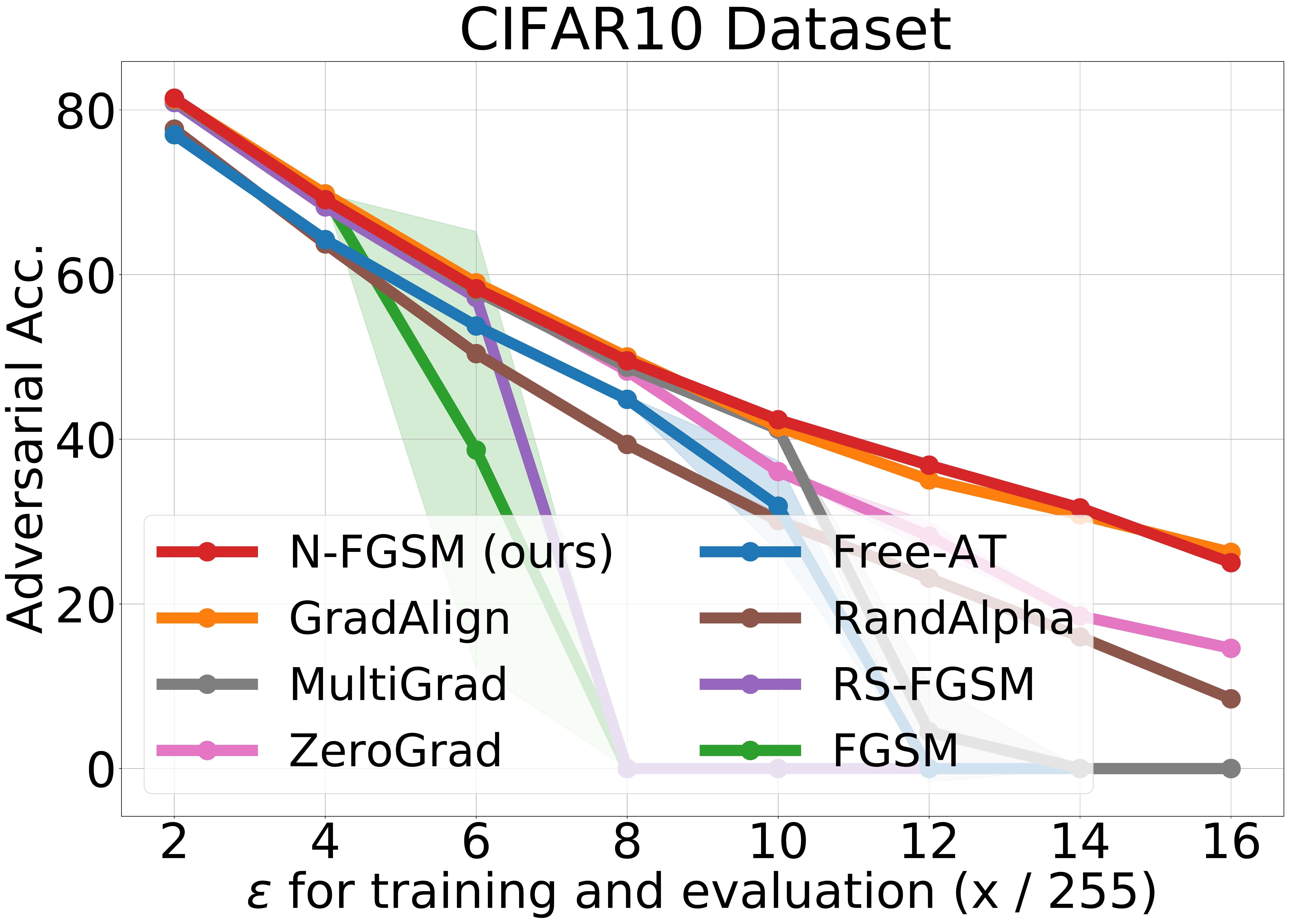}
\end{subfigure}
\begin{subfigure}[b]{.32\linewidth}
\includegraphics[width=\linewidth]{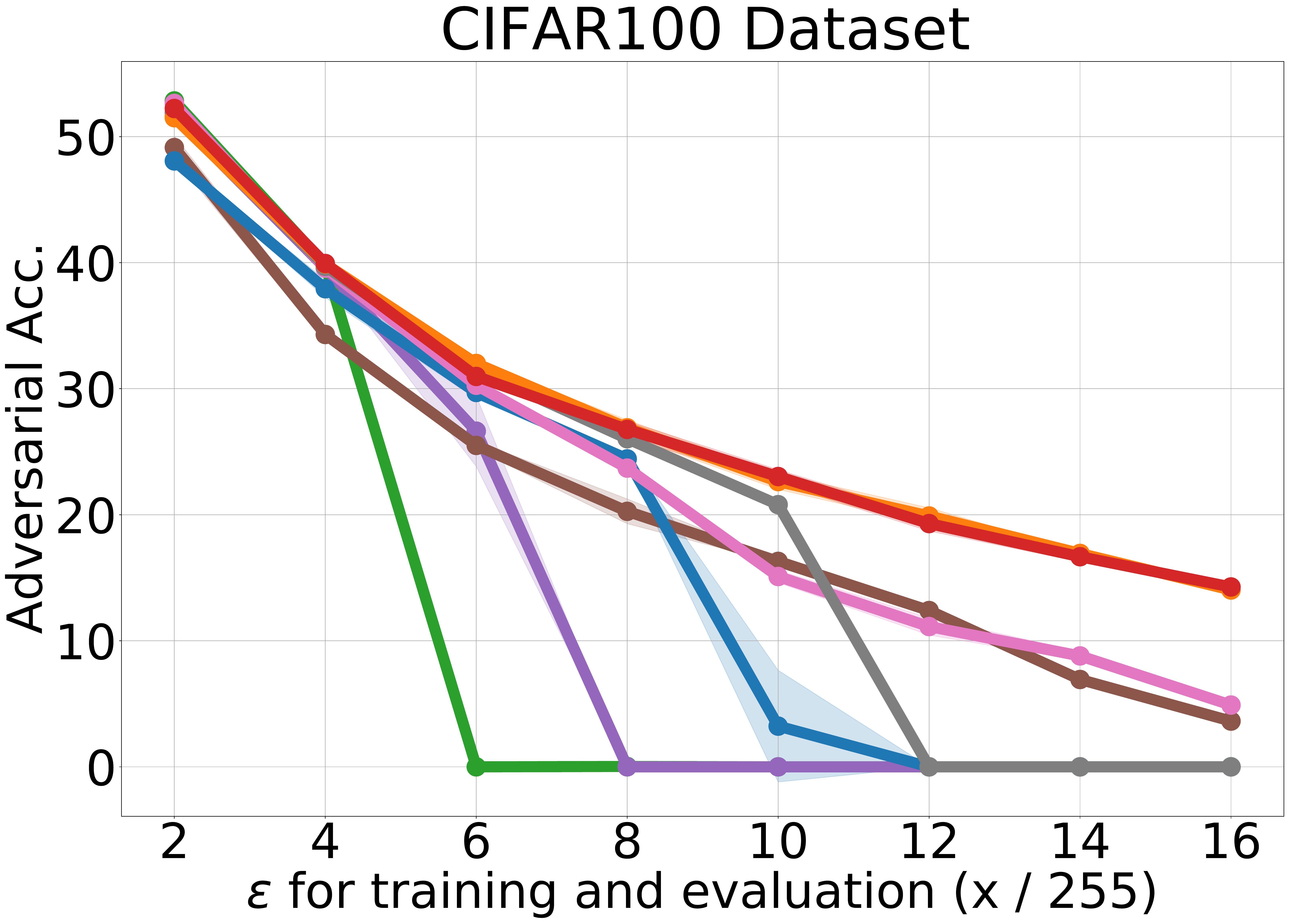}
\end{subfigure}
\begin{subfigure}[b]{.32\linewidth}
\includegraphics[width=\linewidth]{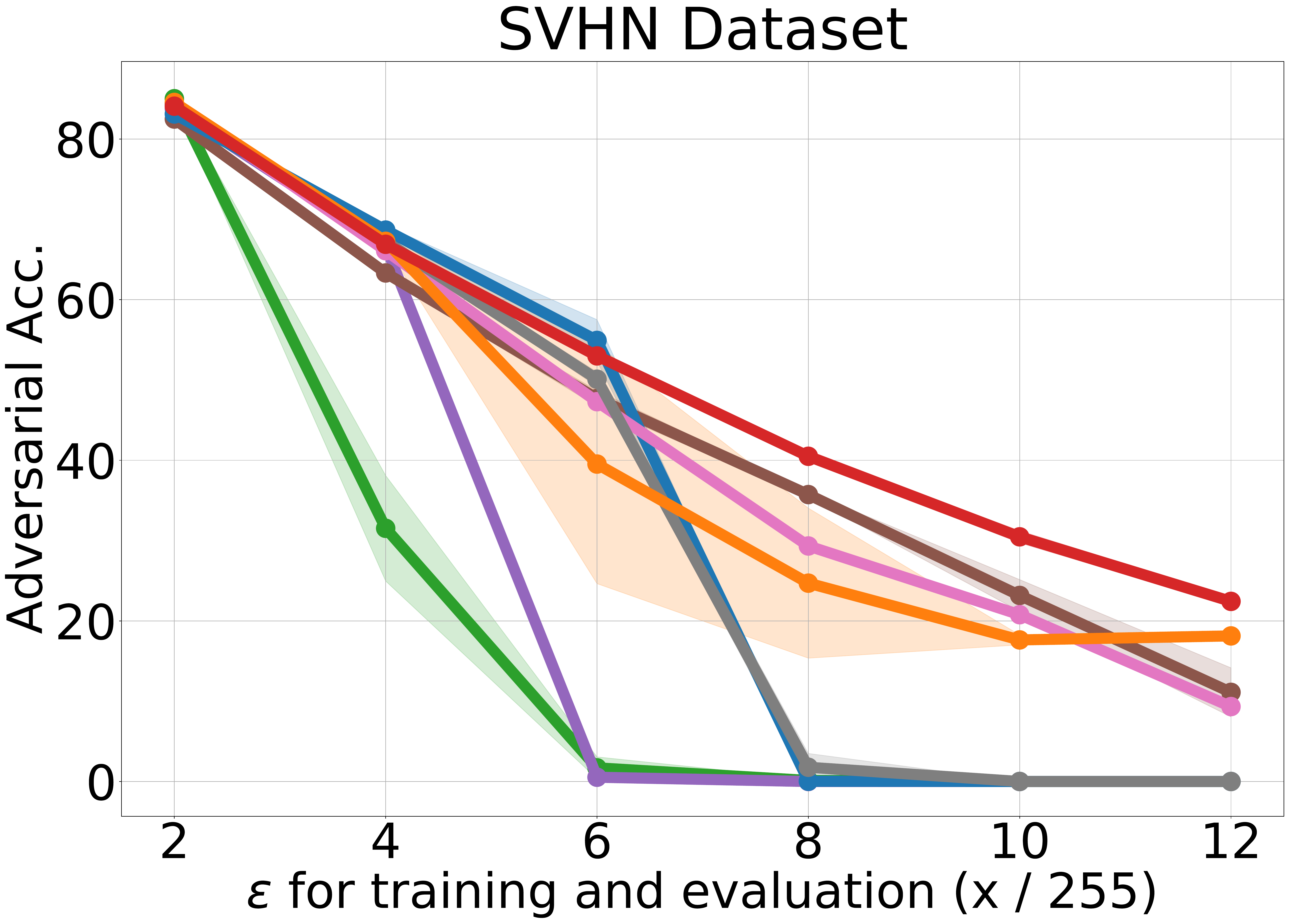}
\end{subfigure}
\caption{Comparison of single-step methods on CIFAR-10, CIFAR-100 and SVHN with WideResNet28-10 over different perturbation radius ($\epsilon$ is divided by 255). Our method, N-FGSM, is able to match or surpass the state-of-the-art single-step method GradAlign while \textit{reducing the cost by a $3\times$ factor}. Moreover, we could not find any competitive hyperparameter setting for GradAlign for $\epsilon \geq 6$ in SVHN dataset. Adversarial accuracy is based on PGD-50-10 and experiments are averaged over 3 seeds. Legend is shared among all plots.}
\label{figure:wideresnet_single_step}
\end{figure} 

For SVHN we observed that the default values for $\lambda$ led to models close to a constant classifier for $\epsilon \geq 6$. We tried to increase the lambda for those $\epsilon$ values to $1.25 \lambda$ but observed the same result. Since the model did not show typical CO but rather it seemed as it was underfitting, we tried to reduce the step-size to $\alpha = 0.75 \epsilon$ and also both decreasing $\alpha$ and increasing $\lambda$. When reducing the step size we obtain accuracies above those of a constant classifier for some radii, however, some or all seeds converge to a constant classifier for each setting, hence the large standard deviations. For N-FGSM, the default configuration of N-FGSM ($\alpha=\epsilon,\ k=2\epsilon$) works well in all settings except for $\epsilon=16$ on CIFAR-10 and $\epsilon = 10,\ 12$ on SVHN. For CIFAR-10, we increase the noise magnitude to $k = 4\epsilon$. For SVHN we find that decreasing $\alpha$ as we tried for GradAlign works better than increasing the noise. We use $\alpha = 8$ for both $\epsilon$ radii. Exact numbers for all the curves are in~\cref{sec:detailed_results}

\begin{figure}[ht]
\centering
\begin{subfigure}[b]{.32\linewidth}
\includegraphics[width=\linewidth]{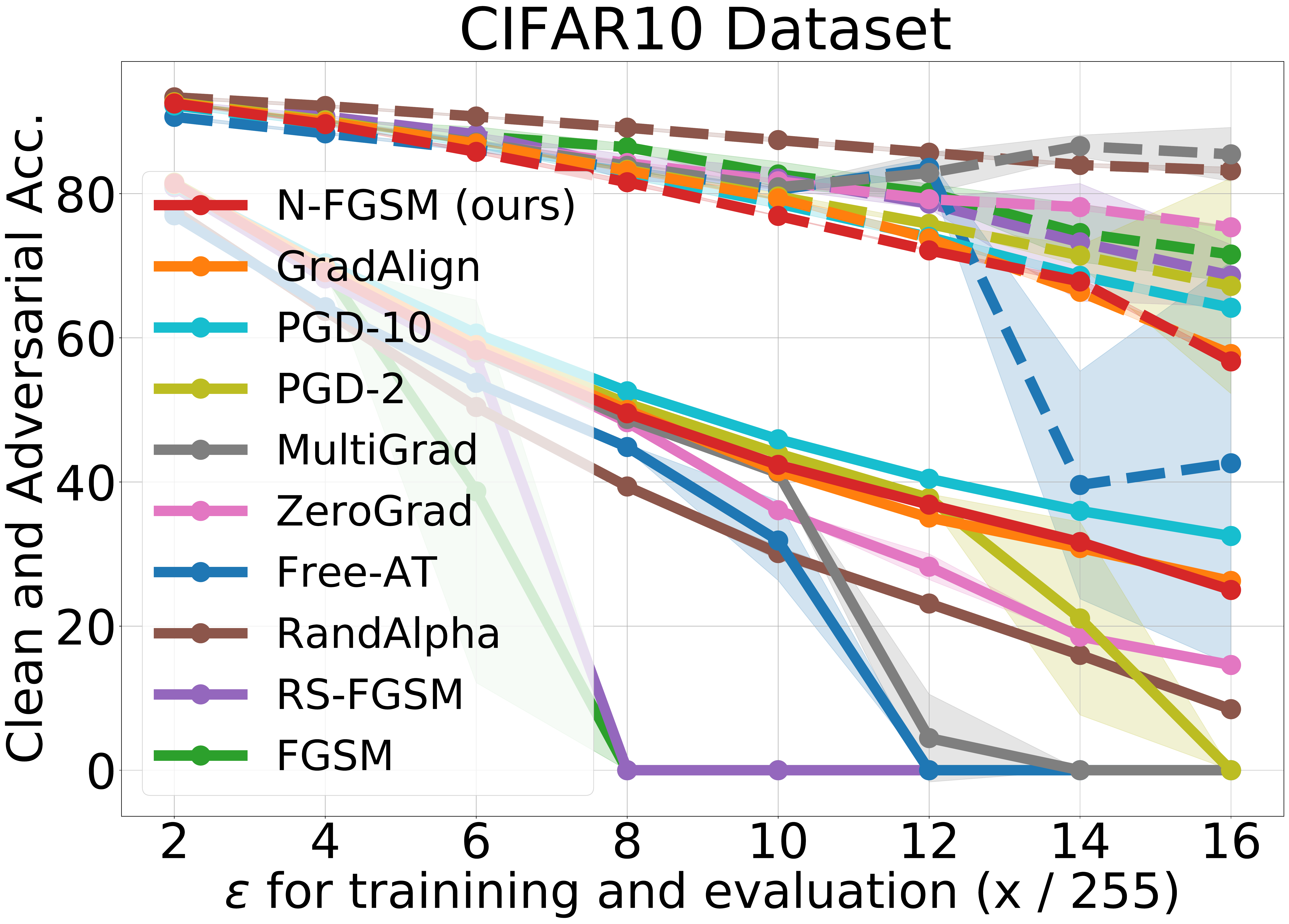}
\end{subfigure}
\begin{subfigure}[b]{.32\linewidth}
\includegraphics[width=\linewidth]{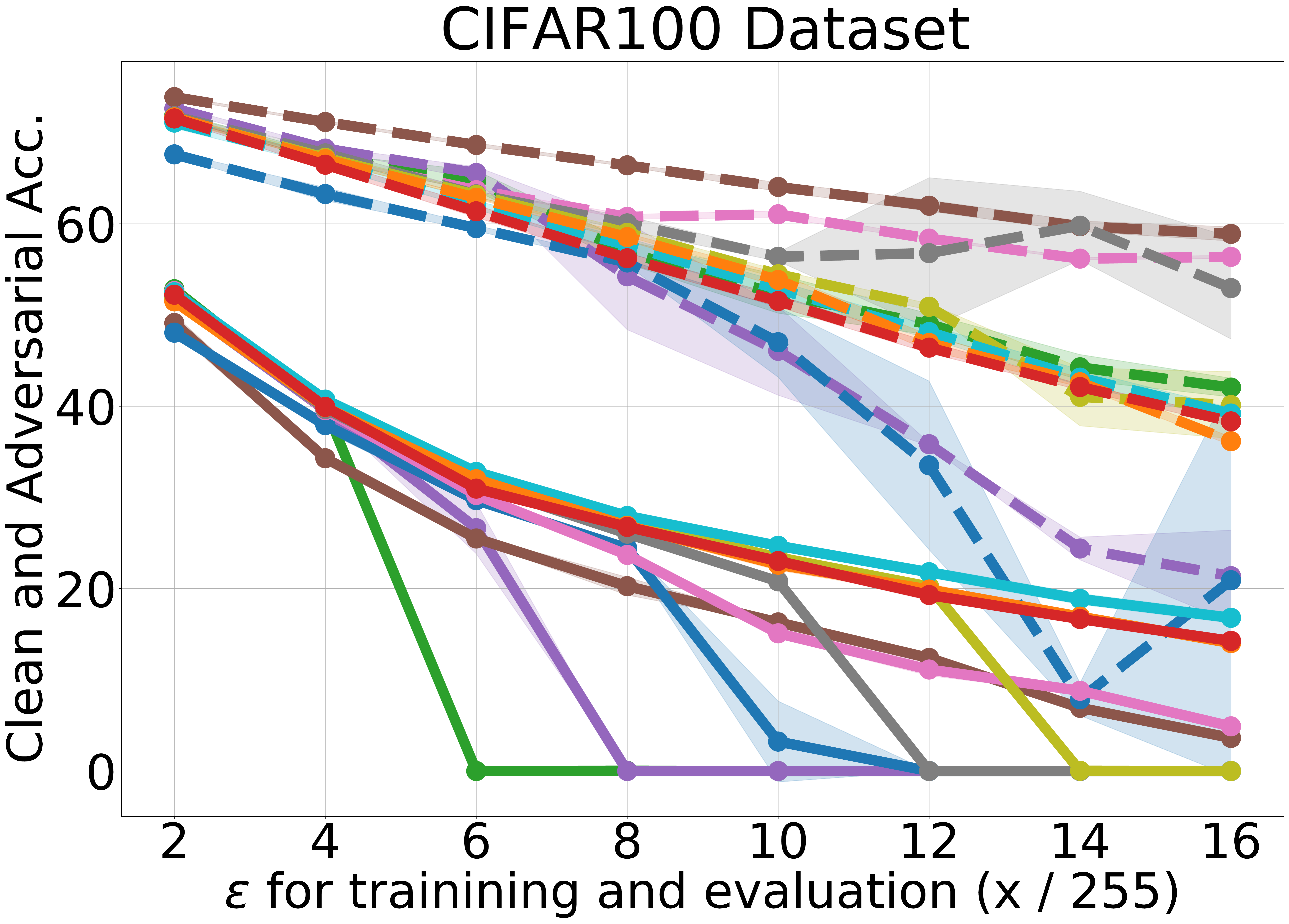}
\end{subfigure}
\begin{subfigure}[b]{.32\linewidth}
\includegraphics[width=\linewidth]{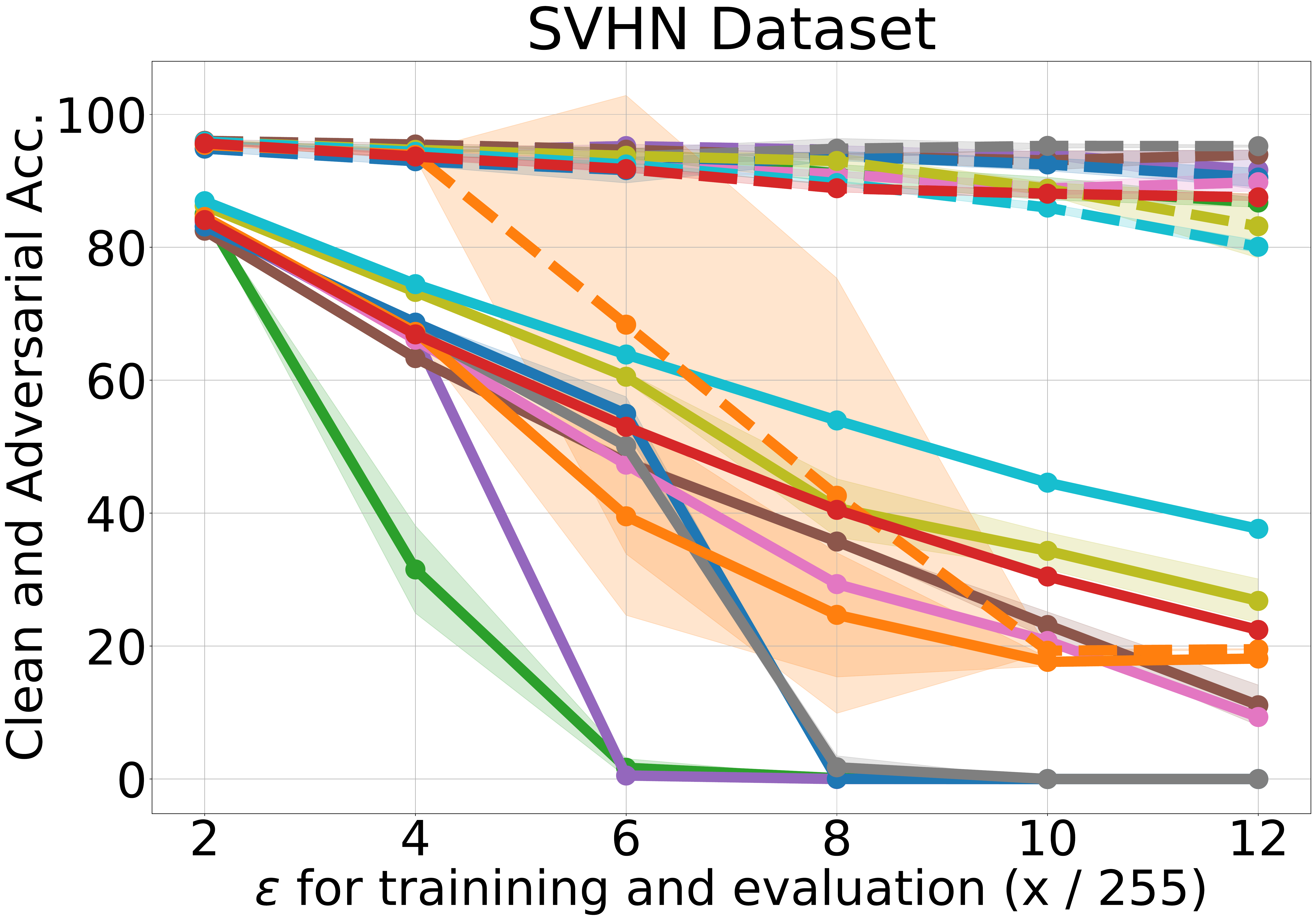}
\end{subfigure}
\caption{Comparison of all methods on CIFAR-10, CIFAR-100 and SVHN with WideResNet28-10 over different perturbation radius ($\epsilon$ is divided by 255). We plot both the robust (solid line) and the clean (dashed line) accuracy for each method. Legend is shared among all plots.}
\label{figure:wideresnet_all}
\end{figure}

\section{Increasing adversarial perturbations during training}\label{sec:inc-adv-pert}

As mentioned in the main paper, N-FGSM perturbations have $\ell_{\infty}-$norm larger than $\epsilon$, see \cref{sec:l2_norm}. In~\cref{sec:increased_perturb} we have seen that the benefits of N-FGSM can not be reproduced by simply increasing $\alpha$ without increasing the noise. However, for the sake of completeness, we also ablate other single-step baselines by using a larger $\epsilon$ during training \ie $\{\epsilon = 8/255, \epsilon = 12/255, \epsilon = 16/255\}$ while testing with a fixed $\epsilon=8/255$ on CIFAR-10. Results are presented in \cref{table:increased_epsilon}. We observe that increasing $\epsilon_{\textrm{train}}$ seems to lead to a decrease in robustness for most methods, \eg PGD-50-10 accuracy for RS-FGSM goes from $46.08 \pm 0.18$ when training with $\epsilon=8/255$ to $0.0 \pm 0.0$ with $\epsilon=12/255$. In two cases (GradAlign and MultiGrad) we observe a small increase, highest increase is for GradAlign which goes from $48.14 \pm 0.15$ to $50.6 \pm 0.45$, however, the clean accuracy drops from $81.9 \pm 0.22$ to $73.29 \pm 0.23$. This is similar to increasing $\alpha$ for N-FGSM (see \cref{figure:exp2}~(C)). However, this is tied to a significant degradation of clean accuracy. All in all, taking into account both clean and robust accuracy we conclude all baselines perform best without increasing the training $\epsilon$. All ablation results are presented in \cref{table:increased_epsilon}.

\vspace{10pt}
\input{Tables/tab_ablation_singlestep}

\section{Longer training schedule}\label{sec:long_schedule}
In our experiments, we have followed the ``fast'' training schedule introduced by \cite{RS-FGSM}. However, \cite{overfitting} suggest that a longer training schedule coupled with early stopping may lead to a boost in performance. We also use the long training schedule for N-FGSM and observe that it does not lead to CO. In~\cref{table:long_schedule} we compare the performance of N-FGSM and GradAlign for the long training schedule. We observe that GradAlign does not seem to benefit from the long training schedule. On the other hand, although N-FGSM seems to obtain a slight increase in performance, the ``fast'' schedule provides comparable performance. It is worth mentioning that for GradAlign, the default regularizer hyperparameter for $\epsilon=\nicefrac{8}{255}$ and CIFAR-10 ($\lambda=0.2$) does not prevent CO. We do a hyperparameter search and keep the value with the largest final robust accuracy ($\lambda = 0.632$).


\begin{table}[h]
\caption{Comparison of ``long'' \citep{overfitting} and ``fast'' \citep{RS-FGSM} training schedules for N-FGSM and GradAlign. GradAlign does not seem to benefit from the long training schedule. Although N-FGSM seems to obtain a slight increase in performance, the ``fast'' schedule provides comparable performance.}
\label{table:long_schedule}
\vspace{4pt}
\renewcommand{\arraystretch}{1.2}
\centering
\small{
\begin{tabular}{@{}ccll@{}}
\toprule
\multicolumn{2}{c}{\textbf{N-FGSM}}      & \multicolumn{2}{c}{\textbf{Grad Align}}                                          \\ \midrule
\textbf{Clean Acc} & \textbf{Robust Acc} & \multicolumn{1}{c}{\textbf{Clean Acc}} & \multicolumn{1}{c}{\textbf{Robust Acc}} \\ \midrule
\multicolumn{4}{c}{\textbf{Long schedule: Final model}}                                                                              \\ \midrule
\textbf{83.18 $\pm$ 0.11}   & 36.56 $\pm$ 0.26    & \textbf{84.13 $\pm$ 0.24}                          & 36.17 $\pm$ 0.19                           \\ \midrule
\multicolumn{4}{c}{\textbf{Long schedule: Best model}}                                                                               \\ \midrule
80.8 $\pm$ 0.36    & \textbf{48.48 $\pm$ 0.27}    & 81.57 $\pm$ 0.44                          & 47.86 $\pm$ 0.1                            \\ \midrule
\multicolumn{4}{c}{\textbf{fast schedule: Final model}}                                                                             \\ \midrule
80.58 $\pm$ 0.22   & 48.12 $\pm$ 0.07    & 81.9 $\pm$ 0.22                           & \textbf{48.14 $\pm$ 0.15}                           \\ \bottomrule
\end{tabular}
}

\end{table}

In \cref{table:long_schedule} we observe that the performance of the final model is lower than that of an early stopped method. This could be expected due to the phenomena of robust overfitting described in \cite{overfitting}. However, as a sanity check to make sure that this is not due to a hidden CO during the long schedule which somehow the model recovers from we plot the full training history in \cref{figure:long_schedule_history}. There we can observe that \textit{for N-FGSM there is no CO during training}.We also show FGSM (which is well known has CO for $\epsilon = \nicefrac{8}{255}$) for comparison.

\begin{figure}[h]
\centering
\begin{subfigure}[b]{.49\linewidth}
\includegraphics[width=\linewidth]{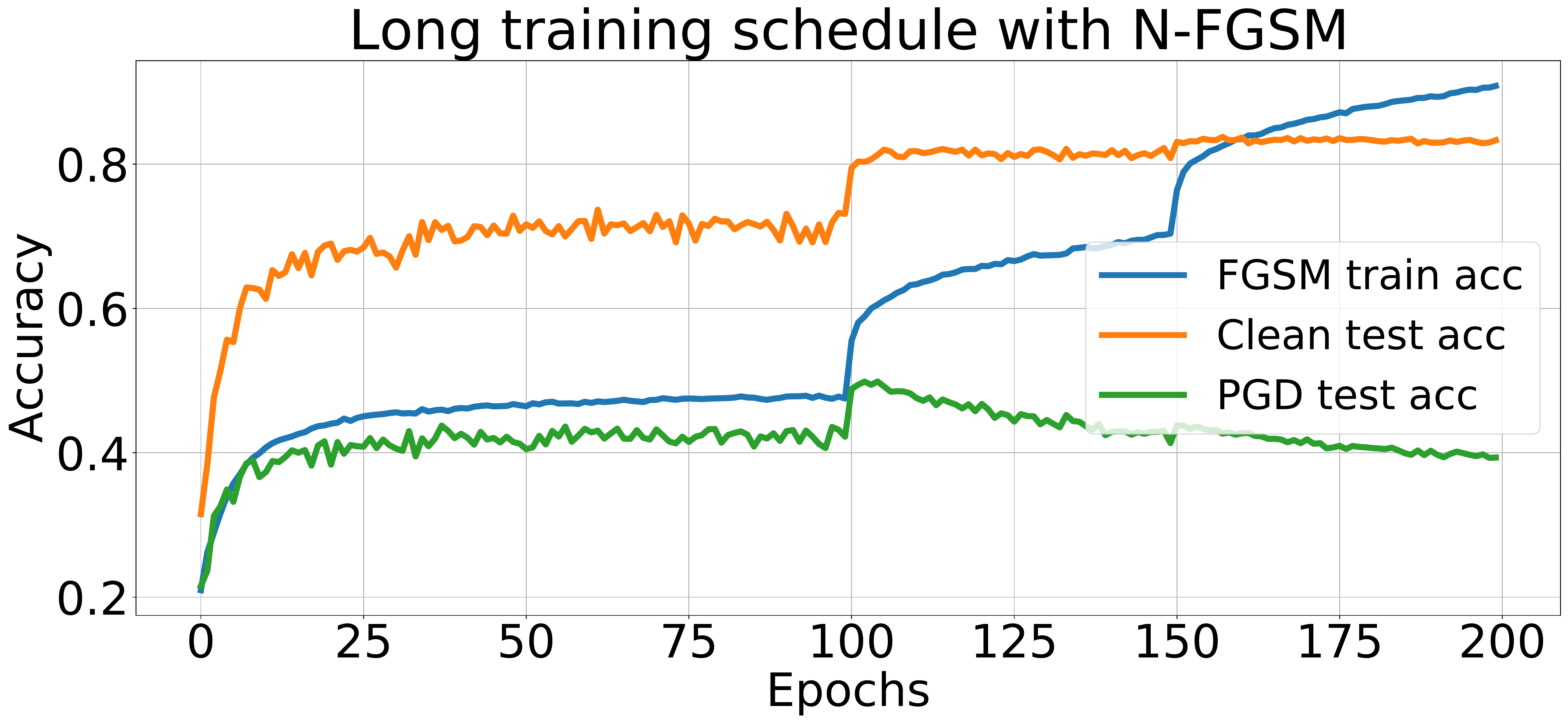}
\end{subfigure}
\begin{subfigure}[b]{.49\linewidth}
\includegraphics[width=\linewidth]{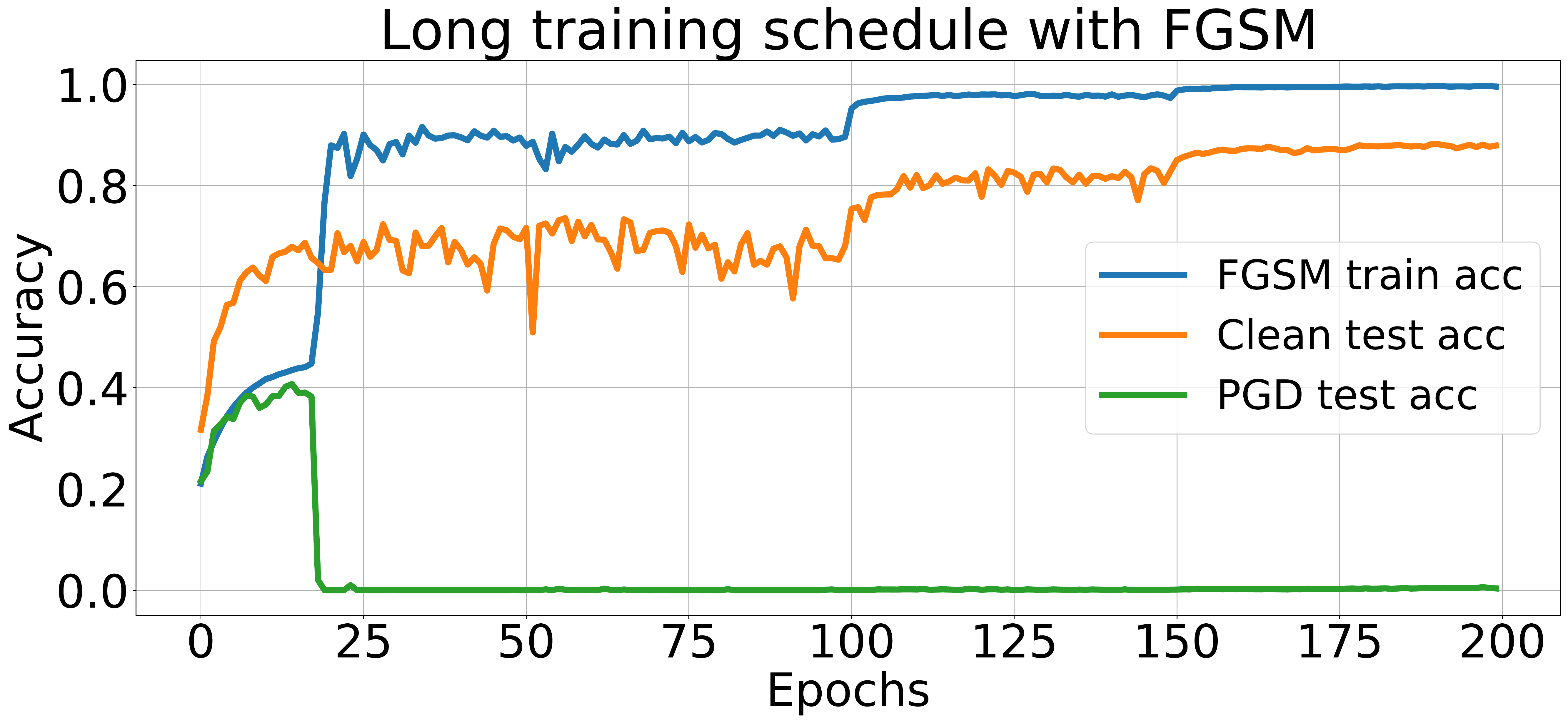}
\end{subfigure}
\caption{Training and test accuracy during the long training schedule proposed in \cite{overfitting}. We observe that N-FGSM (left) does not present CO at any point during training, however suffers from robust overfitting as described in \cite{overfitting} which suggested selecting the best validated model as a simple and yet effective way to improve robustness. On the other hand FGSM (right) suffers from CO where the robustness drops suddenly to 0 and does not recover.}
\label{figure:long_schedule_history}
\end{figure}

\section{Randomized Alpha}  \label{sec:RandAlpha_appendix} \citet{AAAI} evaluate intermediate points along the RS-FGSM direction in order to pick the ``optimal'' perturbation size. However, we find that increasing the number of intermediate evaluated points does not necessarily lead to increased adversarial accuracy. Moreover, for large perturbations we could not prevent CO even with twice the number of evaluations tested by \cite{AAAI}. This motivates us to test a very simple baseline where instead of evaluating intermediate steps, the RS-FGSM perturbation size is randomly selected as: $\delta = t \cdot \delta_{\textrm{RS-FGSM}}$ where $t \sim \mathcal{U}[0, 1]^d$. Interestingly, as reported in \cref{figure:rand_alpha_all}, we find that this very simple baseline, dubbed \textit{RandAlpha}, is able to avoid CO for all values of $\epsilon$ and outperforms \cite{AAAI} on CIFAR-10, CIFAR-100 and SVHN. This is aligned with our main finding that combining noise with adversarial attacks is indeed a powerful tool that should be explored more thoroughly before developing more expensive solutions.

\begin{figure}[h]
\centering
\begin{subfigure}[b]{.27\linewidth}
\includegraphics[width=\linewidth]{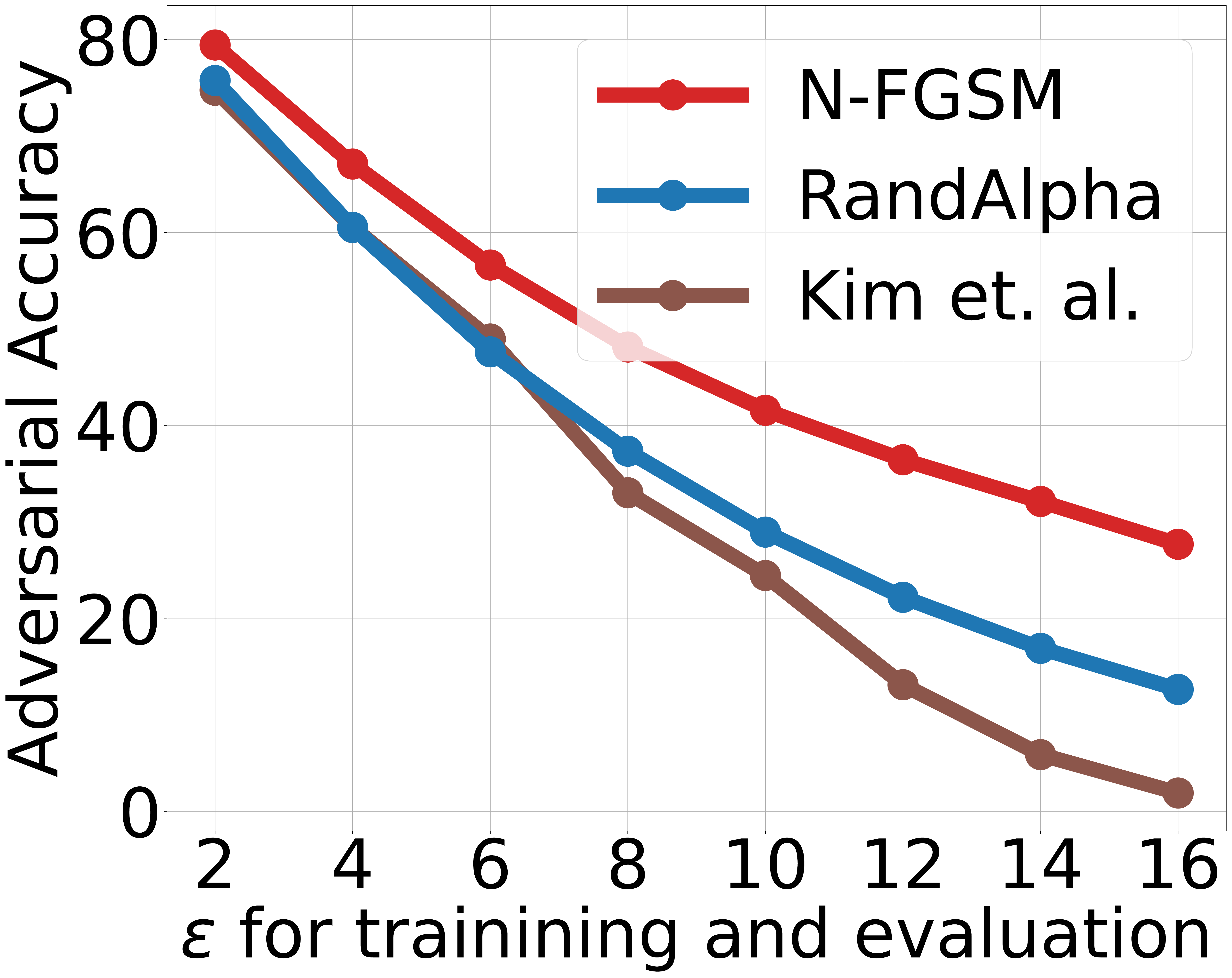}
\end{subfigure}
\hspace{10pt}
\begin{subfigure}[b]{.27\linewidth}
\includegraphics[width=\linewidth]{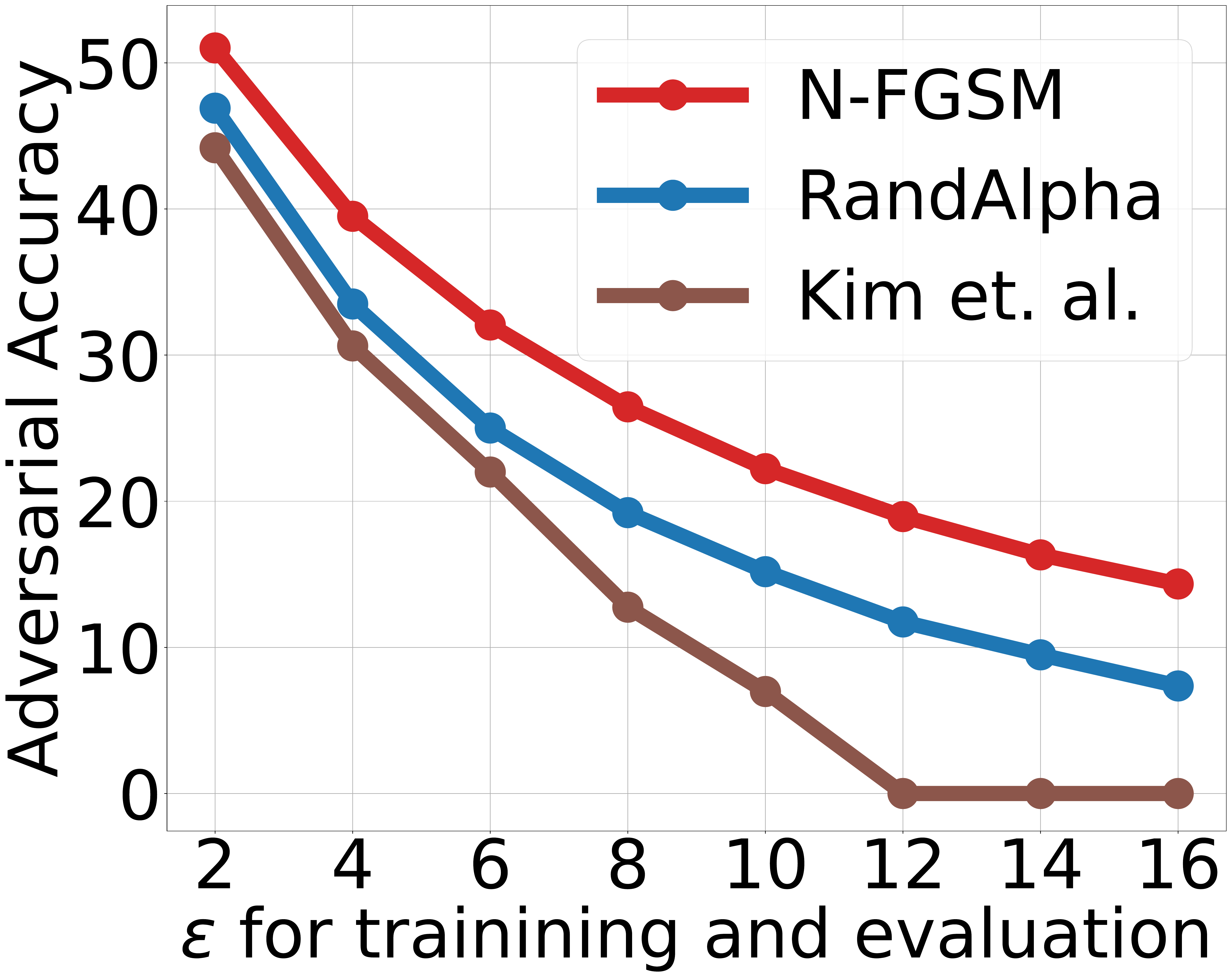}
\end{subfigure}
\hspace{10pt}
\begin{subfigure}[b]{.27\linewidth}
\includegraphics[width=\linewidth]{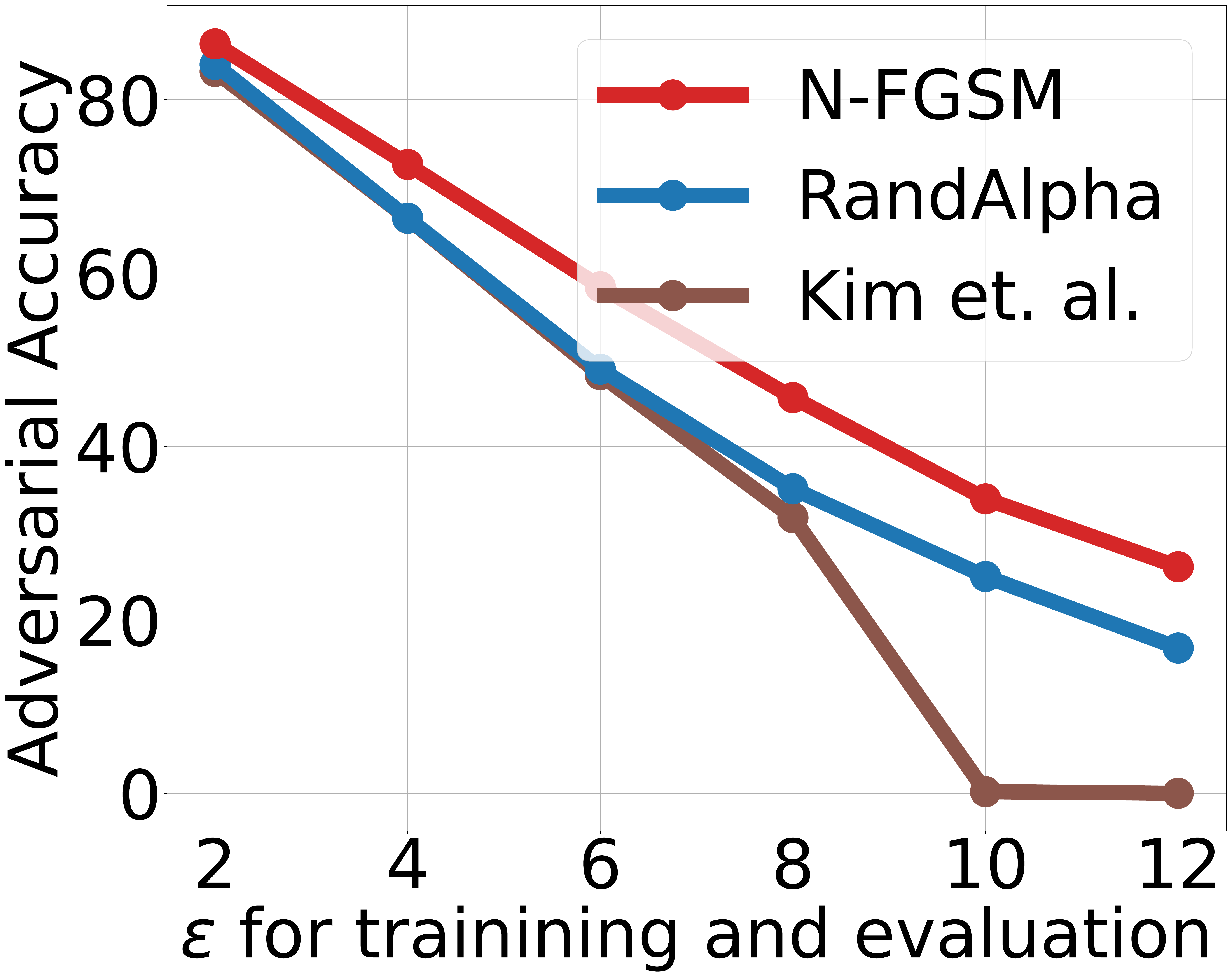}
\end{subfigure}
\caption{Comparison of \cite{AAAI} with RandomAlpha, our baseline where we multiply the RS-FGSM perturbation by a scalar uniformly sampled in $[0, 1]$. We present results on CIFAR-10 (Left), CIFAR-100 (Middle) and SVHN (Right) with PreActResNet18.}
\label{figure:rand_alpha_all}
\end{figure}

\clearpage

\section{Further visualizations of adversarial perturbations and gradients}
In this section we present an extension of \cref{figure:grad_norm_delta_rank} with further examples. As observed in the main paper, early in training adversarial perturbations ($\delta$) and gradients are consistent across epochs, however, after CO they become hard to interpret. Note that although we label rows as either pre-CO or post-CO we only observe CO for FGSM and RS-FGSM. Both PGD-10 and N-FGSM obtain robust models as shown in detail in the paper.

\vspace {10pt}

\begin{figure}[h]
\centering
\begin{subfigure}[b]{0.99\linewidth}
\includegraphics[width=\linewidth]{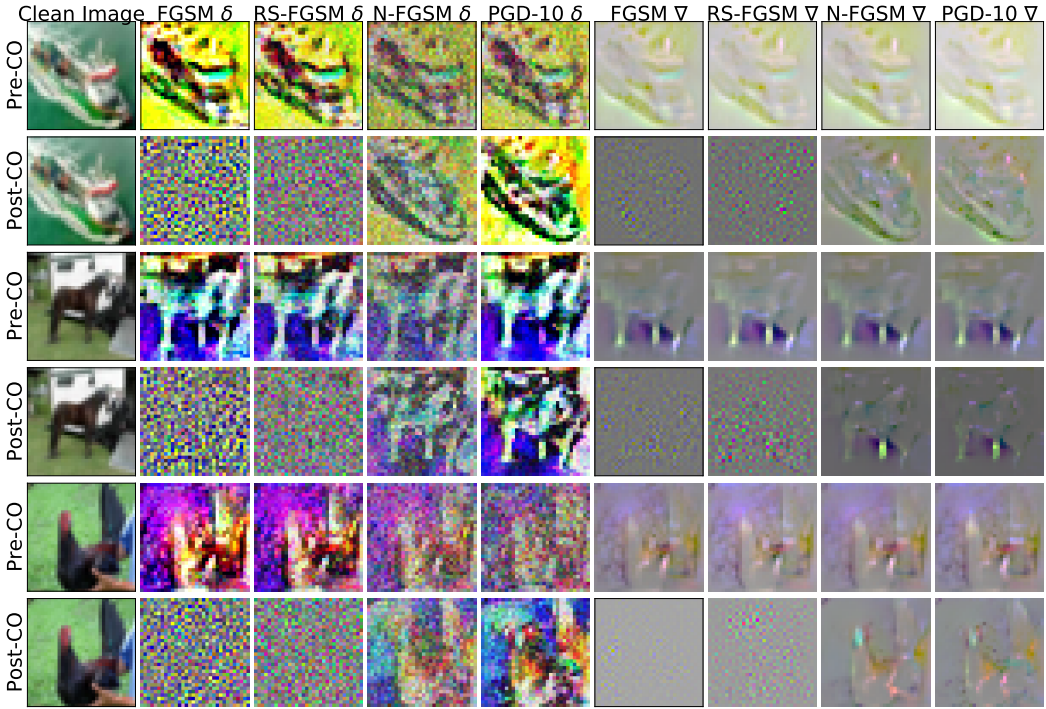}
\end{subfigure}
\caption{Visualization of adversarial perturbations ($\delta$'s) and gradients averaged across several epochs before CO (pre-CO) and after (post-CO). Note that only FGSM and RS-FGSM present CO, PGD-10 and N-FGSM do not. Post-CO, FGSM and RS-FGSM obtain $\delta$'s that are hard to interpret, idem for their gradients.
}
\label{figure:deltas_grads_viz_extended}
\end{figure}

\section{Robust evaluations with autoattack} \label{sec:autoattack}

\begin{table}[h]
\caption{Clean (top) and robust accuracy (bottom) for CIFAR-10 and PreacResNet18 evaluated with autoattack (AA) \cite{croce2020reliable}. We find the same trend as with PGD50-10.}
\vspace{2pt}
\label{table:autoattack}
\renewcommand{\arraystretch}{1.5}
\centering
\scriptsize{
\begin{tabular}{c|c|c|c|c|c|c|c|c}
\toprule
 &
  $\epsilon=\nicefrac{2}{255}$ &
  $\epsilon=\nicefrac{4}{255}$ &
  $\epsilon=\nicefrac{6}{255}$ &
  $\epsilon=\nicefrac{8}{255}$ &
  $\epsilon=\nicefrac{10}{255}$ &
  $\epsilon=\nicefrac{12}{255}$ &
  $\epsilon=\nicefrac{14}{255}$ &
  $\epsilon=\nicefrac{16}{255}$ \\ \hline \hline
FGSM &
  \begin{tabular}[c]{@{}l@{}}91.52 $\pm$ 0.08\\ 78.99 $\pm$ 0.19\end{tabular} &
  \begin{tabular}[c]{@{}l@{}}88.59 $\pm$ 0.08\\ 65.99 $\pm$ 0.24\end{tabular} &
  \begin{tabular}[c]{@{}l@{}}85.17 $\pm$ 0.03\\ 54.0 $\pm$ 0.32\end{tabular} &
  \begin{tabular}[c]{@{}l@{}}86.62 $\pm$ 0.08\\ 0.0 $\pm$ 0.0\end{tabular} &
  \begin{tabular}[c]{@{}l@{}}83.35 $\pm$ 2.03\\ 0.0 $\pm$ 0.0\end{tabular} &
  \begin{tabular}[c]{@{}l@{}}78.51 $\pm$ 3.3\\ 0.0 $\pm$ 0.0\end{tabular} &
  \begin{tabular}[c]{@{}l@{}}77.31 $\pm$ 1.9\\ 0.0 $\pm$ 0.0\end{tabular} &
  \begin{tabular}[c]{@{}l@{}}75.88 $\pm$ 1.49\\ 0.0 $\pm$ 0.0\end{tabular} \\ \hline
GradAlign &
  \begin{tabular}[c]{@{}l@{}}91.48 $\pm$ 0.08\\ 79.09 $\pm$ 0.21\end{tabular} &
  \begin{tabular}[c]{@{}l@{}}88.55 $\pm$ 0.18\\ 65.65 $\pm$ 0.13\end{tabular} &
  \begin{tabular}[c]{@{}l@{}}85.23 $\pm$ 0.22 \\ 53.99 $\pm$ 0.2\end{tabular} &
  \begin{tabular}[c]{@{}l@{}}81.69 $\pm$ 0.1\\ 44.11 $\pm$ 0.34\end{tabular} &
  \begin{tabular}[c]{@{}l@{}}77.73 $\pm$ 0.18\\ 35.72 $\pm$ 0.34\end{tabular} &
  \begin{tabular}[c]{@{}l@{}}73.46 $\pm$ 0.16\\ 28.66 $\pm$ 0.15\end{tabular} &
  \begin{tabular}[c]{@{}l@{}}67.87 $\pm$ 0.5\\ 22.93 $\pm$ 0.33\end{tabular} &
  \begin{tabular}[c]{@{}l@{}}61.66 $\pm$ 0.32\\ 18.4 $\pm$ 0.28\end{tabular} \\ \hline
N-FGSM &
  \begin{tabular}[c]{@{}l@{}}91.44 $\pm$ 0.09\\ 78.99 $\pm$ 0.17\end{tabular} &
  \begin{tabular}[c]{@{}l@{}}88.36 $\pm$ 0.04\\ 66.06 $\pm$ 0.25\end{tabular} &
  \begin{tabular}[c]{@{}l@{}}84.56 $\pm$ 0.12\\ 53.94 $\pm$ 0.3\end{tabular} &
  \begin{tabular}[c]{@{}l@{}}80.36 $\pm$ 0.03\\ 44.36 $\pm$ 0.26\end{tabular} &
  \begin{tabular}[c]{@{}l@{}}75.81 $\pm$ 0.22\\ 36.73 $\pm$ 0.27\end{tabular} &
  \begin{tabular}[c]{@{}l@{}}71.03 $\pm$ 0.16\\ 30.45 $\pm$ 0.2\end{tabular} &
  \begin{tabular}[c]{@{}l@{}}66.49 $\pm$ 0.36\\ 25.08 $\pm$ 0.15\end{tabular} &
  \begin{tabular}[c]{@{}l@{}}62.86 $\pm$ 0.88\\ 19.0 $\pm$ 1.08\end{tabular} \\ \hline
  
\end{tabular}
}
\end{table}

Following previous work, \cite{grad_align, fgsm} we have evaluated robustness with PGD50-10, i.e. PGD with 50 iterations and 10 restarts. However, for the sake of completeness, we also present results of robust accuracy evaluated with autoattack \cite{croce2020reliable}. In \cref{table:autoattack} we evaluate models adversarially trained with our proposed method N-FGSM, the baseline FGSM and GradAlign. We observe the same pattern as with the PGD50-10 attack, therefore we are conviced that our results are general.

\section{Catastrophic Overfitting outside the ResNet family}
 Previous work focusing on CO has only used architectures from the ResNet family. In \cref{table:vgg} we present results for adversarial training with a VGG-16 architecture \cite{vgg}. Similarly to other studied models we observe that FGSM leads to CO while N-FGSM is able to prevent it. However, it seems that FGSM presents CO for slighly larger $\epsilon$ radii, indicating that the architecture might play a role in CO. We consider investigating this further a promising direction of future work. 
 
\begin{table}[h]
\caption{Clean (top) and robust accuracy (bottom) for CIFAR-10 and VGG-16 \cite{vgg} evaluated with PGD50-10. We also observe CO for VGG architecture when trained with FGSM, moreover, N-FGSM is able to prevent CO. Interestingly, for VGG CO happens for slighly large $\epsilon$ values indicating that the architecture might play a role in CO.}
\vspace{2pt}
\label{table:vgg}
\renewcommand{\arraystretch}{1.5}
\centering
\scriptsize{
\begin{tabular}{c|c|c|c|c|c|c|c}
\toprule
 &
  $\epsilon=\nicefrac{4}{255}$ &
  $\epsilon=\nicefrac{6}{255}$ &
  $\epsilon=\nicefrac{8}{255}$ &
  $\epsilon=\nicefrac{10}{255}$ &
  $\epsilon=\nicefrac{12}{255}$ &
  $\epsilon=\nicefrac{14}{255}$ &
  $\epsilon=\nicefrac{16}{255}$ \\ \hline \hline
FGSM &
  \begin{tabular}[c]{@{}l@{}}85.04 $\pm$ 0.1\\ 62.94 $\pm$ 0.07\end{tabular} &
  \begin{tabular}[c]{@{}l@{}}79.34 $\pm$ 0.11 \\ 52.72 $\pm$ 0.12\end{tabular} &
  \begin{tabular}[c]{@{}l@{}}73.39 $\pm$ 0.0\\ 44.0 $\pm$ 0.02\end{tabular} &
  \begin{tabular}[c]{@{}l@{}}82.6 $\pm$ 0.0\\ 0.07 $\pm$ 0.0\end{tabular} &
  \begin{tabular}[c]{@{}l@{}}83.04 $\pm$ 0.0\\ 0.8 $\pm$ 0.0\end{tabular} &
  \begin{tabular}[c]{@{}l@{}}81.4 $\pm$ 0.0\\ 0.25 $\pm$ 0.0\end{tabular} &
  \begin{tabular}[c]{@{}l@{}}80.41 $\pm$ 0.21\\ 0.31 $\pm$ 0.15\end{tabular} \\ \hline
N-FGSM &
  \begin{tabular}[c]{@{}l@{}}84.53 $\pm$ 0.0\\ 63.32 $\pm$ 0.0\end{tabular} &
  \begin{tabular}[c]{@{}l@{}}79.42 $\pm$ 0.0\\ 53.0 $\pm$ 0.0\end{tabular} &
  \begin{tabular}[c]{@{}l@{}}72.01 $\pm$ 0.28\\ 44.3 $\pm$ 0.09\end{tabular} &
  \begin{tabular}[c]{@{}l@{}}66.81 $\pm$ 0.54\\ 38.25 $\pm$ 0.1\end{tabular} &
  \begin{tabular}[c]{@{}l@{}}61.19 $\pm$ 0.0\\ 33.36 $\pm$ 0.0\end{tabular} &
  \begin{tabular}[c]{@{}l@{}}56.97 $\pm$ 0.0\\ 29.23 $\pm$ 0.0\end{tabular} &
  \begin{tabular}[c]{@{}l@{}}53.1 $\pm$ 1.19\\ 25.72 $\pm$ 0.22\end{tabular} \\ \hline
\end{tabular}
}
\end{table}

\section{Further increasing the attack radii}
Following previous work \cite{grad_align} we have studied $\epsilon$ attack radii up to $epsilon = \nicefrac{16}{255}$. Indeed, the performance at these radius is already significantly degraded and thus it would not be very practical for most applications. However, to show that N-FGSM can prevent CO at even larger radii we test two additional radii, $\epsilon= \nicefrac{20}{255}$ and $\epsilon= \nicefrac{24}{255}$. In both cases N-FGSM is able to prevent CO. For $\epsilon= \nicefrac{20}{255}$ we obtain a clean accuracy of 51.63 $\pm$ 0.38  and robust of 20.62 $\pm$ 0.37 while for $\epsilon= \nicefrac{24}{255}$ we obtain a clean accuracy of 40.16 $\pm$ 0.96 and robust of 15.3 $\pm$ 1.49. We argue that it is of little interest to try even larger perturbations unless more effective methods to improve both the clean and robust performance are found.

\section{Testing other norms}

Following previous work, we have focused on the $\ell_\infty$ threat model. Although this is where works studying CO have mainly focused, we observe that CO is also present in other norms such as $\ell_1$ and $\ell_2$. Moreover, in both cases we observe that N-FGSM is able to prevent CO. Interestingly, the range of norms in which we observe CO is usually much higher than normally tested for these norms which would explain why the $\ell_\infty$ norm has been the main focus of study in related works.

\begin{figure}[h]
\centering
\begin{subfigure}[b]{.37\linewidth}
\includegraphics[width=\linewidth]{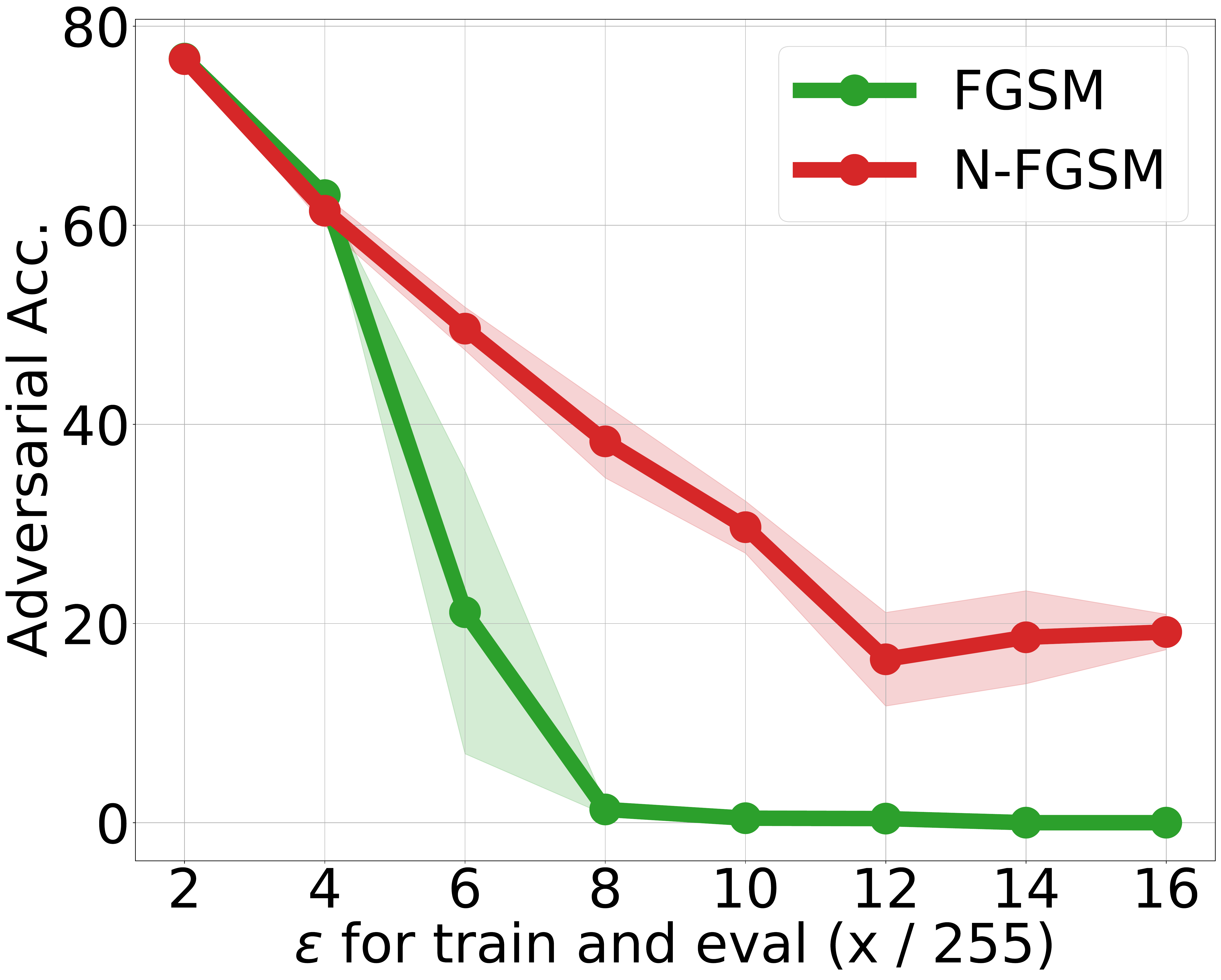}
\end{subfigure}
\hspace{14pt}
\begin{subfigure}[b]{.37\linewidth}
\includegraphics[width=\linewidth]{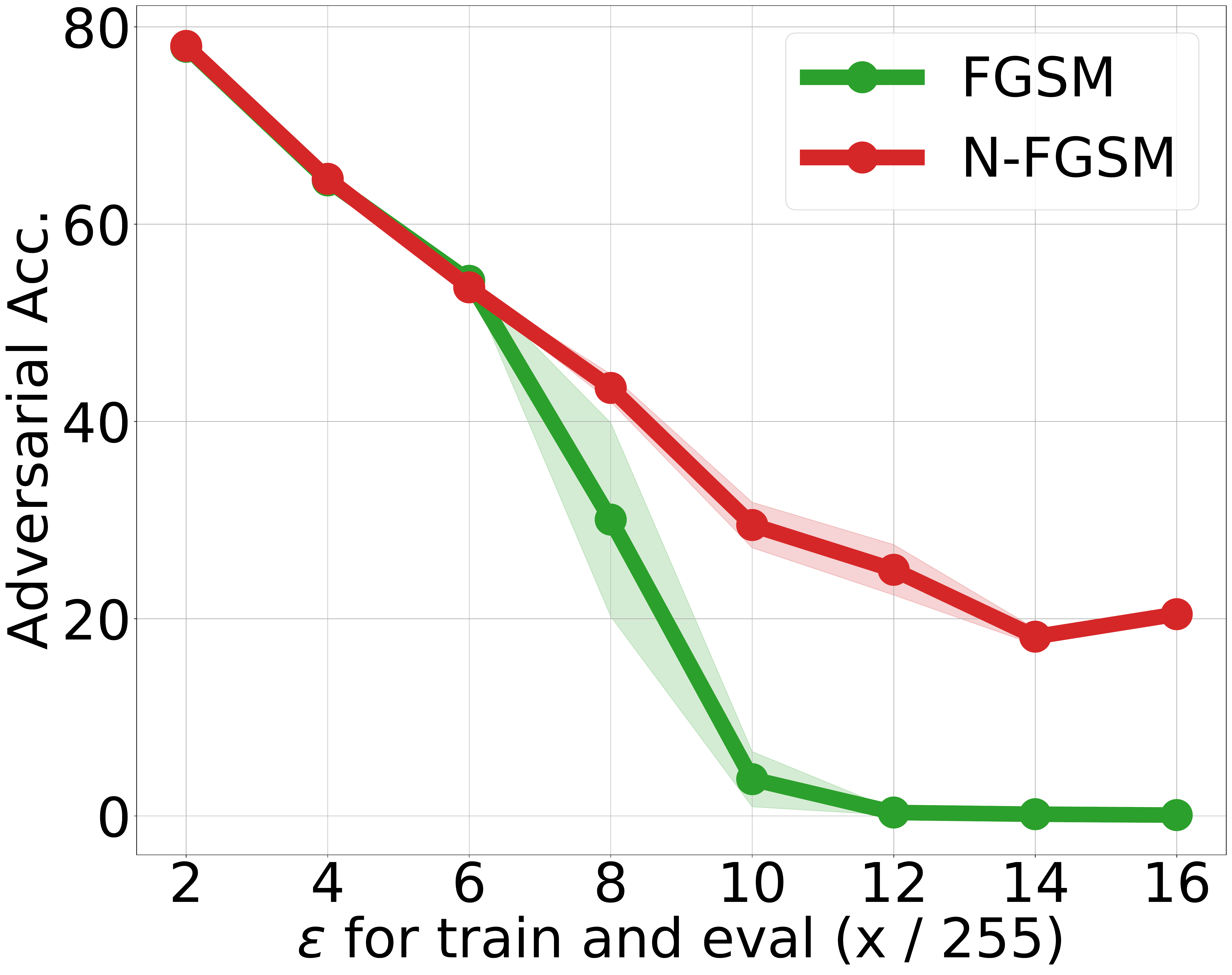}
\end{subfigure}
\caption{Robust accuracy after training with FGSM or N-FGSM using $\ell_1$ (left) and $\ell_2$ (right) perturbations. As observed for $\ell_\infty$ perturbations FGSM leads to CO, while N-FGSM is able to prevent it. Note that the strength of the perturbations is indicated to be equivalent to $\ell_\infty$ perturbations where all pixels have maximum magnitude i.e. $\epsilon=8/255$ indicates perturbations were restricted to an $\ell_p$ norm of a vector where all components are in $\{-\epsilon, +\epsilon\}$. Which would correspond to an $\ell_1$ norm of $n \epsilon$ and an $\ell_2$ norm of $\epsilon \sqrt n$ where $n$ indicates the dimensionality of the input.}
\label{figure:other_norms}
\end{figure}

\section{Combining N-FGSM with additional regularizers} \label{sec:app_gat}

In this section we present the results from \cref{sec:gat} where we combine N-FGSM with additional regularizers \citep{GAT, NuAT} that were proposed for single-step adversarial training to boost the performance. First, we try the proposed methods with the default settings (which use a version of RS-FGSM with Bernoulli noise) and observe they lead to CO for larger $\epsilon$. Then we compare them with N-FGSM + Regularizer where we apply their proposed regularizers to N-FGSM. If we apply GAT or NuAT regularizers to N-FGSM then we do not observe CO and usually a boost in performance. Results are presented in \cref{tab:gat_nuat}.

\begin{table}[h]
\caption{Clean accuracy (top) and PGD50-10 accuracy (bottom) of N-FGSM with additional regularizers introduced in GAT \citep{GAT} and NuAT \citep{NuAT}. Both GAT and NuAT present CO with their default training method. If we apply their proposed regularizers to N-FGSM we can avoid CO while achieving a boost in performance.\looseness=-1}
\label{tab:gat_nuat}
\vspace{2pt}
\renewcommand{\arraystretch}{1.5}
\centering
\scriptsize{
\begin{tabular}{c|c|c|c|c|c|c|c|c}
\toprule
 &
  $\epsilon=\nicefrac{2}{255}$ &
  $\epsilon=\nicefrac{4}{255}$ &
  $\epsilon=\nicefrac{6}{255}$ &
  $\epsilon=\nicefrac{8}{255}$ &
  $\epsilon=\nicefrac{10}{255}$ &
  $\epsilon=\nicefrac{12}{255}$ &
  $\epsilon=\nicefrac{14}{255}$ &
  $\epsilon=\nicefrac{16}{255}$ \\ \hline \hline
GAT &
  \begin{tabular}[c]{@{}l@{}}88.79 $\pm$ 0.15 \\ 80.04 $\pm$ 0.06\end{tabular} &
  \begin{tabular}[c]{@{}l@{}}84.35 $\pm$ 0.11 \\ 68.51 $\pm$ 0.08\end{tabular} &
  \begin{tabular}[c]{@{}l@{}}80.16 $\pm$ 0.15 \\ 59.16 $\pm$ 0.24\end{tabular} &
  \begin{tabular}[c]{@{}l@{}}76.75 $\pm$ 0.38\\ 50.98 $\pm$ 0.12\end{tabular} &
  \multicolumn{1}{l|}{\begin{tabular}[c]{@{}l@{}}73.71 $\pm$ 0.12 \\ 43.34 $\pm$ 0.23\end{tabular}} &
  \begin{tabular}[c]{@{}l@{}}80.44 $\pm$ 5.08\\ 14.93 $\pm$ 9.26\end{tabular} &
  \begin{tabular}[c]{@{}l@{}}83.9 $\pm$ 1.0\\ 2.33 $\pm$ 0.58\end{tabular} &
  \begin{tabular}[c]{@{}l@{}}82.17 $\pm$ 2.47\\ 1.25 $\pm$ 0.51\end{tabular} \\ \hline
N-FGSM+GAT &
  \begin{tabular}[c]{@{}l@{}}89.1 $\pm$ 0.08\\ 79.96 $\pm$ 0.21\end{tabular} &
  \begin{tabular}[c]{@{}l@{}}84.84 $\pm$ 0.05\\ 69.5 $\pm$ 0.18\end{tabular} &
  \begin{tabular}[c]{@{}l@{}}81.38 $\pm$ 0.07\\ 60.06 $\pm$ 0.09\end{tabular} &
  \begin{tabular}[c]{@{}l@{}}78.28 $\pm$ 0.04\\ 51.8 $\pm$ 0.34\end{tabular} &
  \multicolumn{1}{l|}{\begin{tabular}[c]{@{}l@{}}75.66 $\pm$ 0.35\\ 44.97 $\pm$ 0.07\end{tabular}} &
  \begin{tabular}[c]{@{}l@{}}73.56 $\pm$ 0.23\\ 38.71 $\pm$ 0.16\end{tabular} &
  \begin{tabular}[c]{@{}l@{}}70.84 $\pm$ 0.51\\ 32.71 $\pm$ 0.11\end{tabular} &
  \begin{tabular}[c]{@{}l@{}}65.48 $\pm$ 0.96\\ 27.87 $\pm$ 0.35\end{tabular} \\ \hline \hline
NuAT &
  \begin{tabular}[c]{@{}l@{}}87.81 $\pm$ 0.24\\ 79.49 $\pm$ 0.03\end{tabular} &
  \begin{tabular}[c]{@{}l@{}}82.9 $\pm$ 0.18\\ 67.77 $\pm$ 0.13\end{tabular} &
  \begin{tabular}[c]{@{}l@{}}78.06 $\pm$ 0.2\\ 57.93 $\pm$ 0.17\end{tabular} &
  \begin{tabular}[c]{@{}l@{}}73.22 $\pm$ 0.34\\ 50.1 $\pm$ 0.33\end{tabular} &
  \multicolumn{1}{l|}{\begin{tabular}[c]{@{}l@{}}71.08 $\pm$ 4.87\\ 34.35 $\pm$ 9.0\end{tabular}} &
  \begin{tabular}[c]{@{}l@{}}74.38 $\pm$ 7.32\\ 17.54 $\pm$ 8.82\end{tabular} &
  \begin{tabular}[c]{@{}l@{}}78.5 $\pm$ 1.54\\ 6.6 $\pm$ 0.77\end{tabular} &
  \begin{tabular}[c]{@{}l@{}}80.1 $\pm$ 1.08\\ 3.29 $\pm$ 0.87\end{tabular} \\ \hline
NGFSM+NuAT &
  \begin{tabular}[c]{@{}l@{}}87.92 $\pm$ 0.0\\ 79.52 $\pm$ 0.0\end{tabular} &
  \begin{tabular}[c]{@{}l@{}}83.54 $\pm$ 0.0\\ 68.36 $\pm$ 0.0\end{tabular} &
  \begin{tabular}[c]{@{}l@{}}78.86 $\pm$ 0.25\\ 58.88 $\pm$ 0.16\end{tabular} &
  \begin{tabular}[c]{@{}l@{}}74.61 $\pm$ 0.34\\ 51.12 $\pm$ 0.2\end{tabular} &
  \multicolumn{1}{l|}{\begin{tabular}[c]{@{}l@{}}70.37 $\pm$ 0.12\\ 44.62 $\pm$ 0.38\end{tabular}} &
  \begin{tabular}[c]{@{}l@{}}65.56 $\pm$ 0.19\\ 38.24 $\pm$ 0.38\end{tabular} &
  \begin{tabular}[c]{@{}l@{}}60.76 $\pm$ 0.74\\ 32.85 $\pm$ 0.58\end{tabular} &
  \begin{tabular}[c]{@{}l@{}}52.79 $\pm$ 0.66\\ 29.19 $\pm$ 0.35\end{tabular} \\ \hline \hline
N-FGSM &
\begin{tabular}[c]{@{}c@{}}91.48 $\pm$ 0.17 \\ {79.43 $\pm$ 0.21}\end{tabular}  &
\begin{tabular}[c]{@{}c@{}}88.44 $\pm$ 0.09  \\ {67.09 $\pm$ 0.31}\end{tabular} &
\begin{tabular}[c]{@{}c@{}}84.72 $\pm$ 0.04  \\ {56.62 $\pm$ 0.26}\end{tabular} &
\begin{tabular}[c]{@{}c@{}}80.58 $\pm$ 0.22  \\ {48.12 $\pm$ 0.07}\end{tabular} &
\begin{tabular}[c]{@{}c@{}}75.98 $\pm$ 0.1  \\ {41.56 $\pm$ 0.16}\end{tabular}  &
\begin{tabular}[c]{@{}c@{}}71.46 $\pm$ 0.14 \\  {36.43 $\pm$ 0.16}\end{tabular}  &
\begin{tabular}[c]{@{}c@{}}67.11 $\pm$ 0.37  \\ {32.11 $\pm$ 0.2}\end{tabular}  &
\begin{tabular}[c]{@{}c@{}}63.18 $\pm$ 0.49 \\  {27.67 $\pm$ 0.93}\end{tabular}  \\ \hline

\end{tabular}
}
\end{table}

\section{Imagenet experimental details} \label{sec:imagenet_details}
For our experiments on Imagenet we mainly follow the settings from \cite{RS-FGSM}. However, for simplicity we did not do image resizing which requires storing two additional Imagenet datasets. More importantly, we found that the learning rate schedule suggested in \cite{RS-FGSM} was not optimal for N-FGSM. The schedule suggested in \cite{RS-FGSM} follows three different stages in which the learning increases or decreases linearly for some iterations. In particular in the first stage, the learning rate has an initial warm-up where it increases linearly from 0.0 to 0.4 during the first epoch and then decreases linearly to 0.04 during the next 5 epochs. As a lucky coincidence when debugging, we modified this initial stage such that we preserved the initial increase to 0.4 for the first epoch, but then we directly jumped to a learning rate of 0.04 which remained constant for the next 5 epochs. For phase 2 and 3 both schedules remained the same. First decreasing from 0.04 to 0.004 for epoch 6 to 12 and finally from 0.004 to 0.0004 for epoch 12 to 15. This small change made N-FGSM improve both in clean and robust accuracy for $\epsilon=\nicefrac{4}{255},\ \nicefrac{6}{255}$ and these are the numbers reported. This indicates that further tuning the learning rate schedule might be an effective way to improve performance and even help prevent CO, however, due to the computational demands of ImageNet adversarial training we leave it for future work. To be thorough we also trained RS-FGSM and FGSM with the modified schedule and found that neither of them benefit from it. Regarding N-FGSM hyperparameters, for $\epsilon=\nicefrac{2}{255}$ we used $\alpha=\nicefrac{2}{255}$ and $k = 1$; for $\epsilon=\nicefrac{4}{255}$ we used $\alpha=\nicefrac{4}{255}$ and $k = 1$; and for $\epsilon=\nicefrac{6}{255}$ we also used $\alpha=\nicefrac{4}{255}$ and $k = 1$.

\section{Visualization of the loss surface}
\label{sec:viz-loss}
In this section we present a visualization of the loss surface. We adapted the code from \cite{AAAI} to analyse the shape of the loss surface at the end of training for different methods. \cite{AAAI} reported that after adversarial training CO, the loss surface would become non-linear. In particular, they found that the FGSM perturbation seems to be misguided by local maxima very close to the clean image that result in ineffective attacks. We note this was already reported by \cite{tramer2018ensemble} which proposed to perform a random step to \textit{escape} those maxima. We argue that adding noise to the random step, when properly implemented, actually prevents those maxima to appear in the first place. 

\vspace{10pt} 

\begin{figure}[h]
\centering
\begin{subfigure}[b]{.44\linewidth}
\includegraphics[width=\linewidth]{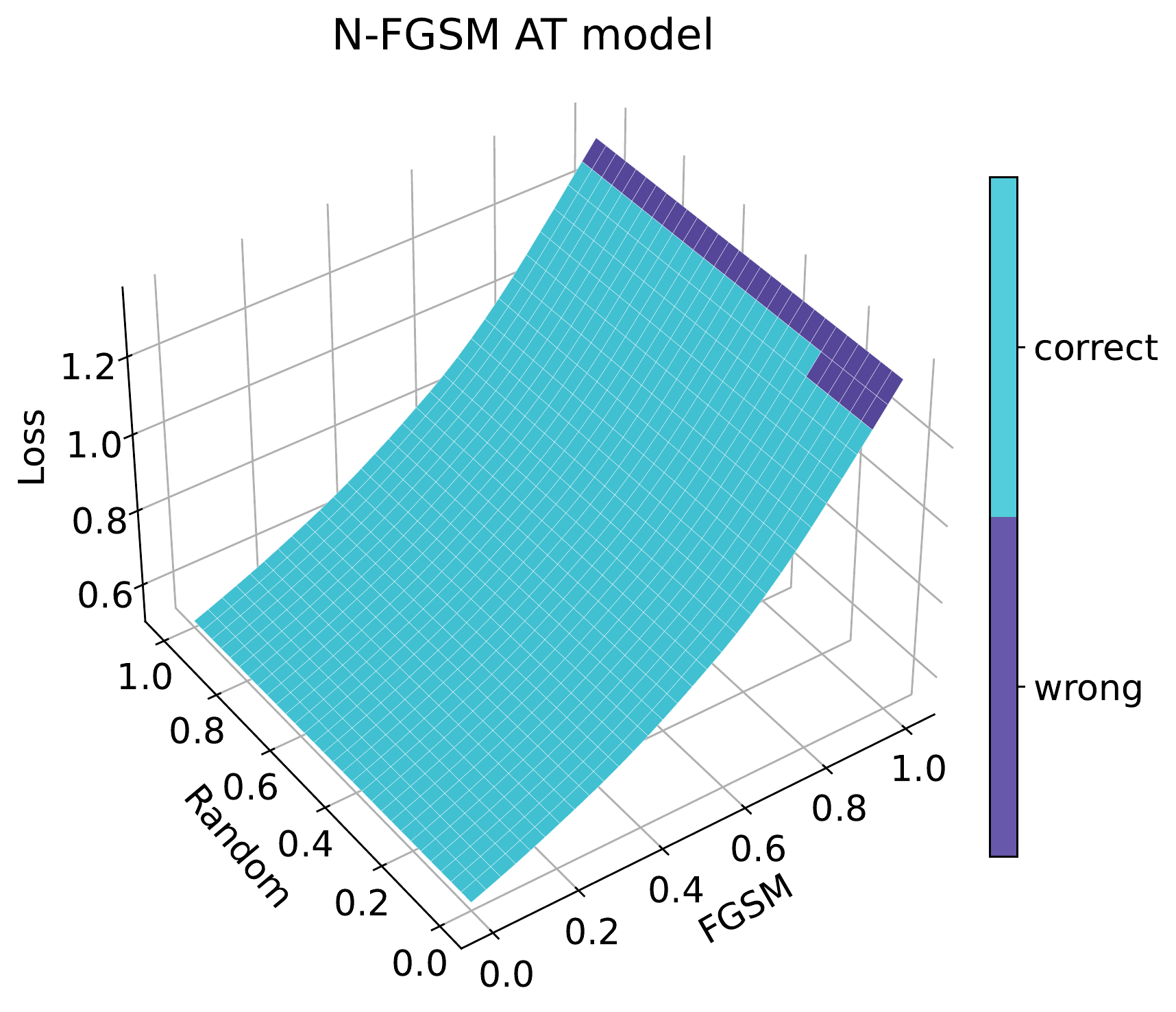}
\end{subfigure}
\begin{subfigure}[b]{.44\linewidth}
\includegraphics[width=\linewidth]{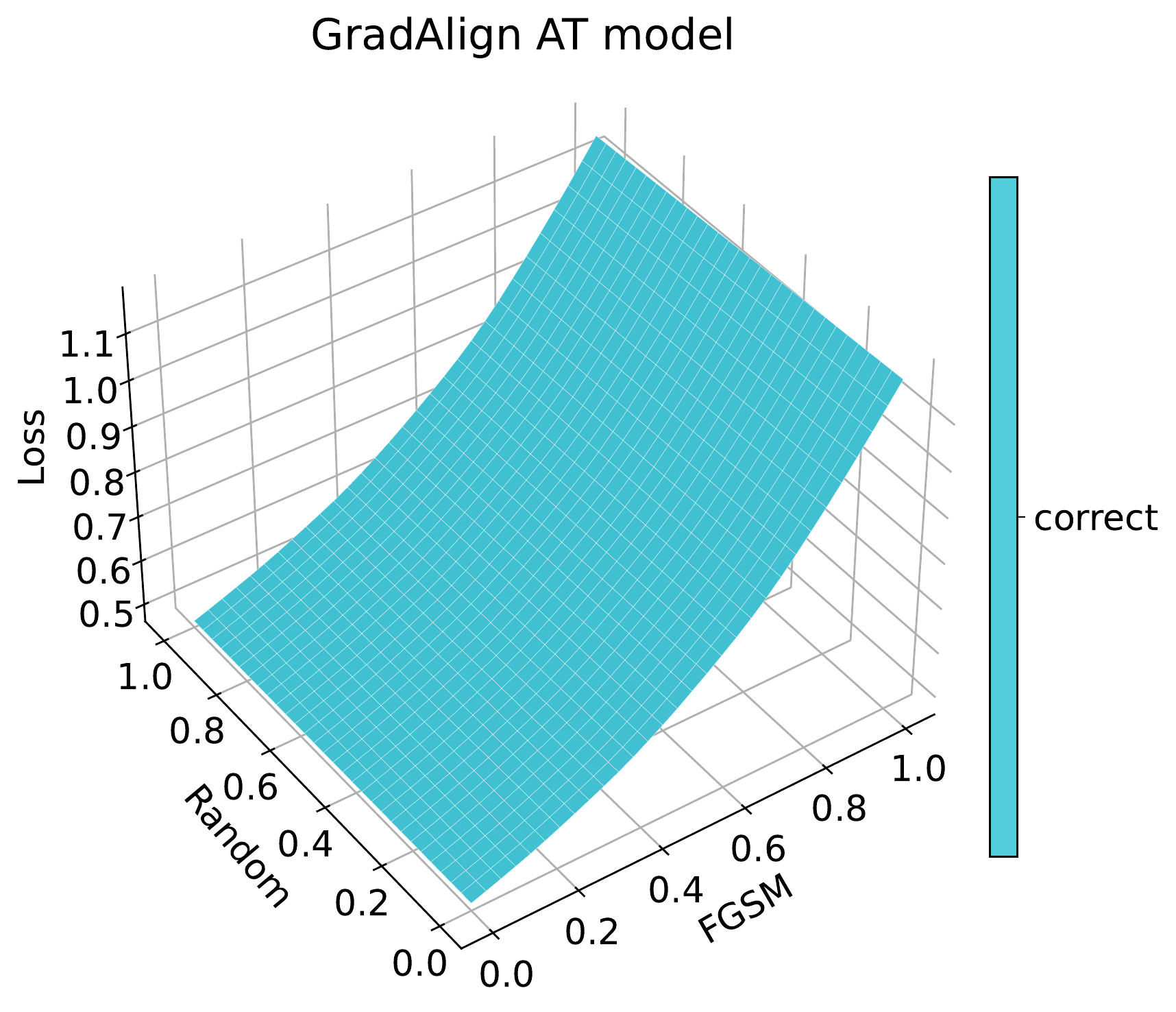}
\end{subfigure}
\vspace{20pt}

\centering
\begin{subfigure}[b]{.44\linewidth}
\includegraphics[width=\linewidth]{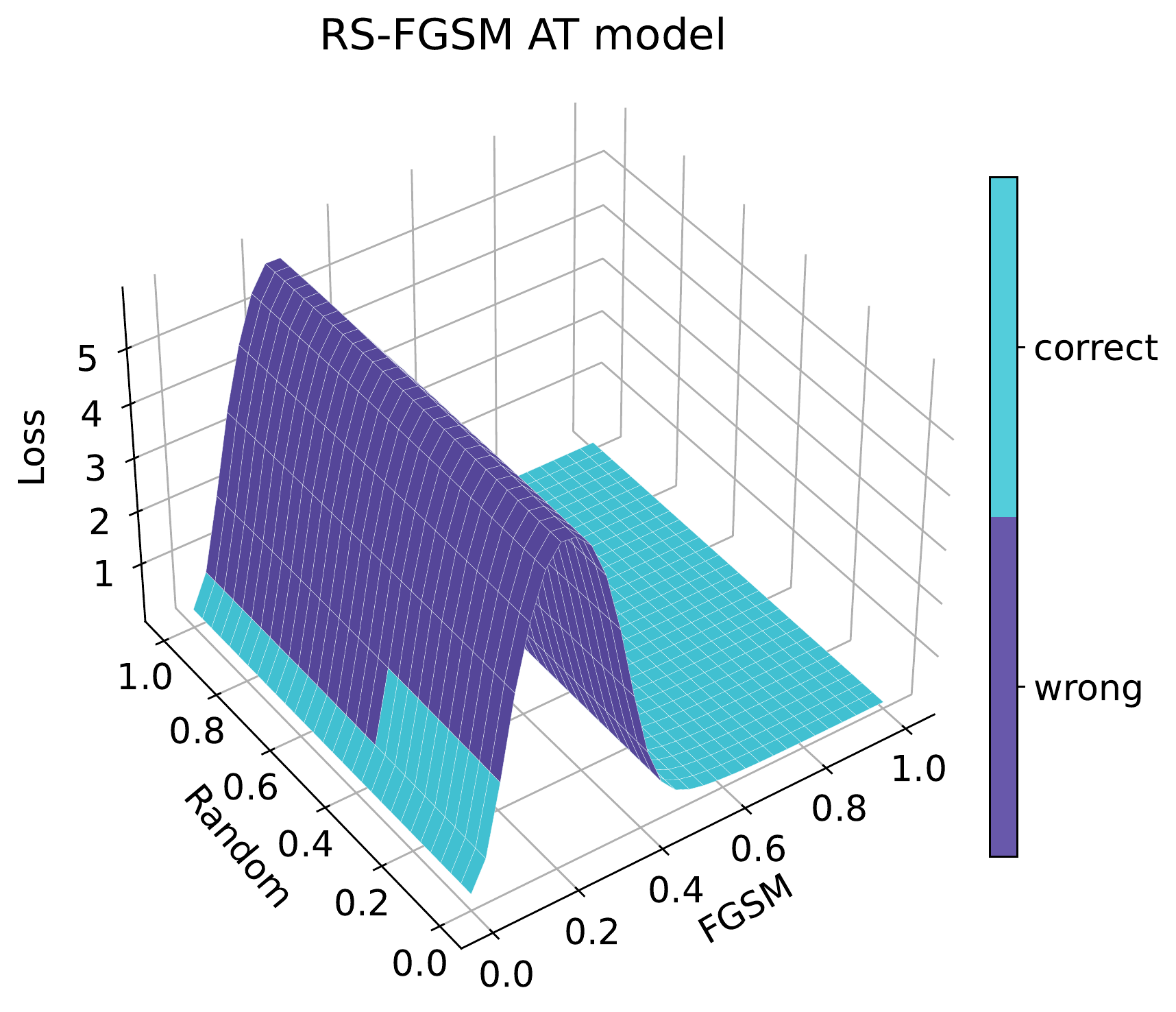}
\end{subfigure}
\begin{subfigure}[b]{.44\linewidth}
\includegraphics[width=\linewidth]{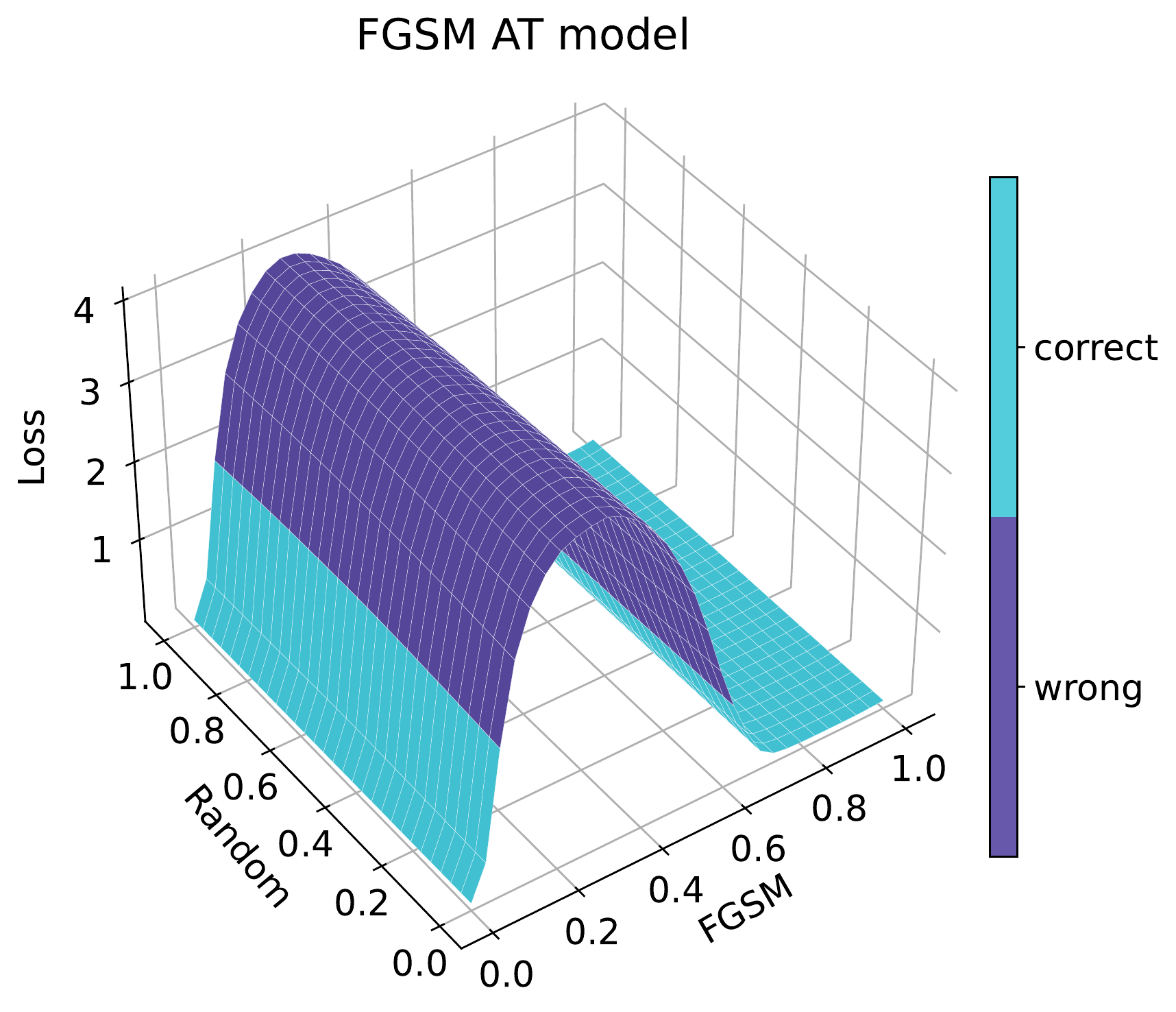}
\end{subfigure}

\caption{Visualization of the loss surface for models trained using different methods. Given a clean sample from the test set in coordinate $(0,0)$, we compute the FGSM perturbation and evaluate the loss on the subspace generated by the FGSM perturbation direction and a random direction. That is, we evaluate $x_{\textrm{clean}} + t_1\cdot \delta_{\textrm{FGSM}} + t_2 \cdot \delta_{\textrm{random}}$, where $t_1,\ t_2 \in [0, 1]$. Note that FGSM and RS-FGSM both have CO and the final models present a highly non-linear loss surface, on the other hand, both N-FGSM and GradAlign produce final models with a very linear loss surface which is key to obtain meaningful perturbations.}
\label{figure:loss_surface}
\end{figure}

\clearpage

\section{Magnitude of N-FGSM perturbations} \label{sec:l2_norm}

\begin{lemma}[Expected perturbation] \label{lemma_1}
Consider the N-FGSM perturbation as defined in~\Cref{equation:n-fgsm}
\[\delta_{\textrm{N-FGSM}} =  \eta + \alpha \cdot \textrm{sign} \left( \nabla_x \ell(f(x + \eta), y)  \right), \ \ \textrm{where} \ \ \eta \sim \Omega.\]

Let the distribution \(\Omega\) be the uniform distribution
\ \(\mathcal{U}\left(\left[-k\epsilon,k\epsilon\right]^d\right)\) and  \(\alpha > 0\). Then,

\[\mathbb{E}_\eta\left[\|\delta_{\textrm{N-FGSM}}\||^2_2\right] = d\left(\dfrac{k^2\epsilon^2}{3} + \alpha^2 \right) \quad \textrm{and} \quad \mathbb{E}_\eta\left[\|\delta_{\textrm{N-FGSM}}\||_2\right]\le \sqrt{d\left(\dfrac{k^2\epsilon^2}{3} + \alpha^2 \right)}\]
\end{lemma}
\begin{proof}
By Jensen's inequality, we have
    \[\mathbb{E}_\eta\left[\|\delta_{\textrm{N-FGSM}}\|_2\right]\le \sqrt{\mathbb{E}_\eta\left[\|\delta_{\textrm{N-FGSM}}\|_2^2\right]}\]
Then let us consider the term \(\mathbb{E}_\eta\left[\|\delta_{\textrm{N-FGSM}}\|_2^2\right]\) and use the shorthand \(\nabla(\eta)_i= \left( \nabla_x \ell(f(x + \eta), y)  \right)_i\).

\begin{align*}
    \mathbb{E}_\eta\left[\|\delta_{\textrm{N-FGSM}}\|_2^2\right] =& \mathbb{E}_\eta\| \eta + \alpha \cdot \textrm{sign} \left( \nabla_x \ell(f(x + \eta), y)  \right)\|_2^2\\
    =&\mathbb{E}_{\eta}\left[\sum_{i=1}^d \left( \eta_i + \alpha \cdot \textrm{sign}(\nabla(\eta)_i )\right)^2\right]\\
    =&\sum_{i=1}^d \mathbb{E}_{\eta}\left[\left( \eta_i + \alpha \cdot \textrm{sign} (\nabla(\eta)_i ) \right)^2\right]\\
    =&\sum_{i=1}^d\mathbb{E}_{\eta}\left[ \left( \eta_i + \alpha \cdot \textrm{sign} (\nabla(\eta)_i )\right)^2 {\vert\Large} \textrm{sign} (\nabla(\eta)_i )  =1\right]\mathbb{P}_\eta\left[ \textrm{sign}(\nabla(\eta)_i ) =1\right]\\
    &+\sum_{i=1}^d \mathbb{E}_{\eta}\left[\left( \eta_i + \alpha \cdot \textrm{sign} (\nabla(\eta)_i )\right)^2 {\vert\Large} \textrm{sign} (\nabla(\eta)_i )  =-1\right]\mathbb{P}_\eta\left[ \textrm{sign}(\nabla(\eta)_i ) =-1\right]\\
    =&\sum_{i=1}^d\dfrac{1}{2k\epsilon}\int_{-k\epsilon}^{k\epsilon} \left( \eta_i + \alpha\right)^2 d\eta_i\cdot \mathbb{P}_\eta\left[ \textrm{sign}(\nabla(\eta)_i ) =1\right]\\
    &+ \dfrac{1}{2k\epsilon}\sum_{i=1}^d\int_{-k\epsilon}^{k\epsilon}\left( \eta_i - \alpha\right)^2 d\eta_i\cdot\mathbb{P}_\eta\left[ \textrm{sign}(\nabla(\eta)_i ) =-1\right] \\
    =&\sum_{i=1}^d\dfrac{1}{2k\epsilon}\int_{\alpha-k\epsilon}^{\alpha+k\epsilon}z^2 dz\cdot \mathbb{P}_\eta\left[ \textrm{sign}(\nabla(\eta)_i ) =1\right] \\ &+\dfrac{1}{2k\epsilon}\sum_{i=1}^d\int_{-\alpha-k\epsilon}^{-\alpha+k\epsilon}z^2 dz\cdot\mathbb{P}_\eta\left[ \textrm{sign}(\nabla(\eta)_i ) =-1\right] \\ 
    =&\sum_{i=1}^d\dfrac{1}{2k\epsilon}\int_{\alpha-k\epsilon}^{\alpha+k\epsilon}z^2 dz\cdot \mathbb{P}_\eta\left[ \textrm{sign}(\nabla(\eta)_i ) =1\right] \\ &+\dfrac{1}{2k\epsilon}\sum_{i=1}^d\int_{\alpha-k\epsilon}^{\alpha+k\epsilon}z^2 dz\cdot\mathbb{P}_\eta\left[ \textrm{sign}(\nabla(\eta)_i ) =-1\right] \\
    =& \dfrac{1}{2k\epsilon}\int_{\alpha-k\epsilon}^{\alpha+k\epsilon}z^2 dz\sum_{i=1}^d\left(\mathbb{P}_\eta\left[ \textrm{sign}(\nabla(\eta)_i ) =1\right]+\mathbb{P}_\eta\left[ \textrm{sign}(\nabla(\eta)_i ) =-1\right]\right)\\
    =&\dfrac{d}{6k\epsilon}\left[(\alpha+k\epsilon)^3 - (\alpha-k\epsilon)^3\right] = \dfrac{dk^2\epsilon^2}{3}+d\alpha^2
\end{align*}

\clearpage

Therefore,

\[\mathbb{E}_\eta\left[\|\delta_{\textrm{N-FGSM}}\||_2\right]\le \sqrt{d\left(\dfrac{k^2\epsilon^2}{3} + \alpha^2 \right).}\]
\end{proof}

\begin{theorem} \label{theorem:l2_norm}
Let $\delta_{\text{N-FGSM}}$ be our proposed single-step method defined by~\Cref{equation:n-fgsm}, $\delta_{\text{FGSM}}$ be the FGSM method~\citep{fgsm} and $\delta_{\text{RS-FGSM}}$ be the RS-FGSM method~\citep{RS-FGSM}. Then, with default hyperparameter values and for any $\epsilon > 0$, we have that
$$\mathbb{E}_\eta\left[\|\delta_{\text{N-FGSM}}\|_2^2\right] > \mathbb{E}_\eta\left[\|\delta_{\text{FGSM}}\|_2^2\right] > \mathbb{E}_\eta\left[\|\delta_{\text{RS-FGSM}}\|_2^2\right].$$
\end{theorem}

\begin{proof}

From Lemma~{\color{red}\ref{lemma_1}} we have that 
\[\mathbb{E}_\eta\left[\|\delta_{\textrm{N-FGSM}}\||^2_2\right] = d\left(\dfrac{k^2\epsilon^2}{3} + \alpha^2 \right).\]

On the other hand, \cite{grad_align} showed that 
\[ \mathbb{E}_\eta\left[\|\delta_{\text{RS-FGSM}}\|_2^2\right] = d \left(-\frac{1}{6\epsilon}\alpha^3 + \frac{1}{2}\alpha^2 + \frac{1}{3}\epsilon^2\right).\] 

Finally, we note that
\[ \mathbb{E}_\eta\left[\|\delta_{\text{FGSM}}\|_2^2\right] = \|\delta_{\text{FGSM}}\|_2^2 = d\epsilon^2. \]

The default hyperparameters for N-FGSM are $k=2,\ \ \alpha=\epsilon$ and RS-FGSM uses $\alpha=5\epsilon/4$. With these hyperparameters and any $\epsilon > 0$ we have
\[\mathbb{E}_\eta\left[\|\delta_{\textrm{N-FGSM}}\||^2_2\right] = \dfrac{7}{3} d\epsilon^2 > \mathbb{E}_\eta\left[\|\delta_{\textrm{FGSM}}\||^2_2\right] = d \epsilon^2 > \mathbb{E}_\eta\left[\|\delta_{\textrm{RS-FGSM}}\||^2_2\right] = \dfrac{101}{128} d\epsilon^2 \]

\end{proof}

In Lemma~{\color{red}\ref{lemma_1}} we compute the expected value of the squared $\ell_2$ norm of N-FGSM perturbations and by Jensen's inequality we obtain an upper bound for the expected $\ell_2$ norm of N-FGSM perturbations. However, obtaining the exact expected magnitude is more complex. To compliment our analytic results, we approximate the $\ell_2$ norm of FGSM, RS-FGSM and N-FGSM via Monte Carlo sampling. Results are presented in~\cref{figure:magnitude_perturbations}. We observe that the empirical estimations are very close to the analytical upper bounds and that indeed, N-FGSM has a magnitude significantly above that of FGSM or RS-FGSM.

\begin{figure}[h]
\centering
\begin{subfigure}[b]{.45\linewidth}
\includegraphics[width=\linewidth]{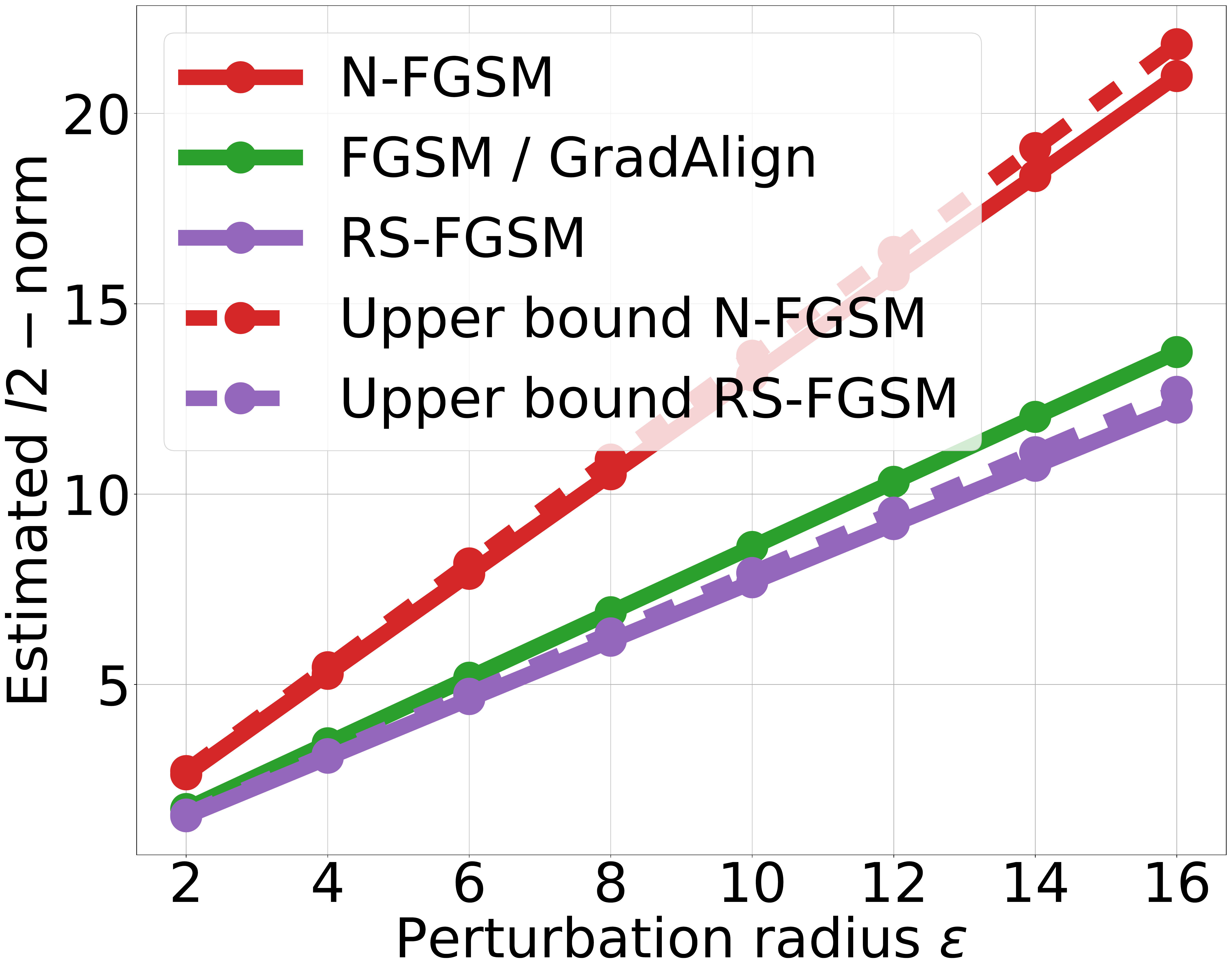}
\end{subfigure}
\caption{Monte Carlo estimations of the expected $l_2-$norm of perturbations from different methods and corresponding analytical upper bounds. As mentioned in \cite{grad_align}, we observe that RS-FGSM perturbations have lower $l_2$ norm than FGSM. However, N-FGSM perturbations have a significantly higher $l_2-$norm than both RS-FGSM and FGSM. This seems to indicate that the role of random step is not simply to lower the $l_2$ norm as previously suggested \citep{grad_align}. }
\label{figure:magnitude_perturbations}
\end{figure} 

\section{N-FGSM with Gaussian noise} \label{sec:gaussian_noise}
In the main paper we have only explored noise sources coming from a Uniform distribution. Since we are measuring robustness against $l_\infty-$ attacks, the Uniform distribution is a natural choice because the random perturbations will be bounded to the $l_\infty$ ball defined by the span of the distribution. However, for the sake of completeness, we also explore the performance of augmenting the samples from a Gaussian distribution where we choose its standard deviation to match that of the uniform distribution. In~\cref{table:gaussian_noise} we present a comparison of the clean (top) and PGD-50-10 (bottom) accuracy for different values of $\alpha$ and noise magnitude with $\epsilon = \nicefrac{8}{255}$. Recall that by default we use Uniform distribution $\mathcal{U}[-k, k]$, therefore hyperparameter $k$ sets the noise magnitude. 

Increasing the FGSM step size without increasing the amount of noise leads to CO. Note results for $k=0.5\epsilon$. More importantly, results are very similar when the two noise distributions share the same standard deviation. Thus, using Gaussian instead of Uniform noise does not seem to alter the results.  Although this might be expected, we remark that the Gaussian is an unbounded noise distribution and the common practice in adversarial training is to always restrict the norm of the perturbations.

\begin{table}
\caption{Comparison of the clean (top) and PGD-50-10 (bottom) accuracy across different values of step-size $\alpha$ and noise magnitude for the Uniform and Gaussian distributions with $\epsilon = 8/255$. For every value of $k$, we use a Gaussian with matching standard deviation. We observe that when we match the standard deviation, both distribution perform similarly.}\label{table:gaussian_noise}
\vspace{4pt}
\renewcommand{\arraystretch}{1.5}
\scriptsize
\centering
\begin{tabular}{cc|c|c||c|c|c}
\toprule
& \multicolumn{3}{c}{\footnotesize{\textbf{Uniform Noise}}} & \multicolumn{3}{c}{\footnotesize{\textbf{Gaussian Noise}}}\\ 
\midrule
\multicolumn{1}{c|}{}  & \scriptsize{$\alpha = \nicefrac{6}{255}$ ($0.75 \epsilon$)}                                                               & \scriptsize{$\alpha = \nicefrac{8}{255}$ ($1 \epsilon$)}                                                       & \scriptsize{$\alpha = \nicefrac{10}{255}$ ($1.25 \epsilon$)}  & \scriptsize{$\alpha = \nicefrac{6}{255}$ ($0.75 \epsilon$)}  & \scriptsize{$\alpha = \nicefrac{8}{255}$ ($1 \epsilon$)} & \multicolumn{1}{c}{\scriptsize{$\alpha = \nicefrac{10}{255}$ ($1.25 \epsilon$)}}                                                              \\ \hline
\multicolumn{1}{r|}{\scriptsize{$k = 0.5\epsilon$}} & \begin{tabular}[c]{@{}c@{}}85.52 $\pm$ 0.23 \\  44.14 $\pm$ 0.24\end{tabular} & \begin{tabular}[c]{@{}c@{}}81.54 $\pm$ 0.19 \\ 47.93 $\pm$ 0.11\end{tabular}  & \begin{tabular}[c]{@{}c@{}}82.81 $\pm$ 1.11\\  0.0 $\pm$ 0.0\end{tabular}     & \begin{tabular}[c]{@{}c@{}}85.27 $\pm$ 0.11 \\  44.23 $\pm$ 0.17\end{tabular} & \begin{tabular}[c]{@{}c@{}}81.71 $\pm$ 0.27 \\  47.98 $\pm$ 0.14\end{tabular} & \multicolumn{1}{c}{\begin{tabular}[c]{@{}c@{}}83.34 $\pm$ 1.48 \\ 0.0 $\pm$ 0.0\end{tabular}}     \\ \hline
\multicolumn{1}{c|}{\scriptsize{$k = 1\epsilon$}}   & \begin{tabular}[c]{@{}c@{}}85.03 $\pm$ 0.09 \\  44.44 $\pm$ 0.13\end{tabular} & \begin{tabular}[c]{@{}c@{}}81.57 $\pm$ 0.07 \\  48.16 $\pm$ 0.21\end{tabular} & \begin{tabular}[c]{@{}c@{}}77.32 $\pm$ 0.14 \\  49.68 $\pm$ 0.25\end{tabular} & \begin{tabular}[c]{@{}c@{}}85.01 $\pm$ 0.17 \\  44.41 $\pm$ 0.04\end{tabular} & \begin{tabular}[c]{@{}c@{}}81.35 $\pm$ 0.14 \\  48.21 $\pm$ 0.11\end{tabular} & \multicolumn{1}{c}{\begin{tabular}[c]{@{}c@{}}77.22 $\pm$ 0.32 \\  49.83 $\pm$ 0.1\end{tabular}}  \\ \hline
\multicolumn{1}{c|}{\scriptsize{$k = 2\epsilon$}}   & \begin{tabular}[c]{@{}c@{}}84.49 $\pm$ 0.1 \\ 44.44 $\pm$ 0.15\end{tabular}   & \begin{tabular}[c]{@{}c@{}}80.58 $\pm$ 0.22 \\ 48.12 $\pm$ 0.07\end{tabular}  & \begin{tabular}[c]{@{}c@{}}76.49 $\pm$ 0.14 \\ 49.77 $\pm$ 0.37\end{tabular}  & \begin{tabular}[c]{@{}c@{}}84.35 $\pm$ 0.24 \\  44.59 $\pm$ 0.22\end{tabular} & \begin{tabular}[c]{@{}c@{}}80.44 $\pm$ 0.31 \\  48.34 $\pm$ 0.1\end{tabular}  & \multicolumn{1}{c}{\begin{tabular}[c]{@{}c@{}}76.33 $\pm$ 0.37 \\  49.77 $\pm$ 0.23\end{tabular}} \\ \hline
\end{tabular}

\end{table}

\section{Training with noise augmented samples} \label{sec:noise_augmentation}
\citet{gilmer2019adversarial} and \citet{fawzi2018empirical} report a close link between robustness to adversarial attacks and robustness to random noise. Actually, \cite{gilmer2019adversarial} report that training with noise-augmented samples can improve adversarial accuracy and vice-versa. We note that N-FGSM can actually be seen as a combination of noise-augmentation and adversarial attacks. Here we perform an ablation where we train models with samples augmented with uniform noise $\mathcal{U}[-k, k]$ and then test the PGD-50-10 accuracy. We observe, that indeed random noise can increase the robustness to worst-case perturbations for small $\epsilon - l_\infty$ balls. However, as we increase $\epsilon$, noise augmentation is no longer very effective. With N-FGSM, we apply a weak attack to these noise-augmented samples and this seems to be enough to make them effective for adversarial training.

\begin{figure}
\centering
\begin{subfigure}[b]{.47\linewidth}
\includegraphics[width=\linewidth]{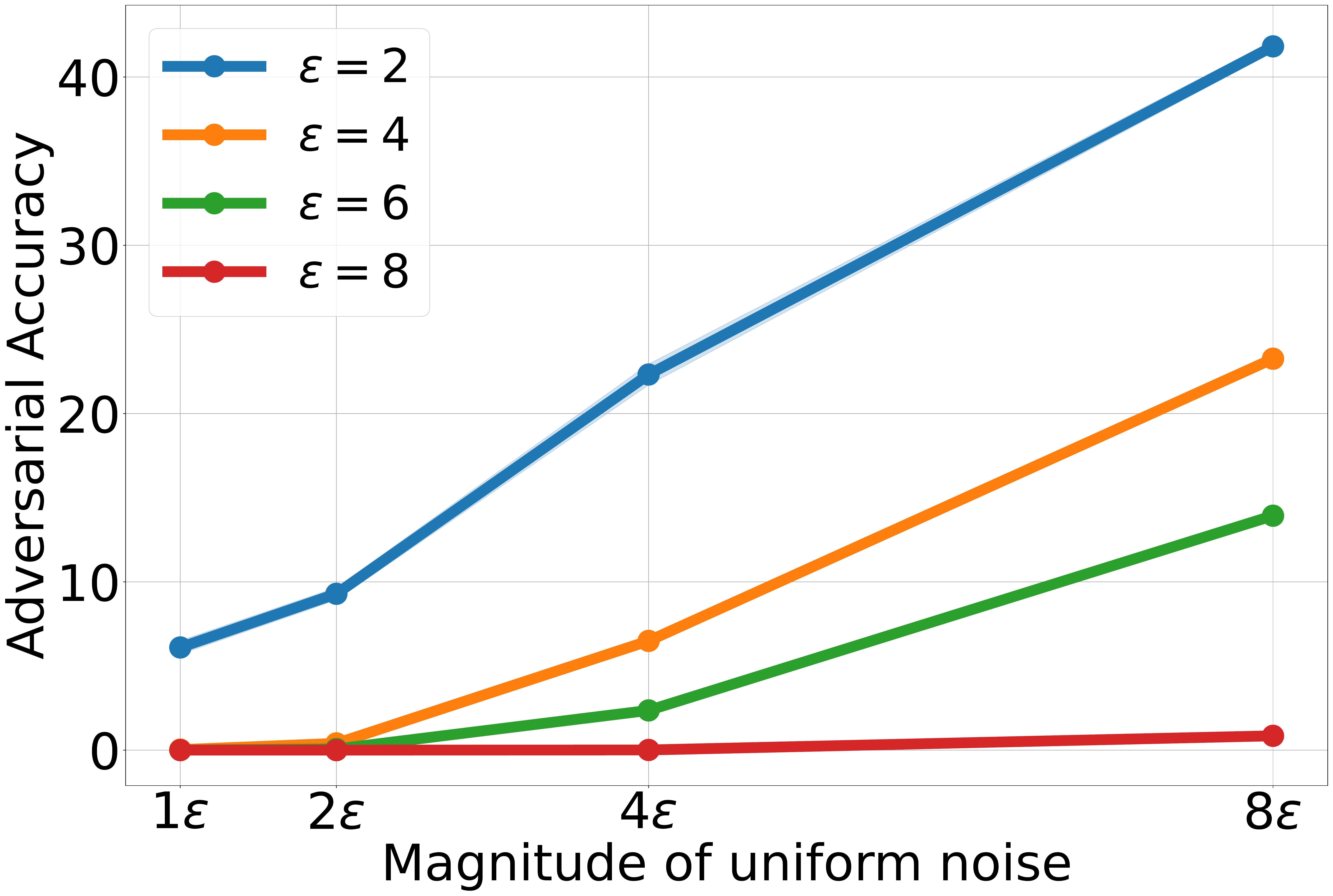}
\end{subfigure}
\caption{Training with uniform noise augmented samples improves adversarial accuracy for small perturbations but is not effective to protect against larger $l_\infty$ radius $\epsilon$. This motivates us to further augment the noisy samples with FGSM. All experiments are averaged over 3 runs.}
\label{figure:noise_augmentation}
\vspace{-0.3cm}
\end{figure}

\section{Comparison of adversarial training cost}\label{sec:train_cost}
In this section we describe how we compute the relative training cost for single-step methods shown in~\cref{figure:splash} (right). We approximate the cost based on the number of forward/backward passes each method uses, disregarding the cost of other additional operations such as adding a random step for RS-FGSM or N-FGSM. We understand these operations have a negligible cost compared to a full forward or backward pass.

\textbf{FGSM:} FGSM is the cheapest of all methods since it only uses one forward/backward to compute the attack and an additional forward/backward to compute the weight update. Hence, Cost FGSM = 2 F/B.

\textbf{RS-FGSM:} As previously mentioned, we do not take into account the cost of random steps or clipping, hence we consider RS-FGSM to have the same cost as standard FGSM. Cost RS-FGSM = 2 F/B.

\textbf{N-FGSM:} Idem as before, cost of N-FGSM = 2 F/B.

\textbf{ZeroGrad:} For ZeroGrad they need to do an additional sorting operation to find the smallest gradient components. This could be potentially expensive, however, since the size of the input image is several orders of magnitude smaller than that of the network, we also ignore this cost. Cost ZeroGrad = 2 F/B.

\textbf{MultiGrad:} MultiGrad computes 3 random steps and evaluates the gradient in all of them. Therefore, it needs to do 3 F/B to compute the attack and an additional one to update the weights. Cost MultiGrad = 4 F/B. 

\textbf{\cite{AAAI}:} \cite{AAAI}  compute the RS-FGSM perturbation and evaluate the model on $c$ points along this direction. Therefore, they will spend 1F/B on the RS-FGSM attack, $c - 1$ F on the evaluations since the clean image has already been evaluated; and 1 F/B for the weight update. In our plot, we used $c=3$ since it was the most chosen setting. \cite{AAAI} assume the cost of a forward is similar to that of a backward pass, following this assumption, cost of \cite{AAAI} is 1 F/B + 2 F + 1 F/B = 3 F/B

\textbf{Free-AT:} \cite{free} re-use the gradient from the previous backward pass to compute the FGSM perturbation of the current iteration. Hence, the cost of their training is only 1 F/B per iteration. However, \cite{RS-FGSM} observed they needed a longer training schedule to produce comparable results. Therefore, the total training cost per iteration (1 F/B) is scaled by 96 in the case of Free-AT, while it is only scaled by 30 for other methods.
Relative cost Free = (96 $\cdot$ 1 F/B) / (30 $\cdot$ 2 F/B).

\textbf{GradAlign:} Finally, GradAlign uses FGSM with a regularizer. However, this regularizer needs to compute second-order derivatives via double backpropagation, which does not have the same cost as regular backpropagation. \cite{grad_align} report that the cost of using GradAlign regularizer increased the cost of FGSM by 3.

\section{Infrastructure details and GPU hours} \label{sec:gpu-hours}
All our training runs have been conducted on either NVIDIA GPU V-100 or P-100 from an internal cluster. The total compute for the results presented in this work is roughly 2500 hours.

\section{Detailed results for~\cref{sec:single-step} and~\cref{sec:wideresnet}} \label{sec:detailed_results}
In this section we present the tables with the exact numbers used in plots comparing adversarial training methods. For each method and $\epsilon-l_\infty$ radius, the top number is the clean accuracy while the bottom number is the PGD-50-10 accuracy. We separate single-step from multi-step methods with a double line.

\begin{table}[h]
\renewcommand{\arraystretch}{1.5}
\centering
\scriptsize{
\begin{tabular}{c|c|c|c|c|c|c|c|c}
\multicolumn{4}{c}{\small{\textbf{PreActResNet18 -- CIFAR-10 Dataset}}} \\
\multicolumn{2}{c}{} \\
\toprule
           & $\epsilon$ = $\nicefrac{2}{255}$                                                                 & $\epsilon$ = $\nicefrac{4}{255}$                                                                 & $\epsilon$ = $\nicefrac{6}{255}$                                                                 & $\epsilon$ = $\nicefrac{8}{255}$                                                                 & $\epsilon$ = $\nicefrac{10}{255}$                                                                & $\epsilon$ = $\nicefrac{12}{255}$                                                                 & $\epsilon$ = $\nicefrac{14}{255}$                                                                & $\epsilon$ = $\nicefrac{16}{255}$    \\ \hline
           \midrule
\textbf{N-FGSM}     & \begin{tabular}[c]{@{}c@{}}91.48 $\pm$ 0.17 \\ \textbf{79.43 $\pm$ 0.21}\end{tabular}  & \begin{tabular}[c]{@{}c@{}}88.44 $\pm$ 0.09  \\ \textbf{67.09 $\pm$ 0.31}\end{tabular} & \begin{tabular}[c]{@{}c@{}}84.72 $\pm$ 0.04  \\ \textbf{56.62 $\pm$ 0.26}\end{tabular} & \begin{tabular}[c]{@{}c@{}}80.58 $\pm$ 0.22  \\ \textbf{48.12 $\pm$ 0.07}\end{tabular} & \begin{tabular}[c]{@{}c@{}}75.98 $\pm$ 0.1  \\ \textbf{41.56 $\pm$ 0.16}\end{tabular}  & \begin{tabular}[c]{@{}c@{}}71.46 $\pm$ 0.14 \\  \textbf{36.43 $\pm$ 0.16}\end{tabular}  & \begin{tabular}[c]{@{}c@{}}67.11 $\pm$ 0.37  \\ \textbf{32.11 $\pm$ 0.2}\end{tabular}  & \begin{tabular}[c]{@{}c@{}}63.18 $\pm$ 0.49 \\  \textbf{27.67 $\pm$ 0.93}\end{tabular}  \\ \hline
Grad Align & \begin{tabular}[c]{@{}c@{}}91.73 $\pm$ 0.04 \\  79.16 $\pm$ 0.03\end{tabular} & \begin{tabular}[c]{@{}c@{}}88.76 $\pm$ 0.0\\   \textbf{67.13 $\pm$ 0.26}\end{tabular}  & \begin{tabular}[c]{@{}c@{}}85.67 $\pm$ 0.02\\   \textbf{56.27 $\pm$ 0.31}\end{tabular} & \begin{tabular}[c]{@{}c@{}}81.9 $\pm$ 0.22\\   \textbf{48.14 $\pm$ 0.15}\end{tabular}  & \begin{tabular}[c]{@{}c@{}}77.54 $\pm$ 0.06\\   40.75 $\pm$ 0.28\end{tabular} & \begin{tabular}[c]{@{}c@{}}73.29 $\pm$ 0.23  \\ 34.51 $\pm$ 0.63\end{tabular}  & \begin{tabular}[c]{@{}c@{}}68.01 $\pm$ 0.32 \\  30.36 $\pm$ 0.27\end{tabular} & \begin{tabular}[c]{@{}c@{}}61.3 $\pm$ 0.15  \\ \textbf{26.64 $\pm$ 0.27}\end{tabular}   \\ \hline
FGSM       & \begin{tabular}[c]{@{}c@{}}91.6 $\pm$ 0.1 \\ 79.35 $\pm$ 0.06\end{tabular}    & \begin{tabular}[c]{@{}c@{}}88.77 $\pm$ 0.04  \\ 67.11 $\pm$ 0.09\end{tabular} & \begin{tabular}[c]{@{}c@{}}85.58 $\pm$ 0.11 \\  56.33 $\pm$ 0.41\end{tabular} & \begin{tabular}[c]{@{}c@{}}86.41 $\pm$ 0.7  \\ 0.0 $\pm$ 0.0\end{tabular}     & \begin{tabular}[c]{@{}c@{}}82.08 $\pm$ 1.62 \\  0.0 $\pm$ 0.0\end{tabular}    & \begin{tabular}[c]{@{}c@{}}80.6 $\pm$ 2.59 \\  0.0 $\pm$ 0.0\end{tabular}      & \begin{tabular}[c]{@{}c@{}}76.04 $\pm$ 2.37  \\ 0.0 $\pm$ 0.0\end{tabular}    & \begin{tabular}[c]{@{}c@{}}77.14 $\pm$ 2.46 \\  0.0 $\pm$ 0.0\end{tabular}     \\ \hline
RS-FGSM    & \begin{tabular}[c]{@{}c@{}}92.09 $\pm$ 0.05 \\ 78.64 $\pm$ 0.08\end{tabular}  & \begin{tabular}[c]{@{}c@{}}89.69 $\pm$ 0.01 \\  66.12 $\pm$ 0.22\end{tabular} & \begin{tabular}[c]{@{}c@{}}87.0 $\pm$ 0.12 \\  54.87 $\pm$ 0.22\end{tabular}  & \begin{tabular}[c]{@{}c@{}}84.05 $\pm$ 0.13  \\ 46.08 $\pm$ 0.18\end{tabular} & \begin{tabular}[c]{@{}c@{}}85.21 $\pm$ 0.51 \\  0.0 $\pm$ 0.0\end{tabular}    & \begin{tabular}[c]{@{}c@{}}65.22 $\pm$ 23.23  \\ 0.0 $\pm$ 0.0\end{tabular}    & \begin{tabular}[c]{@{}c@{}}43.59 $\pm$ 25.01  \\ 0.0 $\pm$ 0.0\end{tabular}   & \begin{tabular}[c]{@{}c@{}}76.66 $\pm$ 0.38 \\  0.0 $\pm$ 0.0\end{tabular}     \\ \hline
Kim et. al.       & \begin{tabular}[c]{@{}c@{}}92.85 $\pm$ 0.11 \\ 74.74 $\pm$ 0.35\end{tabular}  & \begin{tabular}[c]{@{}c@{}}91.1 $\pm$ 0.04  \\ 60.51 $\pm$ 0.4\end{tabular}   & \begin{tabular}[c]{@{}c@{}}89.34 $\pm$ 0.05\\  48.95 $\pm$ 0.45\end{tabular}  & \begin{tabular}[c]{@{}c@{}}89.02 $\pm$ 0.1  \\ 33.01 $\pm$ 0.09\end{tabular}  & \begin{tabular}[c]{@{}c@{}}88.27 $\pm$ 0.14  \\ 24.43 $\pm$ 0.84\end{tabular} & \begin{tabular}[c]{@{}c@{}}88.35 $\pm$ 0.31\\   13.11 $\pm$ 0.63\end{tabular}  & \begin{tabular}[c]{@{}c@{}}90.01 $\pm$ 0.25  \\ 5.86 $\pm$ 0.57\end{tabular}  & \begin{tabular}[c]{@{}c@{}}90.45 $\pm$ 0.08  \\ 1.88 $\pm$ 0.05\end{tabular}   \\ \hline
AT Free    & \begin{tabular}[c]{@{}c@{}}87.99 $\pm$ 0.16 \\ 74.27 $\pm$ 0.33\end{tabular}  & \begin{tabular}[c]{@{}c@{}}84.98 $\pm$ 0.13  \\ 62.47 $\pm$ 0.25\end{tabular} & \begin{tabular}[c]{@{}c@{}}81.77 $\pm$ 0.11 \\  53.18 $\pm$ 0.15\end{tabular} & \begin{tabular}[c]{@{}c@{}}78.41 $\pm$ 0.18  \\ 46.03 $\pm$ 0.36\end{tabular} & \begin{tabular}[c]{@{}c@{}}74.79 $\pm$ 0.22 \\  39.87 $\pm$ 0.07\end{tabular} & \begin{tabular}[c]{@{}c@{}}73.91 $\pm$ 4.19  \\ 22.99 $\pm$ 16.26\end{tabular} & \begin{tabular}[c]{@{}c@{}}61.92 $\pm$ 14.94 \\  0.0 $\pm$ 0.0\end{tabular}   & \begin{tabular}[c]{@{}c@{}}71.64 $\pm$ 3.89  \\ 0.0 $\pm$ 0.0\end{tabular}     \\ \hline
ZeroGrad   & \begin{tabular}[c]{@{}c@{}}91.71 $\pm$ 0.08 \\ 79.36 $\pm$ 0.05\end{tabular}  & \begin{tabular}[c]{@{}c@{}}88.8 $\pm$ 0.11  \\ \textbf{67.32 $\pm$ 0.02}\end{tabular}  & \begin{tabular}[c]{@{}c@{}}85.71 $\pm$ 0.1  \\ 56.14 $\pm$ 0.21\end{tabular}  & \begin{tabular}[c]{@{}c@{}}82.62 $\pm$ 0.05  \\ 47.08 $\pm$ 0.1\end{tabular}  & \begin{tabular}[c]{@{}c@{}}79.91 $\pm$ 0.12  \\ 37.58 $\pm$ 0.2\end{tabular}  & \begin{tabular}[c]{@{}c@{}}78.11 $\pm$ 0.2  \\ 27.41 $\pm$ 0.27\end{tabular}   & \begin{tabular}[c]{@{}c@{}}75.66 $\pm$ 0.46  \\ 21.29 $\pm$ 0.97\end{tabular} & \begin{tabular}[c]{@{}c@{}}75.42 $\pm$ 0.13  \\ 13.06 $\pm$ 0.22\end{tabular}  \\ \hline
MultiGrad  & \begin{tabular}[c]{@{}c@{}}91.57 $\pm$ 0.16 \\ 79.34 $\pm$ 0.02\end{tabular}  & \begin{tabular}[c]{@{}c@{}}88.74 $\pm$ 0.12  \\ 66.81 $\pm$ 0.02\end{tabular} & \begin{tabular}[c]{@{}c@{}}85.75 $\pm$ 0.05 \\  56.02 $\pm$ 0.3\end{tabular}  & \begin{tabular}[c]{@{}c@{}}82.33 $\pm$ 0.14  \\ 47.29 $\pm$ 0.07\end{tabular} & \begin{tabular}[c]{@{}c@{}}78.73 $\pm$ 0.16 \\  40.11 $\pm$ 0.24\end{tabular} & \begin{tabular}[c]{@{}c@{}}75.28 $\pm$ 0.2  \\ 33.87 $\pm$ 0.17\end{tabular}   & \begin{tabular}[c]{@{}c@{}}80.94 $\pm$ 5.94 \\  9.55 $\pm$ 13.5\end{tabular}  & \begin{tabular}[c]{@{}c@{}}71.42 $\pm$ 5.63  \\ 16.35 $\pm$ 11.57\end{tabular} \\ \hline \hline
PGD-2      & \begin{tabular}[c]{@{}c@{}}91.4 $\pm$ 0.07 \\  \textbf{79.55 $\pm$ 0.15}\end{tabular}  & \begin{tabular}[c]{@{}c@{}}88.46 $\pm$ 0.13 \\ 67.62 $\pm$ 0.03\end{tabular}  & \begin{tabular}[c]{@{}c@{}}85.14 $\pm$ 0.13 \\ 57.39 $\pm$ 0.13\end{tabular}  & \begin{tabular}[c]{@{}c@{}}81.41 $\pm$ 0.05 \\ 49.58 $\pm$ 0.08\end{tabular}  & \begin{tabular}[c]{@{}c@{}}77.18 $\pm$ 0.15 \\ 43.3 $\pm$ 0.11\end{tabular}   & \begin{tabular}[c]{@{}c@{}}72.9 $\pm$ 0.26 \\ 38.13 $\pm$ 0.15\end{tabular}    & \begin{tabular}[c]{@{}c@{}}70.39 $\pm$ 2.71 \\ 22.89 $\pm$ 15.26\end{tabular} & \begin{tabular}[c]{@{}c@{}}64.81 $\pm$ 11.58 \\ 9.6 $\pm$ 13.37\end{tabular}   \\ \hline
PGD-10     & \begin{tabular}[c]{@{}c@{}}91.25 $\pm$ 0.04 \\  \textbf{79.47 $\pm$ 0.13}\end{tabular} & \begin{tabular}[c]{@{}c@{}}88.34 $\pm$ 0.11 \\ \textbf{68.29 $\pm$ 0.24}\end{tabular}  & \begin{tabular}[c]{@{}c@{}}84.79 $\pm$ 0.11 \\ \textbf{58.85 $\pm$ 0.18}\end{tabular}  & \begin{tabular}[c]{@{}c@{}}80.71 $\pm$ 0.14 \\ \textbf{51.33 $\pm$ 0.31}\end{tabular}  & \begin{tabular}[c]{@{}c@{}}76.13 $\pm$ 0.35 \\ \textbf{45.02 $\pm$ 0.49}\end{tabular}  & \begin{tabular}[c]{@{}c@{}}71.24 $\pm$ 0.3 \\ \textbf{39.93 $\pm$ 0.5}\end{tabular}     & \begin{tabular}[c]{@{}c@{}}66.7 $\pm$ 0.39 \\ \textbf{36.02 $\pm$ 0.67}\end{tabular}   & \begin{tabular}[c]{@{}c@{}}62.11 $\pm$ 0.62 \\ \textbf{32.22 $\pm$ 0.64}\end{tabular}   \\ \hline
\end{tabular}
}
\end{table}

\begin{table}[h]
\renewcommand{\arraystretch}{1.5}
\centering
\scriptsize{
\begin{tabular}{c|c|c|c|c|c|c|c|c}
\multicolumn{4}{c}{\small{\textbf{PreActResNet18 -- CIFAR-100 Dataset}}} \\
\multicolumn{2}{c}{} \\
\toprule
           & $\epsilon$ = $\nicefrac{2}{255}$                                                                 & $\epsilon$ = $\nicefrac{4}{255}$                                                                & $\epsilon$ = $\nicefrac{6}{255}$                                                                 & $\epsilon$ = $\nicefrac{8}{255}$                                                                & $\epsilon$ = $\nicefrac{10}{255}$                                                                & $\epsilon$ = $\nicefrac{12}{255}$                                                                & $\epsilon$ = $\nicefrac{14}{255}$                                                               & $\epsilon$ = $\nicefrac{16}{255}$                                                                \\ \hline
           \midrule
\textbf{N-FGSM}     & \begin{tabular}[c]{@{}c@{}}69.12 $\pm$ 0.27 \\  \textbf{51.02 $\pm$ 0.34}\end{tabular} & \begin{tabular}[c]{@{}c@{}}64.0 $\pm$ 0.06 \\ \textbf{39.5 $\pm$ 0.12}\end{tabular}   & \begin{tabular}[c]{@{}c@{}}59.53 $\pm$ 0.02\\  \textbf{32.06 $\pm$ 0.37}\end{tabular}  & \begin{tabular}[c]{@{}c@{}}54.9 $\pm$ 0.2 \\ \textbf{26.46 $\pm$ 0.22}\end{tabular}   & \begin{tabular}[c]{@{}c@{}}50.6 $\pm$ 0.16 \\  \textbf{22.23 $\pm$ 0.17}\end{tabular}  & \begin{tabular}[c]{@{}c@{}}46.06 $\pm$ 0.14 \\ \textbf{18.95 $\pm$ 0.15}\end{tabular}  & \begin{tabular}[c]{@{}c@{}}41.67 $\pm$ 0.25 \\ \textbf{16.33 $\pm$ 0.15}\end{tabular} & \begin{tabular}[c]{@{}c@{}}37.91 $\pm$ 0.11 \\  \textbf{14.34 $\pm$ 0.07}\end{tabular} \\ \hline
Grad Align & \begin{tabular}[c]{@{}c@{}}68.96 $\pm$ 0.15 \\ \textbf{51.31 $\pm$ 0.12}\end{tabular}  & \begin{tabular}[c]{@{}c@{}}64.71 $\pm$ 0.16 \\ \textbf{39.37 $\pm$ 0.25}\end{tabular} & \begin{tabular}[c]{@{}c@{}}60.42 $\pm$ 0.23 \\\textbf{ 31.91 $\pm$ 0.28}\end{tabular}  & \begin{tabular}[c]{@{}c@{}}56.53 $\pm$ 0.31 \\ 25.8 $\pm$ 0.14\end{tabular}  & \begin{tabular}[c]{@{}c@{}}54.06 $\pm$ 0.44 \\  18.7 $\pm$ 1.92\end{tabular}  & \begin{tabular}[c]{@{}c@{}}48.87 $\pm$ 0.32 \\  17.86 $\pm$ 0.04\end{tabular} & \begin{tabular}[c]{@{}c@{}}43.84 $\pm$ 0.14 \\ 15.51 $\pm$ 0.16\end{tabular} & \begin{tabular}[c]{@{}c@{}}38.93 $\pm$ 0.21 \\ 13.62 $\pm$ 0.19\end{tabular}  \\ \hline
FGSM       & \begin{tabular}[c]{@{}c@{}}69.01 $\pm$ 0.13 \\ 51.3 $\pm$ 0.19\end{tabular}   & \begin{tabular}[c]{@{}c@{}}64.47 $\pm$ 0.15 \\ 39.7 $\pm$ 0.16\end{tabular}  & \begin{tabular}[c]{@{}c@{}}63.85 $\pm$ 2.18 \\ 10.93 $\pm$ 14.64\end{tabular} & \begin{tabular}[c]{@{}c@{}}53.42 $\pm$ 0.65 \\ 0.0 $\pm$ 0.0\end{tabular}    & \begin{tabular}[c]{@{}c@{}}45.06 $\pm$ 2.29 \\ 0.0 $\pm$ 0.0\end{tabular}     & \begin{tabular}[c]{@{}c@{}}46.14 $\pm$ 2.58 \\ 0.0 $\pm$ 0.0\end{tabular}     & \begin{tabular}[c]{@{}c@{}}41.66 $\pm$ 0.88\\  0.0 $\pm$ 0.0\end{tabular}    & \begin{tabular}[c]{@{}c@{}}44.68 $\pm$ 1.74 \\ 0.0 $\pm$ 0.0\end{tabular}     \\ \hline
RS-FGSM    & \begin{tabular}[c]{@{}c@{}}69.83 $\pm$ 0.29 \\ 50.13 $\pm$ 0.32\end{tabular}  & \begin{tabular}[c]{@{}c@{}}65.9 $\pm$ 0.36 \\ 38.36 $\pm$ 0.19\end{tabular}  & \begin{tabular}[c]{@{}c@{}}62.15 $\pm$ 0.23\\  30.82 $\pm$ 0.08\end{tabular}  & \begin{tabular}[c]{@{}c@{}}55.26 $\pm$ 6.86\\  0.01 $\pm$ 0.01\end{tabular}  & \begin{tabular}[c]{@{}c@{}}32.33 $\pm$ 12.12 \\  0.0 $\pm$ 0.0\end{tabular}   & \begin{tabular}[c]{@{}c@{}}36.07 $\pm$ 2.59 \\ 0.0 $\pm$ 0.0\end{tabular}     & \begin{tabular}[c]{@{}c@{}}21.52 $\pm$ 5.56 \\ 0.0 $\pm$ 0.0\end{tabular}    & \begin{tabular}[c]{@{}c@{}}20.38 $\pm$ 6.15 \\ 0.0 $\pm$ 0.0\end{tabular}     \\ \hline
Kim et. al.       & \begin{tabular}[c]{@{}c@{}}72.92 $\pm$ 0.41 \\ 44.19 $\pm$ 0.25\end{tabular}  & \begin{tabular}[c]{@{}c@{}}70.16 $\pm$ 0.07 \\ 30.63 $\pm$ 0.28\end{tabular} & \begin{tabular}[c]{@{}c@{}}67.98 $\pm$ 0.19 \\ 22.0 $\pm$ 0.02\end{tabular}   & \begin{tabular}[c]{@{}c@{}}68.07 $\pm$ 0.1 \\ 12.75 $\pm$ 0.21\end{tabular}  & \begin{tabular}[c]{@{}c@{}}68.37 $\pm$ 0.21 \\  6.98 $\pm$ 0.23\end{tabular}  & \begin{tabular}[c]{@{}c@{}}74.09 $\pm$ 0.06 \\ 0.0 $\pm$ 0.0\end{tabular}     & \begin{tabular}[c]{@{}c@{}}74.06 $\pm$ 0.34 \\ 0.0 $\pm$ 0.0\end{tabular}    & \begin{tabular}[c]{@{}c@{}}74.01 $\pm$ 0.36 \\ 0.0 $\pm$ 0.0\end{tabular}     \\ \hline
AT Free    & \begin{tabular}[c]{@{}c@{}}63.01 $\pm$ 0.19 \\ 45.7 $\pm$ 0.33\end{tabular}   & \begin{tabular}[c]{@{}c@{}}59.41 $\pm$ 0.27 \\ 35.95 $\pm$ 0.09\end{tabular} & \begin{tabular}[c]{@{}c@{}}55.43 $\pm$ 0.37\\  29.37 $\pm$ 0.21\end{tabular}  & \begin{tabular}[c]{@{}c@{}}51.91 $\pm$ 0.08 \\ 24.32 $\pm$ 0.4\end{tabular}  & \begin{tabular}[c]{@{}c@{}}48.11 $\pm$ 0.09 \\ 20.64 $\pm$ 0.22\end{tabular}  & \begin{tabular}[c]{@{}c@{}}43.48 $\pm$ 1.25 \\ 5.71 $\pm$ 8.05\end{tabular}   & \begin{tabular}[c]{@{}c@{}}18.33 $\pm$ 4.86 \\ 0.0 $\pm$ 0.0\end{tabular}    & \begin{tabular}[c]{@{}c@{}}20.43 $\pm$ 11.25 \\  0.0 $\pm$ 0.0\end{tabular}   \\ \hline
ZeroGrad   & \begin{tabular}[c]{@{}c@{}}69.35 $\pm$ 0.36 \\ \textbf{51.1 $\pm$ 0.09}\end{tabular}   & \begin{tabular}[c]{@{}c@{}}64.59 $\pm$ 0.32 \\ \textbf{39.38 $\pm$ 0.15}\end{tabular} & \begin{tabular}[c]{@{}c@{}}60.69 $\pm$ 0.09 \\ 31.72 $\pm$ 0.21\end{tabular}  & \begin{tabular}[c]{@{}c@{}}56.94 $\pm$ 0.13 \\ 25.87 $\pm$ 0.09\end{tabular} & \begin{tabular}[c]{@{}c@{}}54.55 $\pm$ 0.17 \\  19.49 $\pm$ 0.08\end{tabular} & \begin{tabular}[c]{@{}c@{}}52.97 $\pm$ 0.34\\ 14.32 $\pm$ 0.08\end{tabular}   & \begin{tabular}[c]{@{}c@{}}50.87 $\pm$ 0.26 \\ 10.92 $\pm$ 0.59\end{tabular} & \begin{tabular}[c]{@{}c@{}}50.73 $\pm$ 0.3 \\ 7.3 $\pm$ 0.16\end{tabular}     \\ \hline
MultiGrad  & \begin{tabular}[c]{@{}c@{}}69.01 $\pm$ 0.16 \\ 51.15 $\pm$ 0.03\end{tabular}  & \begin{tabular}[c]{@{}c@{}}64.44 $\pm$ 0.11 \\ 39.16 $\pm$ 0.03\end{tabular} & \begin{tabular}[c]{@{}c@{}}60.65 $\pm$ 0.26\\  31.73 $\pm$ 0.09\end{tabular}  & \begin{tabular}[c]{@{}c@{}}56.84 $\pm$ 0.2 \\ 25.96 $\pm$ 0.11\end{tabular}  & \begin{tabular}[c]{@{}c@{}}53.62 $\pm$ 0.25 \\ 21.37 $\pm$ 0.16\end{tabular}  & \begin{tabular}[c]{@{}c@{}}53.05 $\pm$ 1.85 \\ 9.57 $\pm$ 7.32\end{tabular}   & \begin{tabular}[c]{@{}c@{}}48.28 $\pm$ 0.66 \\  3.2 $\pm$ 4.49\end{tabular}  & \begin{tabular}[c]{@{}c@{}}45.28 $\pm$ 11.14 \\ 0.0 $\pm$ 0.0\end{tabular}    \\ \hline \hline
PGD-2      & \begin{tabular}[c]{@{}c@{}}69.18 $\pm$ 0.1 \\ \textbf{51.36 $\pm$ 0.03}\end{tabular}   & \begin{tabular}[c]{@{}c@{}}64.32 $\pm$ 0.14 \\ 40.06 $\pm$ 0.14\end{tabular} & \begin{tabular}[c]{@{}c@{}}60.21 $\pm$ 0.13 \\  32.99 $\pm$ 0.24\end{tabular} & \begin{tabular}[c]{@{}c@{}}55.8 $\pm$ 0.16 \\ 27.38 $\pm$ 0.16\end{tabular}  & \begin{tabular}[c]{@{}c@{}}51.68 $\pm$ 0.1 \\ 23.39 $\pm$ 0.19\end{tabular}   & \begin{tabular}[c]{@{}c@{}}48.2 $\pm$ 0.1 \\ 19.83 $\pm$ 0.29\end{tabular}    & \begin{tabular}[c]{@{}c@{}}46.14 $\pm$ 1.24 \\ 10.55 $\pm$ 7.51\end{tabular} & \begin{tabular}[c]{@{}c@{}}37.97 $\pm$ 10.52 \\ 4.79 $\pm$ 6.75\end{tabular}  \\ \hline
PGD-10     & \begin{tabular}[c]{@{}c@{}}68.83 $\pm$ 0.07 \\ \textbf{51.51 $\pm$ 0.27}\end{tabular}  & \begin{tabular}[c]{@{}c@{}}63.87 $\pm$ 0.09 \\ \textbf{40.59 $\pm$ 0.36}\end{tabular} & \begin{tabular}[c]{@{}c@{}}59.37 $\pm$ 0.07 \\ \textbf{33.65 $\pm$ 0.02}\end{tabular}  & \begin{tabular}[c]{@{}c@{}}54.79 $\pm$ 0.38 \\ \textbf{28.55 $\pm$ 0.27}\end{tabular} & \begin{tabular}[c]{@{}c@{}}50.53 $\pm$ 0.15 \\ \textbf{24.17 $\pm$ 0.12}\end{tabular}  & \begin{tabular}[c]{@{}c@{}}46.05 $\pm$ 0.21 \\ \textbf{21.2 $\pm$ 0.12}\end{tabular}   & \begin{tabular}[c]{@{}c@{}}41.76 $\pm$ 0.07\\  \textbf{18.72 $\pm$ 0.06}\end{tabular} & \begin{tabular}[c]{@{}c@{}}37.81 $\pm$ 0.14 \\ \textbf{16.59 $\pm$ 0.16}\end{tabular}  \\ \hline
\end{tabular}
}
\end{table}

\begin{table}[h]
\centering
\renewcommand{\arraystretch}{1.5}
\scriptsize{
\begin{tabular}{c|c|c|c|c|c|c}
\multicolumn{4}{c}{\small{\textbf{PreActResNet18 -- SVHN Dataset}}} \\
\multicolumn{2}{c}{} \\
\toprule
           & $\epsilon$ = $\nicefrac{2}{255}$                                                                & $\epsilon$ = $\nicefrac{4}{255}$                                                                & $\epsilon$ = $\nicefrac{6}{255}$                                                     & $\epsilon$ = $\nicefrac{8}{255}$                                                                & $\epsilon$ = $\nicefrac{10}{255}$                                                               & $\epsilon$ = $\nicefrac{12}{255}$                                                               \\ \hline
           \midrule
\textbf{N-FGSM }    & \begin{tabular}[c]{@{}c@{}}96.01 $\pm$ 0.04\\ \textbf{ 86.44 $\pm$ 0.1}\end{tabular}  & \begin{tabular}[c]{@{}c@{}}94.54 $\pm$ 0.15\\  \textbf{72.53 $\pm$ 0.19}\end{tabular} & \begin{tabular}[c]{@{}c@{}}92.25 $\pm$ 0.33\\  58.42 $\pm$ 0.14\end{tabular} & \begin{tabular}[c]{@{}c@{}}89.56 $\pm$ 0.49\\ \textbf{45.63 $\pm$ 0.11}\end{tabular}   & \begin{tabular}[c]{@{}c@{}}86.74 $\pm$ 0.86\\ \textbf{33.96 $\pm$ 0.49}\end{tabular}  & \begin{tabular}[c]{@{}c@{}}81.48 $\pm$ 1.64\\ \textbf{26.13 $\pm$ 0.81}\end{tabular}  \\ \hline
Grad Align & \begin{tabular}[c]{@{}c@{}}96.02 $\pm$ 0.05\\  \textbf{86.43 $\pm$ 0.1}\end{tabular}  & \begin{tabular}[c]{@{}c@{}}94.56 $\pm$ 0.21\\ 72.12 $\pm$ 0.19\end{tabular}  & \begin{tabular}[c]{@{}c@{}}92.53 $\pm$ 0.24\\  57.34 $\pm$ 0.24\end{tabular} & \begin{tabular}[c]{@{}c@{}}90.1 $\pm$ 0.34\\  43.85 $\pm$ 0.14\end{tabular}   & \begin{tabular}[c]{@{}c@{}}87.23 $\pm$ 0.75\\  32.87 $\pm$ 0.33\end{tabular} & \begin{tabular}[c]{@{}c@{}}84.01 $\pm$ 0.46\\  23.62 $\pm$ 0.41\end{tabular} \\ \hline
FGSM       & \begin{tabular}[c]{@{}c@{}}96.04 $\pm$ 0.07\\ \textbf{86.5 $\pm$ 0.05}\end{tabular}   & \begin{tabular}[c]{@{}c@{}}95.67 $\pm$ 0.07\\  13.61 $\pm$ 5.83\end{tabular} & \begin{tabular}[c]{@{}c@{}}93.73 $\pm$ 0.68\\  0.56 $\pm$ 0.72\end{tabular}  & \begin{tabular}[c]{@{}c@{}}91.74 $\pm$ 0.86\\  0.26 $\pm$ 0.36\end{tabular}   & \begin{tabular}[c]{@{}c@{}}90.76 $\pm$ 0.63\\ 0.07 $\pm$ 0.1\end{tabular}    & \begin{tabular}[c]{@{}c@{}}87.17 $\pm$ 0.43\\ 0.0 $\pm$ 0.0\end{tabular}     \\ \hline
RS-FGSM    & \begin{tabular}[c]{@{}c@{}}96.18 $\pm$ 0.11\\ 86.16 $\pm$ 0.14\end{tabular}  & \begin{tabular}[c]{@{}c@{}}95.09 $\pm$ 0.09\\  71.28 $\pm$ 0.4\end{tabular}  & \begin{tabular}[c]{@{}c@{}}95.11 $\pm$ 0.44\\  0.11 $\pm$ 0.08\end{tabular}  & \begin{tabular}[c]{@{}c@{}}94.46 $\pm$ 0.16\\  0.0 $\pm$ 0.0\end{tabular}     & \begin{tabular}[c]{@{}c@{}}93.88 $\pm$ 0.24\\  0.0 $\pm$ 0.0\end{tabular}    & \begin{tabular}[c]{@{}c@{}}92.74 $\pm$ 0.5\\ 0.0 $\pm$ 0.0\end{tabular}      \\ \hline
Kim et. al.       & \begin{tabular}[c]{@{}c@{}}96.35 $\pm$ 0.02\\  83.26 $\pm$ 0.24\end{tabular} & \begin{tabular}[c]{@{}c@{}}95.25 $\pm$ 0.08\\  66.32 $\pm$ 0.63\end{tabular} & \begin{tabular}[c]{@{}c@{}}94.83 $\pm$ 0.02\\  48.27 $\pm$ 0.52\end{tabular} & \begin{tabular}[c]{@{}c@{}}94.88 $\pm$ 0.29\\  31.8 $\pm$ 1.1\end{tabular}    & \begin{tabular}[c]{@{}c@{}}96.61 $\pm$ 0.09\\  0.18 $\pm$ 0.21\end{tabular}  & \begin{tabular}[c]{@{}c@{}}96.61 $\pm$ 0.01\\ 0.0 $\pm$ 0.0\end{tabular}     \\ \hline
AT Free    & \begin{tabular}[c]{@{}c@{}}95.01 $\pm$ 0.09\\  84.55 $\pm$ 0.27\end{tabular} & \begin{tabular}[c]{@{}c@{}}93.66 $\pm$ 0.12\\  71.61 $\pm$ 0.75\end{tabular} & \begin{tabular}[c]{@{}c@{}}91.72 $\pm$ 0.29\\  \textbf{59.31 $\pm$ 1.0}\end{tabular}  & \begin{tabular}[c]{@{}c@{}}91.29 $\pm$ 4.07\\ 0.01 $\pm$ 0.0\end{tabular}     & \begin{tabular}[c]{@{}c@{}}91.86 $\pm$ 3.66\\ 0.0 $\pm$ 0.0\end{tabular}     & \begin{tabular}[c]{@{}c@{}}92.36 $\pm$ 1.0\\ 0.0 $\pm$ 0.0\end{tabular}      \\ \hline
ZeroGrad   & \begin{tabular}[c]{@{}c@{}}96.06 $\pm$ 0.03\\ \textbf{86.43 $\pm$ 0.1}\end{tabular}   & \begin{tabular}[c]{@{}c@{}}94.81 $\pm$ 0.16\\  71.59 $\pm$ 0.22\end{tabular} & \begin{tabular}[c]{@{}c@{}}93.53 $\pm$ 0.26\\ 51.72 $\pm$ 0.53\end{tabular}  & \begin{tabular}[c]{@{}c@{}}92.42 $\pm$ 1.29\\  35.93 $\pm$ 2.73\end{tabular}  & \begin{tabular}[c]{@{}c@{}}90.34 $\pm$ 0.32\\  21.34 $\pm$ 0.31\end{tabular} & \begin{tabular}[c]{@{}c@{}}88.09 $\pm$ 0.4\\  14.14 $\pm$ 0.32\end{tabular}  \\ \hline
MultiGrad  & \begin{tabular}[c]{@{}c@{}}96.01 $\pm$ 0.08\\ \textbf{86.4 $\pm$ 0.08}\end{tabular}   & \begin{tabular}[c]{@{}c@{}}94.71 $\pm$ 0.17\\ \textbf{71.98 $\pm$ 0.26}\end{tabular}  & \begin{tabular}[c]{@{}c@{}}95.75 $\pm$ 0.58\\ 28.1 $\pm$ 18.85\end{tabular}  & \begin{tabular}[c]{@{}c@{}}94.86 $\pm$ 0.97\\  11.49 $\pm$ 16.19\end{tabular} & \begin{tabular}[c]{@{}c@{}}94.7 $\pm$ 0.12\\ 0.0 $\pm$ 0.0\end{tabular}      & \begin{tabular}[c]{@{}c@{}}94.48 $\pm$ 0.19\\  0.0 $\pm$ 0.0\end{tabular}    \\ \hline \hline
PGD-2      & \begin{tabular}[c]{@{}c@{}}96.03 $\pm$ 0.14\\ 86.72 $\pm$ 0.06\end{tabular}  & \begin{tabular}[c]{@{}c@{}}94.66 $\pm$ 0.1\\ 73.29 $\pm$ 0.29\end{tabular}   & \begin{tabular}[c]{@{}c@{}}93.77 $\pm$ 0.61\\  60.53 $\pm$ 0.73\end{tabular} & \begin{tabular}[c]{@{}c@{}}94.63 $\pm$ 1.29\\  20.68 $\pm$ 18.56\end{tabular} & \begin{tabular}[c]{@{}c@{}}84.09 $\pm$ 14.99\\  0.41 $\pm$ 0.29\end{tabular} & \begin{tabular}[c]{@{}c@{}}94.16 $\pm$ 0.54\\  0.02 $\pm$ 0.03\end{tabular}  \\ \hline
PGD-10     & \begin{tabular}[c]{@{}c@{}}95.92 $\pm$ 0.08\\ \textbf{86.94 $\pm$ 0.14}\end{tabular}  & \begin{tabular}[c]{@{}c@{}}94.37 $\pm$ 0.13\\\textbf{ 74.76 $\pm$ 0.19}\end{tabular}  & \begin{tabular}[c]{@{}c@{}}92.46 $\pm$ 0.25\\  \textbf{63.9 $\pm$ 0.48}\end{tabular}  & \begin{tabular}[c]{@{}c@{}}89.67 $\pm$ 0.34\\  \textbf{53.95 $\pm$ 0.55}\end{tabular}  & \begin{tabular}[c]{@{}c@{}}85.75 $\pm$ 0.65\\ \textbf{44.91 $\pm$ 0.45}\end{tabular}  & \begin{tabular}[c]{@{}c@{}}80.08 $\pm$ 0.93\\ \textbf{ 37.65 $\pm$ 0.53}\end{tabular} \\ \hline
\end{tabular}

}
\end{table}

\begin{table}[h]
\centering
\renewcommand{\arraystretch}{1.5}
\scriptsize{
\begin{tabular}{c|c|c|c|c|c|c|c|c}
\multicolumn{4}{c}{\small{\textbf{WideResNet28-10 -- CIFAR-10 Dataset}}} \\
\multicolumn{2}{c}{} \\
\toprule
           & $\epsilon$ = $\nicefrac{2}{255}$                                                                & $\epsilon$ = $\nicefrac{4}{255}$                                                                & $\epsilon$ = $\nicefrac{6}{255}$                                                                 & $\epsilon$ = $\nicefrac{8}{255}$                                                                & $\epsilon$ = $\nicefrac{10}{255}$                                                               & $\epsilon$ = $\nicefrac{12}{255}$                            & $\epsilon$ = $\nicefrac{14}{255}$                                    & $\epsilon$ = $\nicefrac{16}{255}$                                                               \\ \hline
           \midrule
\textbf{N-FGSM}     & \begin{tabular}[c]{@{}c@{}}92.51 $\pm$ 0.11 \\ \textbf{81.43 $\pm$ 0.3}\end{tabular}  & \begin{tabular}[c]{@{}c@{}}89.65 $\pm$ 0.09 \\ 69.11 $\pm$ 0.24\end{tabular} & \begin{tabular}[c]{@{}c@{}}85.8 $\pm$ 0.23 \\ 58.29 $\pm$ 0.14\end{tabular}   & \begin{tabular}[c]{@{}c@{}}81.59 $\pm$ 0.32 \\ 49.53 $\pm$ 0.25\end{tabular} & \begin{tabular}[c]{@{}c@{}}76.92 $\pm$ 0.04 \\ \textbf{42.37 $\pm$ 0.36}\end{tabular} & \begin{tabular}[c]{@{}c@{}}72.13 $\pm$ 0.15 \\ \textbf{36.85 $\pm$ 0.2}\end{tabular}  & \begin{tabular}[c]{@{}c@{}}67.82 $\pm$ 0.43 \\ \textbf{31.66 $\pm$ 0.6}\end{tabular}   & \begin{tabular}[c]{@{}c@{}}56.73 $\pm$ 0.42 \\ 25.01 $\pm$ 0.23\end{tabular} \\ \hline
Grad Align & \begin{tabular}[c]{@{}c@{}}92.59 $\pm$ 0.05 \\ \textbf{81.33 $\pm$ 0.4}\end{tabular}  & \begin{tabular}[c]{@{}c@{}}89.95 $\pm$ 0.3 \\ \textbf{69.81 $\pm$ 0.47}\end{tabular}  & \begin{tabular}[c]{@{}c@{}}86.98 $\pm$ 0.06 \\ \textbf{59.0 $\pm$ 0.13}\end{tabular}   & \begin{tabular}[c]{@{}c@{}}83.19 $\pm$ 0.26 \\ \textbf{50.0 $\pm$ 0.05}\end{tabular}  & \begin{tabular}[c]{@{}c@{}}79.35 $\pm$ 0.26 \\ 41.48 $\pm$ 0.51\end{tabular} & \begin{tabular}[c]{@{}c@{}}73.79 $\pm$ 0.72 \\ 35.06 $\pm$ 0.74\end{tabular} & \begin{tabular}[c]{@{}c@{}}66.38 $\pm$ 0.53 \\ 30.83 $\pm$ 0.39\end{tabular}  & \begin{tabular}[c]{@{}c@{}}57.75 $\pm$ 0.75 \\ \textbf{26.26 $\pm$ 0.13}\end{tabular} \\ \hline
FGSM       & \begin{tabular}[c]{@{}c@{}}92.65 $\pm$ 0.17 \\ \textbf{81.38 $\pm$ 0.22}\end{tabular} & \begin{tabular}[c]{@{}c@{}}90.06 $\pm$ 0.18 \\ \textbf{69.59 $\pm$ 0.25}\end{tabular} & \begin{tabular}[c]{@{}c@{}}87.99 $\pm$ 1.3 \\  38.69 $\pm$ 26.54\end{tabular} & \begin{tabular}[c]{@{}c@{}}86.46 $\pm$ 0.45 \\ 0.0 $\pm$ 0.0\end{tabular}    & \begin{tabular}[c]{@{}c@{}}82.67 $\pm$ 1.78 \\ 0.0 $\pm$ 0.0\end{tabular}    & \begin{tabular}[c]{@{}c@{}}80.14 $\pm$ 1.2 \\ 0.0 $\pm$ 0.0\end{tabular}     & \begin{tabular}[c]{@{}c@{}}74.54 $\pm$ 4.01\\  0.0 $\pm$ 0.0\end{tabular}     & \begin{tabular}[c]{@{}c@{}}71.56 $\pm$ 3.78 \\ 0.0 $\pm$ 0.0\end{tabular}    \\ \hline
RS-FGSM    & \begin{tabular}[c]{@{}c@{}}92.85 $\pm$ 0.1\\  80.9 $\pm$ 0.13\end{tabular}   & \begin{tabular}[c]{@{}c@{}}90.73 $\pm$ 0.2 \\ 68.23 $\pm$ 0.17\end{tabular}  & \begin{tabular}[c]{@{}c@{}}88.24 $\pm$ 0.19 \\ 57.21 $\pm$ 0.17\end{tabular}  & \begin{tabular}[c]{@{}c@{}}83.64 $\pm$ 1.74 \\ 0.0 $\pm$ 0.0\end{tabular}    & \begin{tabular}[c]{@{}c@{}}82.1 $\pm$ 1.45 \\ 0.0 $\pm$ 0.0\end{tabular}     & \begin{tabular}[c]{@{}c@{}}78.62 $\pm$ 0.7 \\ 0.0 $\pm$ 0.0\end{tabular}     & \begin{tabular}[c]{@{}c@{}}73.25 $\pm$ 8.16 \\ 0.0 $\pm$ 0.0\end{tabular}     & \begin{tabular}[c]{@{}c@{}}68.64 $\pm$ 4.3 \\ 0.0 $\pm$ 0.0\end{tabular}     \\ \hline
RandAlpha  & \begin{tabular}[c]{@{}c@{}}93.37 $\pm$ 0.22\\  77.67 $\pm$ 0.66\end{tabular} & \begin{tabular}[c]{@{}c@{}}92.17 $\pm$ 0.21 \\ 63.73 $\pm$ 0.31\end{tabular} & \begin{tabular}[c]{@{}c@{}}90.71 $\pm$ 0.14 \\ 50.4 $\pm$ 0.14\end{tabular}   & \begin{tabular}[c]{@{}c@{}}89.16 $\pm$ 0.19 \\ 39.37 $\pm$ 0.42\end{tabular} & \begin{tabular}[c]{@{}c@{}}87.44 $\pm$ 0.31 \\ 30.13 $\pm$ 0.9\end{tabular}  & \begin{tabular}[c]{@{}c@{}}85.69 $\pm$ 0.28 \\ 23.13 $\pm$ 0.33\end{tabular} & \begin{tabular}[c]{@{}c@{}}83.98 $\pm$ 0.24 \\ 16.0 $\pm$ 0.22\end{tabular}   & \begin{tabular}[c]{@{}c@{}}83.23 $\pm$ 0.46\\  8.47 $\pm$ 0.66\end{tabular}  \\ \hline
AT Free    & \begin{tabular}[c]{@{}c@{}}90.66 $\pm$ 0.25 \\ 77.0 $\pm$ 0.27\end{tabular}  & \begin{tabular}[c]{@{}c@{}}88.37 $\pm$ 0.15 \\ 64.25 $\pm$ 0.33\end{tabular} & \begin{tabular}[c]{@{}c@{}}86.11 $\pm$ 0.29 \\ 53.76 $\pm$ 0.48\end{tabular}  & \begin{tabular}[c]{@{}c@{}}83.5 $\pm$ 0.27 \\ 44.85 $\pm$ 0.39\end{tabular}  & \begin{tabular}[c]{@{}c@{}}80.52 $\pm$ 0.32 \\ 31.87 $\pm$ 5.53\end{tabular} & \begin{tabular}[c]{@{}c@{}}83.59 $\pm$ 1.35 \\ 0.0 $\pm$ 0.0\end{tabular}    & \begin{tabular}[c]{@{}c@{}}39.58 $\pm$ 15.8 \\ 0.0 $\pm$ 0.0\end{tabular}     & \begin{tabular}[c]{@{}c@{}}42.59 $\pm$ 27.96 \\ 0.0 $\pm$ 0.0\end{tabular}   \\ \hline
ZeroGrad   & \begin{tabular}[c]{@{}c@{}}92.62 $\pm$ 0.11 \\\textbf{ 81.42 $\pm$ 0.28}\end{tabular} & \begin{tabular}[c]{@{}c@{}}90.17 $\pm$ 0.05 \\ 69.28 $\pm$ 0.29\end{tabular} & \begin{tabular}[c]{@{}c@{}}86.98 $\pm$ 0.28 \\ 58.4 $\pm$ 0.14\end{tabular}   & \begin{tabular}[c]{@{}c@{}}84.25 $\pm$ 0.28 \\ 48.29 $\pm$ 0.16\end{tabular} & \begin{tabular}[c]{@{}c@{}}81.72 $\pm$ 0.29 \\ 36.08 $\pm$ 0.29\end{tabular} & \begin{tabular}[c]{@{}c@{}}79.24 $\pm$ 0.82 \\ 28.24 $\pm$ 1.79\end{tabular} & \begin{tabular}[c]{@{}c@{}}78.14 $\pm$ 0.46\\  18.54 $\pm$ 0.31\end{tabular}  & \begin{tabular}[c]{@{}c@{}}75.34 $\pm$ 0.12 \\ 14.6 $\pm$ 0.12\end{tabular}  \\ \hline
MultiGrad  & \begin{tabular}[c]{@{}c@{}}92.64 $\pm$ 0.1\\ \textbf{81.19 $\pm$ 0.28}\end{tabular}   & \begin{tabular}[c]{@{}c@{}}90.18 $\pm$ 0.13\\ 69.3 $\pm$ 0.2\end{tabular}    & \begin{tabular}[c]{@{}c@{}}87.11 $\pm$ 0.36\\ 57.98 $\pm$ 0.08\end{tabular}   & \begin{tabular}[c]{@{}c@{}}83.87 $\pm$ 0.46\\ 48.74 $\pm$ 0.09\end{tabular}  & \begin{tabular}[c]{@{}c@{}}80.89 $\pm$ 0.14\\ 41.22 $\pm$ 0.57\end{tabular}  & \begin{tabular}[c]{@{}c@{}}82.88 $\pm$ 2.85\\ 4.46 $\pm$ 6.09\end{tabular}   & \begin{tabular}[c]{@{}c@{}}86.6 $\pm$ 1.52\\ 0.0 $\pm$ 0.0\end{tabular}       & \begin{tabular}[c]{@{}c@{}}85.46 $\pm$ 3.73\\ 0.0 $\pm$ 0.0\end{tabular}     \\ \hline \hline
PGD-2      & \begin{tabular}[c]{@{}c@{}}92.69 $\pm$ 0.14 \\ \textbf{81.54 $\pm$ 0.18}\end{tabular} & \begin{tabular}[c]{@{}c@{}}90.18 $\pm$ 0.19 \\ 69.87 $\pm$ 0.26\end{tabular} & \begin{tabular}[c]{@{}c@{}}86.87 $\pm$ 0.18 \\ 59.4 $\pm$ 0.19\end{tabular}   & \begin{tabular}[c]{@{}c@{}}83.31 $\pm$ 0.16 \\ 50.88 $\pm$ 0.16\end{tabular} & \begin{tabular}[c]{@{}c@{}}79.61 $\pm$ 0.47 \\ 43.94 $\pm$ 0.24\end{tabular} & \begin{tabular}[c]{@{}c@{}}75.81 $\pm$ 0.24 \\ 37.77 $\pm$ 0.57\end{tabular} & \begin{tabular}[c]{@{}c@{}}71.41 $\pm$ 1.38 \\ 21.06 $\pm$ 13.39\end{tabular} & \begin{tabular}[c]{@{}c@{}}67.2 $\pm$ 14.94 \\ 0.0 $\pm$ 0.0\end{tabular}    \\ \hline
PGD-10     & \begin{tabular}[c]{@{}c@{}}92.24 $\pm$ 0.31\\  81.18 $\pm$ 0.57\end{tabular} & \begin{tabular}[c]{@{}c@{}}89.65 $\pm$ 0.33 \\ \textbf{70.34 $\pm$ 0.26}\end{tabular} & \begin{tabular}[c]{@{}c@{}}86.91 $\pm$ 0.51 \\ \textbf{60.59 $\pm$ 0.21}\end{tabular}  & \begin{tabular}[c]{@{}c@{}}82.82 $\pm$ 0.7\\  \textbf{52.58 $\pm$ 0.2}\end{tabular}   & \begin{tabular}[c]{@{}c@{}}78.63 $\pm$ 0.66 \\ \textbf{45.92 $\pm$ 0.38}\end{tabular} & \begin{tabular}[c]{@{}c@{}}74.0 $\pm$ 0.67 \\ \textbf{40.44 $\pm$ 0.17}\end{tabular}  & \begin{tabular}[c]{@{}c@{}}68.6 $\pm$ 0.58 \\ \textbf{35.98 $\pm$ 0.56}\end{tabular}   & \begin{tabular}[c]{@{}c@{}}64.17 $\pm$ 0.72 \\ \textbf{32.5 $\pm$ 0.61}\end{tabular}  \\ \hline
\end{tabular}
}
\end{table}

\begin{table}[h]
\centering
\renewcommand{\arraystretch}{1.5}
\scriptsize{
\begin{tabular}{c|c|c|c|c|c|c|c|c}
\multicolumn{4}{c}{\small{\textbf{WideResNet28-10 -- CIFAR-100 Dataset}}} \\
\multicolumn{2}{c}{} \\
\toprule
           & $\epsilon$ = $\nicefrac{2}{255}$                                                                 & $\epsilon$ = $\nicefrac{5}{255}$                                                                 & $\epsilon$ = $\nicefrac{6}{255}$                                                                & $\epsilon$ = $\nicefrac{8}{255}$                                                                & $\epsilon$ = $\nicefrac{10}{255}$                                                               & $\epsilon$ = $\nicefrac{12}{255}$                                                               & $\epsilon$ = $\nicefrac{14}{255}$                                                               & $\epsilon$ = $\nicefrac{16}{255}$                                                                \\ 
           \hline
           \midrule
\textbf{N-FGSM }    & \begin{tabular}[c]{@{}c@{}}71.56 $\pm$ 0.13 \\ \textbf{52.23 $\pm$ 0.33}\end{tabular}  & \begin{tabular}[c]{@{}c@{}}66.49 $\pm$ 0.46 \\  \textbf{39.93 $\pm$ 0.37}\end{tabular} & \begin{tabular}[c]{@{}c@{}}61.38 $\pm$ 0.68 \\ 30.97 $\pm$ 0.21\end{tabular} & \begin{tabular}[c]{@{}c@{}}56.23 $\pm$ 0.59 \\ \textbf{26.77 $\pm$ 0.65}\end{tabular} & \begin{tabular}[c]{@{}c@{}}51.54 $\pm$ 0.63\\  \textbf{23.03 $\pm$ 0.54}\end{tabular} & \begin{tabular}[c]{@{}c@{}}46.43 $\pm$ 0.61 \\ \textbf{19.3 $\pm$ 0.59}\end{tabular}  & \begin{tabular}[c]{@{}c@{}}42.11 $\pm$ 0.32 \\ \textbf{16.67 $\pm$ 0.4}\end{tabular}  & \begin{tabular}[c]{@{}c@{}}38.34 $\pm$ 0.47 \\ \textbf{14.27 $\pm$ 0.33}\end{tabular}  \\ \hline
Grad Align & \begin{tabular}[c]{@{}c@{}}71.68 $\pm$ 0.33 \\  51.5 $\pm$ 0.45\end{tabular}  & \begin{tabular}[c]{@{}c@{}}67.09 $\pm$ 0.19 \\ \textbf{39.9 $\pm$ 0.42}\end{tabular}   & \begin{tabular}[c]{@{}c@{}}62.86 $\pm$ 0.1 \\ \textbf{32.0 $\pm$ 0.22}\end{tabular}   & \begin{tabular}[c]{@{}c@{}}58.55 $\pm$ 0.41 \\ \textbf{26.9 $\pm$ 0.62}\end{tabular}  & \begin{tabular}[c]{@{}c@{}}53.85 $\pm$ 0.73 \\ \textbf{22.63 $\pm$ 0.62}\end{tabular} & \begin{tabular}[c]{@{}c@{}}46.94 $\pm$ 0.86 \\ \textbf{19.9 $\pm$ 0.65}\end{tabular}  & \begin{tabular}[c]{@{}c@{}}42.63 $\pm$ 0.5 \\ \textbf{16.93 $\pm$ 0.12}\end{tabular}  & \begin{tabular}[c]{@{}c@{}}36.17 $\pm$ 0.45 \\  \textbf{14.03 $\pm$ 0.24}\end{tabular} \\ \hline
FGSM       & \begin{tabular}[c]{@{}c@{}}71.92 $\pm$ 0.33 \\  \textbf{52.83 $\pm$ 0.37}\end{tabular} & \begin{tabular}[c]{@{}c@{}}67.34 $\pm$ 0.36 \\ \textbf{39.83 $\pm$ 0.31}\end{tabular}  & \begin{tabular}[c]{@{}c@{}}64.72 $\pm$ 1.12 \\ 0.0 $\pm$ 0.0\end{tabular}    & \begin{tabular}[c]{@{}c@{}}56.87 $\pm$ 1.24 \\ 0.03 $\pm$ 0.05\end{tabular}  & \begin{tabular}[c]{@{}c@{}}52.31 $\pm$ 2.11 \\ 0.0 $\pm$ 0.0\end{tabular}    & \begin{tabular}[c]{@{}c@{}}48.99 $\pm$ 1.17\\  0.0 $\pm$ 0.0\end{tabular}    & \begin{tabular}[c]{@{}c@{}}44.27 $\pm$ 1.4 \\ 0.0 $\pm$ 0.0\end{tabular}     & \begin{tabular}[c]{@{}c@{}}42.05 $\pm$ 1.03 \\  0.0 $\pm$ 0.0\end{tabular}    \\ \hline
RS-FGSM    & \begin{tabular}[c]{@{}c@{}}72.65 $\pm$ 0.28 \\  51.63 $\pm$ 0.52\end{tabular} & \begin{tabular}[c]{@{}c@{}}68.26 $\pm$ 0.2 \\ 39.57 $\pm$ 0.09\end{tabular}   & \begin{tabular}[c]{@{}c@{}}65.58 $\pm$ 0.69 \\ 26.63 $\pm$ 2.8\end{tabular}  & \begin{tabular}[c]{@{}c@{}}54.25 $\pm$ 5.85 \\ 0.0 $\pm$ 0.0\end{tabular}    & \begin{tabular}[c]{@{}c@{}}46.08 $\pm$ 4.87 \\ 0.0 $\pm$ 0.0\end{tabular}    & \begin{tabular}[c]{@{}c@{}}35.84 $\pm$ 0.17\\  0.0 $\pm$ 0.0\end{tabular}    & \begin{tabular}[c]{@{}c@{}}24.4 $\pm$ 1.25\\  0.0 $\pm$ 0.0\end{tabular}     & \begin{tabular}[c]{@{}c@{}}21.37 $\pm$ 5.04 \\ 0.0 $\pm$ 0.0\end{tabular}     \\ \hline
RandAlpha  & \begin{tabular}[c]{@{}c@{}}73.9 $\pm$ 0.15 \\ 49.13 $\pm$ 0.91\end{tabular}   & \begin{tabular}[c]{@{}c@{}}71.17 $\pm$ 0.12\\  34.3 $\pm$ 0.54\end{tabular}   & \begin{tabular}[c]{@{}c@{}}68.65 $\pm$ 0.22 \\ 25.5 $\pm$ 0.33\end{tabular}  & \begin{tabular}[c]{@{}c@{}}66.42 $\pm$ 0.13 \\ 20.27 $\pm$ 0.98\end{tabular} & \begin{tabular}[c]{@{}c@{}}64.05 $\pm$ 0.5 \\ 16.3 $\pm$ 0.14\end{tabular}   & \begin{tabular}[c]{@{}c@{}}61.99 $\pm$ 0.6 \\ 12.4 $\pm$ 0.29\end{tabular}   & \begin{tabular}[c]{@{}c@{}}59.74 $\pm$ 0.57 \\  6.93 $\pm$ 0.19\end{tabular} & \begin{tabular}[c]{@{}c@{}}58.9 $\pm$ 0.78\\  3.63 $\pm$ 0.12\end{tabular}    \\ \hline
AT Free    & \begin{tabular}[c]{@{}c@{}}67.62 $\pm$ 0.24 \\ 48.07 $\pm$ 0.31\end{tabular}  & \begin{tabular}[c]{@{}c@{}}63.27 $\pm$ 0.72 \\ 37.93 $\pm$ 0.69\end{tabular}  & \begin{tabular}[c]{@{}c@{}}59.53 $\pm$ 0.31 \\ 29.7 $\pm$ 0.51\end{tabular}  & \begin{tabular}[c]{@{}c@{}}55.77 $\pm$ 0.28 \\ 24.43 $\pm$ 0.37\end{tabular} & \begin{tabular}[c]{@{}c@{}}47.02 $\pm$ 3.83 \\ 3.23 $\pm$ 4.43\end{tabular}  & \begin{tabular}[c]{@{}c@{}}33.52 $\pm$ 9.24 \\ 0.0 $\pm$ 0.0\end{tabular}    & \begin{tabular}[c]{@{}c@{}}7.87 $\pm$ 1.78 \\ 0.0 $\pm$ 0.0\end{tabular}     & \begin{tabular}[c]{@{}c@{}}20.92 $\pm$ 21.48 \\ 0.0 $\pm$ 0.0\end{tabular}    \\ \hline
ZeroGrad   & \begin{tabular}[c]{@{}c@{}}71.68 $\pm$ 0.07 \\ \textbf{52.63 $\pm$ 0.61}\end{tabular}  & \begin{tabular}[c]{@{}c@{}}67.2 $\pm$ 0.14 \\ 39.57 $\pm$ 0.33\end{tabular}   & \begin{tabular}[c]{@{}c@{}}63.69 $\pm$ 0.14 \\ 30.27 $\pm$ 0.54\end{tabular} & \begin{tabular}[c]{@{}c@{}}60.77 $\pm$ 0.26 \\  23.7 $\pm$ 0.08\end{tabular} & \begin{tabular}[c]{@{}c@{}}61.05 $\pm$ 0.38 \\ 15.1 $\pm$ 0.49\end{tabular}  & \begin{tabular}[c]{@{}c@{}}58.39 $\pm$ 0.16 \\ 11.13 $\pm$ 0.68\end{tabular} & \begin{tabular}[c]{@{}c@{}}56.19 $\pm$ 0.11 \\ 8.8 $\pm$ 0.36\end{tabular}   & \begin{tabular}[c]{@{}c@{}}56.38 $\pm$ 0.18 \\ 4.9 $\pm$ 0.36\end{tabular}    \\ \hline
MultiGrad  & \begin{tabular}[c]{@{}c@{}}71.8 $\pm$ 0.15 \\ 51.9 $\pm$ 0.29\end{tabular}    & \begin{tabular}[c]{@{}c@{}}67.73 $\pm$ 0.48 \\ 39.7 $\pm$ 0.37\end{tabular}   & \begin{tabular}[c]{@{}c@{}}63.24 $\pm$ 0.33\\  31.5 $\pm$ 0.62\end{tabular}  & \begin{tabular}[c]{@{}c@{}}60.05 $\pm$ 0.79 \\ 26.03 $\pm$ 0.09\end{tabular} & \begin{tabular}[c]{@{}c@{}}56.39 $\pm$ 0.49 \\ 20.8 $\pm$ 0.29\end{tabular}  & \begin{tabular}[c]{@{}c@{}}56.79 $\pm$ 8.27 \\ 0.0 $\pm$ 0.0\end{tabular}    & \begin{tabular}[c]{@{}c@{}}59.8 $\pm$ 3.77 \\ 0.0 $\pm$ 0.0\end{tabular}     & \begin{tabular}[c]{@{}c@{}}52.96 $\pm$ 5.58 \\ 0.0 $\pm$ 0.0\end{tabular}     \\ \hline \hline
PGD-2      & \begin{tabular}[c]{@{}c@{}}71.62 $\pm$ 0.15 \\ 51.73 $\pm$ 0.48\end{tabular}  & \begin{tabular}[c]{@{}c@{}}67.25 $\pm$ 0.43 \\ \textbf{40.27 $\pm$ 0.7}\end{tabular}   & \begin{tabular}[c]{@{}c@{}}63.18 $\pm$ 0.36 \\ 32.23 $\pm$ 0.19\end{tabular} & \begin{tabular}[c]{@{}c@{}}59.02 $\pm$ 0.4 \\ 27.13 $\pm$ 0.37\end{tabular}  & \begin{tabular}[c]{@{}c@{}}54.47 $\pm$ 0.45 \\ 23.43 $\pm$ 0.31\end{tabular} & \begin{tabular}[c]{@{}c@{}}50.91 $\pm$ 0.35 \\ 20.23 $\pm$ 0.39\end{tabular} & \begin{tabular}[c]{@{}c@{}}41.03 $\pm$ 3.18 \\  0.03 $\pm$ 0.05\end{tabular} & \begin{tabular}[c]{@{}c@{}}40.13 $\pm$ 3.66 \\ 0.0 $\pm$ 0.0\end{tabular}     \\ \hline
PGD-10     & \begin{tabular}[c]{@{}c@{}}71.11 $\pm$ 0.62\\  \textbf{52.5 $\pm$ 0.59}\end{tabular}   & \begin{tabular}[c]{@{}c@{}}66.9 $\pm$ 0.57 \\ \textbf{40.73 $\pm$ 0.56}\end{tabular}   & \begin{tabular}[c]{@{}c@{}}62.05 $\pm$ 0.47 \\ \textbf{32.8 $\pm$ 0.29}\end{tabular}  & \begin{tabular}[c]{@{}c@{}}57.64 $\pm$ 0.81 \\ \textbf{27.97 $\pm$ 0.59}\end{tabular} & \begin{tabular}[c]{@{}c@{}}52.84 $\pm$ 0.88 \\ \textbf{24.7 $\pm$ 0.36}\end{tabular}  & \begin{tabular}[c]{@{}c@{}}48.14 $\pm$ 0.73 \\ \textbf{21.8 $\pm$ 0.57}\end{tabular}  & \begin{tabular}[c]{@{}c@{}}43.14 $\pm$ 0.87 \\  \textbf{18.87 $\pm$ 0.6}\end{tabular} & \begin{tabular}[c]{@{}c@{}}39.2 $\pm$ 0.62 \\\textbf{ 16.8 $\pm$ 0.57}\end{tabular}    \\ \hline
\end{tabular}
}
\end{table}

\begin{table}[h]
\centering
\renewcommand{\arraystretch}{1.5}
\scriptsize{
\begin{tabular}{c|c|c|c|c|c|c}
\multicolumn{4}{c}{\small{\textbf{WideResNet28-10 -- SVHN Dataset}}} \\
\multicolumn{2}{c}{} \\
\toprule
           & $\epsilon$ = $\nicefrac{2}{255}$                                                                 & $\epsilon$ = $\nicefrac{4}{255}$                                                                 & $\epsilon$ = $\nicefrac{6}{255}$                                                       & $\epsilon$ = $\nicefrac{8}{255}$                                                               & $\epsilon$ = $\nicefrac{10}{255}$                                                                & $\epsilon$ = $\nicefrac{12}{255}$                                                                \\ \hline
           \midrule
\textbf{N-FGSM}     & \begin{tabular}[c]{@{}c@{}}95.64 $\pm$ 0.09 \\  \textbf{84.1 $\pm$ 0.73}\end{tabular}  & \begin{tabular}[c]{@{}c@{}}93.66 $\pm$ 0.41 \\  66.9 $\pm$ 0.86\end{tabular}  & \begin{tabular}[c]{@{}c@{}}91.77 $\pm$ 0.42 \\ \textbf{ 53.0 $\pm$ 0.36}\end{tabular}   & \begin{tabular}[c]{@{}c@{}}88.89 $\pm$ 0.58 \\\textbf{ 40.5 $\pm$ 0.37}\end{tabular}  & \begin{tabular}[c]{@{}c@{}}88.07 $\pm$ 0.59 \\ \textbf{30.47 $\pm$ 0.76}\end{tabular}  & \begin{tabular}[c]{@{}c@{}}87.52 $\pm$ 0.49 \\ \textbf{22.43 $\pm$ 0.53}\end{tabular}  \\ \hline
Grad Align & \begin{tabular}[c]{@{}c@{}}95.41 $\pm$ 0.06\\ \textbf{ 84.57 $\pm$ 0.56}\end{tabular}  & \begin{tabular}[c]{@{}c@{}}93.9 $\pm$ 0.48 \\ 67.27 $\pm$ 0.54\end{tabular}   & \begin{tabular}[c]{@{}c@{}}68.36 $\pm$ 34.49 \\ 39.53 $\pm$ 14.89\end{tabular} & \begin{tabular}[c]{@{}c@{}}42.62 $\pm$ 32.73 \\ 24.7 $\pm$ 9.34\end{tabular} & \begin{tabular}[c]{@{}c@{}}19.3 $\pm$ 0.21 \\ 17.63 $\pm$ 0.62\end{tabular}   & \begin{tabular}[c]{@{}c@{}}19.53 $\pm$ 0.08 \\ 18.13 $\pm$ 0.52\end{tabular}  \\ \hline
FGSM       & \begin{tabular}[c]{@{}c@{}}95.83 $\pm$ 0.1 \\ \textbf{85.03 $\pm$ 0.37}\end{tabular}   & \begin{tabular}[c]{@{}c@{}}95.0 $\pm$ 0.24 \\  31.53 $\pm$ 6.57\end{tabular}  & \begin{tabular}[c]{@{}c@{}}94.23 $\pm$ 0.79 \\  1.7 $\pm$ 1.36\end{tabular}    & \begin{tabular}[c]{@{}c@{}}91.11 $\pm$ 1.36 \\  0.13 $\pm$ 0.19\end{tabular} & \begin{tabular}[c]{@{}c@{}}88.83 $\pm$ 1.71 \\ 0.0 $\pm$ 0.0\end{tabular}     & \begin{tabular}[c]{@{}c@{}}86.74 $\pm$ 0.7 \\  0.0 $\pm$ 0.0\end{tabular}     \\ \hline
RS-FGSM    & \begin{tabular}[c]{@{}c@{}}95.81 $\pm$ 0.25 \\ 83.8 $\pm$ 0.43\end{tabular}   & \begin{tabular}[c]{@{}c@{}}94.53 $\pm$ 0.4 \\ 66.67 $\pm$ 0.65\end{tabular}   & \begin{tabular}[c]{@{}c@{}}95.23 $\pm$ 0.26 \\ 0.53 $\pm$ 0.26\end{tabular}    & \begin{tabular}[c]{@{}c@{}}94.68 $\pm$ 0.62 \\  0.0 $\pm$ 0.0\end{tabular}   & \begin{tabular}[c]{@{}c@{}}93.9 $\pm$ 0.52 \\ 0.0 $\pm$ 0.0\end{tabular}      & \begin{tabular}[c]{@{}c@{}}91.64 $\pm$ 2.98 \\ 0.0 $\pm$ 0.0\end{tabular}     \\ \hline
RandAlpha  & \begin{tabular}[c]{@{}c@{}}96.02 $\pm$ 0.23 \\ 82.5 $\pm$ 0.45\end{tabular}   & \begin{tabular}[c]{@{}c@{}}95.47 $\pm$ 0.18 \\ 63.33 $\pm$ 0.53\end{tabular}  & \begin{tabular}[c]{@{}c@{}}94.69 $\pm$ 0.26 \\  47.7 $\pm$ 0.99\end{tabular}   & \begin{tabular}[c]{@{}c@{}}93.72 $\pm$ 0.44 \\ 35.73 $\pm$ 0.34\end{tabular} & \begin{tabular}[c]{@{}c@{}}93.08 $\pm$ 1.45 \\  23.17 $\pm$ 1.97\end{tabular} & \begin{tabular}[c]{@{}c@{}}93.96 $\pm$ 0.68 \\ 11.1 $\pm$ 3.05\end{tabular}   \\ \hline
AT Free    & \begin{tabular}[c]{@{}c@{}}94.85 $\pm$ 0.39 \\  83.13 $\pm$ 0.17\end{tabular} & \begin{tabular}[c]{@{}c@{}}92.95 $\pm$ 0.65 \\ \textbf{68.67 $\pm$ 0.53}\end{tabular}  & \begin{tabular}[c]{@{}c@{}}91.62 $\pm$ 1.93 \\ \textbf{54.93 $\pm$ 2.58}\end{tabular}   & \begin{tabular}[c]{@{}c@{}}93.74 $\pm$ 0.69\\ 0.03 $\pm$ 0.05\end{tabular} & \begin{tabular}[c]{@{}c@{}}92.47 $\pm$ 0.97 \\ 0.0 $\pm$ 0.0\end{tabular}     & \begin{tabular}[c]{@{}c@{}}90.5 $\pm$ 1.41 \\ 0.0 $\pm$ 0.0\end{tabular}      \\ \hline
ZeroGrad   & \begin{tabular}[c]{@{}c@{}}95.78 $\pm$ 0.21 \\  \textbf{84.47 $\pm$ 0.83}\end{tabular} & \begin{tabular}[c]{@{}c@{}}94.06 $\pm$ 0.52 \\ 66.1 $\pm$ 0.37\end{tabular}   & \begin{tabular}[c]{@{}c@{}}92.13 $\pm$ 0.98\\   47.3 $\pm$ 0.62\end{tabular}   & \begin{tabular}[c]{@{}c@{}}91.04 $\pm$ 0.4 \\ 29.33 $\pm$ 0.56\end{tabular}  & \begin{tabular}[c]{@{}c@{}}88.85 $\pm$ 0.92 \\  20.77 $\pm$ 0.63\end{tabular} & \begin{tabular}[c]{@{}c@{}}89.8 $\pm$ 1.36 \\ 9.33 $\pm$ 0.76\end{tabular}    \\ \hline
MultiGrad  & \begin{tabular}[c]{@{}c@{}}95.63 $\pm$ 0.16 \\  \textbf{84.37 $\pm$ 0.59}\end{tabular} & \begin{tabular}[c]{@{}c@{}}94.27 $\pm$ 0.38 \\ 67.27 $\pm$ 0.31\end{tabular}  & \begin{tabular}[c]{@{}c@{}}93.64 $\pm$ 1.21 \\ 50.1 $\pm$ 0.9\end{tabular}     & \begin{tabular}[c]{@{}c@{}}94.83 $\pm$ 1.55 \\  1.77 $\pm$ 1.72\end{tabular} & \begin{tabular}[c]{@{}c@{}}95.26 $\pm$ 0.34 \\ 0.0 $\pm$ 0.0\end{tabular}     & \begin{tabular}[c]{@{}c@{}}95.22 $\pm$ 0.15 \\ 0.0 $\pm$ 0.0\end{tabular}     \\ \hline \hline
PGD-2      & \begin{tabular}[c]{@{}c@{}}95.88 $\pm$ 0.35 \\\textbf{ 86.25 $\pm$ 0.7}\end{tabular}   & \begin{tabular}[c]{@{}c@{}}94.66 $\pm$ 0.1 \\ 73.29 $\pm$ 0.25\end{tabular}   & \begin{tabular}[c]{@{}c@{}}93.77 $\pm$ 0.61 \\ 60.53 $\pm$ 0.72\end{tabular}   & \begin{tabular}[c]{@{}c@{}}92.99 $\pm$ 1.11 \\ 40.77 $\pm$ 4.39\end{tabular} & \begin{tabular}[c]{@{}c@{}}88.81 $\pm$ 0.93 \\ 34.33 $\pm$ 2.76\end{tabular}  & \begin{tabular}[c]{@{}c@{}}83.17 $\pm$ 4.78 \\ 26.8 $\pm$ 3.31\end{tabular}   \\ \hline
PGD-10     & \begin{tabular}[c]{@{}c@{}}95.92 $\pm$ 0.08 \\ \textbf{86.94 $\pm$ 0.13}\end{tabular}  & \begin{tabular}[c]{@{}c@{}}94.36 $\pm$ 0.13 \\  \textbf{74.46 $\pm$ 0.54}\end{tabular} & \begin{tabular}[c]{@{}c@{}}92.46 $\pm$ 0.25 \\ \textbf{63.87 $\pm$ 0.49}\end{tabular}   & \begin{tabular}[c]{@{}c@{}}89.67 $\pm$ 0.34 \\ \textbf{53.95 $\pm$ 0.55}\end{tabular} & \begin{tabular}[c]{@{}c@{}}85.98 $\pm$ 0.59 \\ \textbf{44.59 $\pm$ 0.14}\end{tabular}  & \begin{tabular}[c]{@{}c@{}}80.08 $\pm$ 0.93 \\  \textbf{37.64 $\pm$ 0.49}\end{tabular} \\ \hline
\end{tabular}
}
\end{table}

\newpage


\end{document}

%% file: Tables/tab_ablation_singlestep.tex
\begin{table}[ht]
\renewcommand{\arraystretch}{1.2}
\caption{Ablation of the PGD-50-10 accuracy for single-step methods when increasing the $\epsilon_{\textrm{train}}$. All models are evaluated with PGD-50-10 attack and $\epsilon_{\textrm{test}}=\nicefrac{8}{255}$. 
Note that considering the trade-off between clean and robust accuracy, all methods perform best when training with the same epsilon to be applied at test time.}
\label{table:increased_epsilon}
\vspace{2pt}
\centering
\small
\begin{tabular}{cccccccc}
\toprule
          & \multicolumn{2}{c}{$\epsilon_{\textrm{train}} = 1 \epsilon_{\textrm{test}}$}  & \multicolumn{2}{c}{$\epsilon_{\textrm{train}} = 1.5 \epsilon_{\textrm{test}}$} & \multicolumn{2}{c}{$\epsilon_{\textrm{train}} = 2 \epsilon_{\textrm{test}}$}                                                               \\
\cmidrule(r){2-3}
\cmidrule(r){4-5}
\cmidrule(r){6-7} 
\textbf{Method} & \textbf{Clean acc.} & \textbf{PGD acc.} & \textbf{Clean acc}. & \textbf{PGD acc.} & \textbf{Clean acc.} & \textbf{PGD acc.} & \textbf{Rel. Cost} \\
\midrule
GradAlign & 81.9 $\pm$ 0.22& 48.14 $\pm$ 0.15    & 73.29 $\pm$ 0.23  & 50.6 $\pm$ 0.45  & 61.3 $\pm$ 0.15  & 46.67 $\pm$ 0.29 & 3 \\ \hline
MultiGrad  & 82.33 $\pm$ 0.14  & 47.29 $\pm$ 0.07 & 75.28 $\pm$ 0.2& 50.0 $\pm$ 0.79     & 71.42 $\pm$ 5.63& 0.0 $\pm$ 0.0 & 2  \\ \hline 
AT Free    & 78.41 $\pm$ 0.18  & 46.03 $\pm$ 0.36 & 73.91 $\pm$ 4.19  & 32.4 $\pm$ 22.91 & 71.64 $\pm$ 3.89  & 0.0 $\pm$ 0.0  & 1.6 \\ \hline
Kim et. al. & 89.02 $\pm$ 0.1  & 33.01 $\pm$ 0.09  & 88.35 $\pm$ 0.31& 27.36$\pm$0.31     & 90.45 $\pm$ 0.08& 9.28 $\pm$ 0.12 & 1.5 \\ \hline
\midrule
FGSM       & 86.41 $\pm$ 0.7  & 0.0 $\pm$ 0.0     & 80.6 $\pm$ 2.59 &  0.0 $\pm$ 0.0     & 77.14 $\pm$ 2.46 &  0.0 $\pm$ 0.0  & 1 \\ \hline
RS-FGSM    & 84.05 $\pm$ 0.13  & 46.08 $\pm$ 0.18 & 65.22 $\pm$ 23.23  & 0.0 $\pm$ 0.0   & 76.66 $\pm$ 0.38 &  0.0 $\pm$ 0.0  & 1  \\ \hline
ZeroGrad   & 82.62 $\pm$ 0.05  & 47.08 $\pm$ 0.1  & 78.11 $\pm$ 0.2& 46.43 $\pm$ 0.37    & 75.42 $\pm$ 0.13& 45.63 $\pm$ 0.39  & 1 \\ \hline
N-FGSM & 80.58 $\pm$ 0.22  & 48.12 $\pm$ 0.07 & 71.46 $\pm$ 0.14 & 50.23 $\pm$ 0.31  & 63.18 $\pm$ 0.49 & 46.46 $\pm$ 0.1  & 1 \\ \hline
\end{tabular}

\vskip -0.1in

\end{table}

%% file: main.bbl
\begin{thebibliography}{37}
\providecommand{\natexlab}[1]{#1}
\providecommand{\url}[1]{\texttt{#1}}
\expandafter\ifx\csname urlstyle\endcsname\relax
  \providecommand{\doi}[1]{doi: #1}\else
  \providecommand{\doi}{doi: \begingroup \urlstyle{rm}\Url}\fi

\bibitem[Andriushchenko and Flammarion(2020)]{grad_align}
Maksym Andriushchenko and Nicolas Flammarion.
\newblock Understanding and improving fast adversarial training.
\newblock In \emph{{Neural Information Processing Systems (NeurIPS)}}, 2020.

\bibitem[Biggio and Roli(2018)]{biggio2018wild}
Battista Biggio and Fabio Roli.
\newblock Wild patterns: Ten years after the rise of adversarial machine
  learning.
\newblock \emph{{Pattern Recognition}}, 2018.

\bibitem[Bishop(1995)]{bishop1995training}
Chris~M Bishop.
\newblock Training with noise is equivalent to tikhonov regularization.
\newblock \emph{Neural computation}, 7\penalty0 (1):\penalty0 108--116, 1995.

\bibitem[Boloor et~al.(2019)Boloor, He, Gill, Vorobeychik, and
  Zhang]{DBLP:journals/corr/abs-1903-05157}
Adith Boloor, Xin He, Christopher~D. Gill, Yevgeniy Vorobeychik, and Xuan
  Zhang.
\newblock Simple physical adversarial examples against end-to-end autonomous
  driving models.
\newblock \emph{arxiv:1903.05157}, 2019.

\bibitem[Cisse et~al.(2017)Cisse, Bojanowski, Grave, Dauphin, and
  Usunier]{parseval}
Moustapha Cisse, Piotr Bojanowski, Edouard Grave, Yann Dauphin, and Nicolas
  Usunier.
\newblock Parseval networks: Improving robustness to adversarial examples.
\newblock In \emph{{International Conference on Machine Learning (ICML)}},
  2017.

\bibitem[Croce and Hein(2020)]{croce2020reliable}
Francesco Croce and Matthias Hein.
\newblock Reliable evaluation of adversarial robustness with an ensemble of
  diverse parameter-free attacks.
\newblock In \emph{{International Conference on Machine Learning (ICML)}},
  2020.

\bibitem[Devlin et~al.(2019)Devlin, Changm, Lee, and Toutanova]{bert}
Jacob Devlin, Ming{-}Wei Changm, Kenton Lee, and Kristina Toutanova.
\newblock {BERT:} pre-training of deep bidirectional transformers for language
  understanding.
\newblock In \emph{{Annual Conference of the North American Chapter of the
  Association for Computational Linguistics: Human Language Technologies (NAACL
  HLT)}}, 2019.

\bibitem[Fawzi et~al.(2018)Fawzi, Moosavi-Dezfooli, Frossard, and
  Soatto]{fawzi2018empirical}
Alhussein Fawzi, Seyed-Mohsen Moosavi-Dezfooli, Pascal Frossard, and Stefano
  Soatto.
\newblock Empirical study of the topology and geometry of deep networks.
\newblock In \emph{{IEEE} Conference on Computer Vision and Pattern Recognition
  ({CVPR})}, 2018.

\bibitem[Gilmer et~al.(2019)Gilmer, Ford, Carlini, and
  Cubuk]{gilmer2019adversarial}
Justin Gilmer, Nicolas Ford, Nicholas Carlini, and Ekin Cubuk.
\newblock Adversarial examples are a natural consequence of test error in
  noise.
\newblock In \emph{{International Conference on Machine Learning (ICML)}},
  2019.

\bibitem[Golgooni et~al.(2021)Golgooni, Saberi, Eskandar, and
  Rohban]{zero_grad}
Zeinab Golgooni, Mehrdad Saberi, Masih Eskandar, and Mohammad~Hossein Rohban.
\newblock Zerograd: Mitigating and explaining catastrophic overfitting in fgsm
  adversarial training.
\newblock \emph{arXiv:2103.15476}, 2021.

\bibitem[Goodfellow et~al.(2015)Goodfellow, Shlens, and Szegedy]{fgsm}
Ian Goodfellow, Jonathon Shlens, and Christian Szegedy.
\newblock Explaining and harnessing adversarial examples.
\newblock \emph{{International Conference on Learning Representations (ICLR)}},
  2015.

\bibitem[He et~al.(2015)He, Zhang, Ren, and Sun]{he2015delving}
Kaiming He, Xiangyu Zhang, Shaoqing Ren, and Jian Sun.
\newblock Delving deep into rectifiers: Surpassing human-level performance on
  imagenet classification.
\newblock In \emph{{IEEE} International Conference on Computer Vision
  ({ICCV})}, 2015.

\bibitem[He et~al.(2016)He, Zhang, Ren, and Sun]{preact}
Kaiming He, Xiangyu Zhang, Shaoqing Ren, and Jian Sun.
\newblock Identity mappings in deep residual networks.
\newblock In \emph{European Conference on Computer Vision ({ECCV})}, 2016.

\bibitem[Kang and Moosavi-Dezfooli(2021)]{kang2021understanding}
Peilin Kang and Seyed-Mohsen Moosavi-Dezfooli.
\newblock Understanding catastrophic overfitting in adversarial training.
\newblock \emph{arXiv:2105.02942}, 2021.

\bibitem[Kim et~al.(2021)Kim, Lee, and Lee]{AAAI}
Hoki Kim, Woojin Lee, and Jaewook Lee.
\newblock Understanding catastrophic overfitting in single-step adversarial
  training.
\newblock In \emph{{AAAI Conference on Artificial Intelligence (AAAI)}}, 2021.

\bibitem[Krizhevsky and Hinton(2009)]{cifar}
Alex Krizhevsky and Geoffrey Hinton.
\newblock Learning multiple layers of features from tiny images.
\newblock \emph{Master's thesis, Department of Computer Science, University of
  Toronto}, 2009.

\bibitem[Krizhevsky et~al.(2012)Krizhevsky, Sutskever, and
  Hinton]{krizhevsky2012imagenet}
Alex Krizhevsky, Ilya Sutskever, and Geoffrey~E Hinton.
\newblock Imagenet classification with deep convolutional neural networks.
\newblock \emph{{Neural Information Processing Systems (NeurIPS)}}, 2012.

\bibitem[Kurakin et~al.(2017)Kurakin, Goodfellow, and Bengio]{adv_ml_scale}
Alexey Kurakin, Ian Goodfellow, and Samy Bengio.
\newblock Adversarial machine learning at scale.
\newblock In \emph{{International Conference on Learning Representations
  (ICLR)}}, 2017.

\bibitem[Li et~al.(2020)Li, Wang, Jana, and Carin]{towards}
Bai Li, Shiqi Wang, Suman Jana, and Lawrence Carin.
\newblock Towards understanding fast adversarial training.
\newblock \emph{arXiv:2006.03089}, 2020.

\bibitem[Madry et~al.(2018)Madry, Makelov, Schmidt, Tsipras, and Vladu]{PGD}
Aleksander Madry, Aleksandar Makelov, Ludwig Schmidt, Dimitris Tsipras, and
  Adrian Vladu.
\newblock Towards deep learning models resistant to adversarial attacks.
\newblock In \emph{{International Conference on Learning Representations
  (ICLR)}}, 2018.

\bibitem[Netzer et~al.(2011)Netzer, Wang, Coates, Bissacco, Wu, and Ng]{svhn}
Yuval Netzer, Tao Wang, Adam Coates, Alessandro Bissacco, Bo~Wu, and Andrew~Y
  Ng.
\newblock Reading digits in natural images with unsupervised feature learning.
\newblock In \emph{{Neural Information Processing Systems (NeurIPS),
  Workshops}}, 2011.

\bibitem[Papernot et~al.(2016)Papernot, McDaniel, Wu, Jha, and
  Swami]{papernot2016distillation}
Nicolas Papernot, Patrick McDaniel, Xi~Wu, Somesh Jha, and Ananthram Swami.
\newblock Distillation as a defense to adversarial perturbations against deep
  neural networks.
\newblock In \emph{IEEE symposium on security and privacy (SP)}, 2016.

\bibitem[Park and Lee(2021)]{SLAT}
Geon~Yeong Park and Sang~Wan Lee.
\newblock Reliably fast adversarial training via latent adversarial
  perturbation.
\newblock In \emph{{International Conference on Learning Representations
  (ICLR), Workshops}}, 2021.

\bibitem[Rice et~al.(2020)Rice, Wong, and Kolter]{overfitting}
Leslie Rice, Eric Wong, and Zico Kolter.
\newblock Overfitting in adversarially robust deep learning.
\newblock In \emph{{International Conference on Machine Learning (ICML)}},
  2020.

\bibitem[Sanyal et~al.(2018)Sanyal, Kanade, and
  Torr]{DBLP:journals/corr/abs-1804-07090}
Amartya Sanyal, Varun Kanade, and Philip H.~S. Torr.
\newblock Robustness via deep low-rank representations.
\newblock \emph{arxiv:1804.07090}, 2018.

\bibitem[Shafahi et~al.(2019)Shafahi, Najibi, Ghiasi, Xu, Dickerson, Studer,
  Davis, Taylor, and Goldstein]{free}
Ali Shafahi, Mahyar Najibi, Mohammad~Amin Ghiasi, Zheng Xu, John Dickerson,
  Christoph Studer, Larry~S Davis, Gavin Taylor, and Tom Goldstein.
\newblock Adversarial training for free!
\newblock \emph{{Neural Information Processing Systems (NeurIPS)}}, 2019.

\bibitem[Silver et~al.(2016)Silver, Huang, Maddison, Guez, Sifre, Van
  Den~Driessche, Schrittwieser, Antonoglou, Panneershelvam, Lanctot,
  et~al.]{go}
David Silver, Aja Huang, Chris~J Maddison, Arthur Guez, Laurent Sifre, George
  Van Den~Driessche, Julian Schrittwieser, Ioannis Antonoglou, Veda
  Panneershelvam, Marc Lanctot, et~al.
\newblock Mastering the game of go with deep neural networks and tree search.
\newblock \emph{Nature}, 2016.

\bibitem[Simonyan and Zisserman(2015)]{vgg}
Karen Simonyan and Andrew Zisserman.
\newblock Very deep convolutional networks for large-scale image recognition.
\newblock In \emph{International Conference on Learning Representations}, 2015.

\bibitem[Sriramanan et~al.(2020)Sriramanan, Addepalli, Baburaj, et~al.]{GAT}
Gaurang Sriramanan, Sravanti Addepalli, Arya Baburaj, et~al.
\newblock Guided adversarial attack for evaluating and enhancing adversarial
  defenses.
\newblock \emph{Advances in Neural Information Processing Systems},
  33:\penalty0 20297--20308, 2020.

\bibitem[Sriramanan et~al.(2021)Sriramanan, Addepalli, Baburaj, et~al.]{NuAT}
Gaurang Sriramanan, Sravanti Addepalli, Arya Baburaj, et~al.
\newblock Towards efficient and effective adversarial training.
\newblock \emph{Advances in Neural Information Processing Systems},
  34:\penalty0 11821--11833, 2021.

\bibitem[Szegedy et~al.(2014)Szegedy, Zaremba, Sutskever, Bruna, Erhan,
  Goodfellow, and Fergus]{szegedy2013intriguing}
Christian Szegedy, Wojciech Zaremba, Ilya Sutskever, Joan Bruna, Dumitru Erhan,
  Ian Goodfellow, and Rob Fergus.
\newblock Intriguing properties of neural networks.
\newblock In \emph{{International Conference on Learning Representations
  (ICLR)}}, 2014.

\bibitem[Tramèr et~al.(2018)Tramèr, Kurakin, Papernot, Goodfellow, Boneh, and
  McDaniel]{tramer2018ensemble}
Florian Tramèr, Alexey Kurakin, Nicolas Papernot, Ian Goodfellow, Dan Boneh,
  and Patrick McDaniel.
\newblock Ensemble adversarial training: Attacks and defenses.
\newblock In \emph{{International Conference on Learning Representations
  (ICLR)}}, 2018.

\bibitem[Vivek and Babu(2020)]{dropout_fgsm}
BS~Vivek and R~Venkatesh Babu.
\newblock Single-step adversarial training with dropout scheduling.
\newblock In \emph{{IEEE} Conference on Computer Vision and Pattern Recognition
  ({CVPR})}, 2020.

\bibitem[Weng et~al.(2018)Weng, Zhang, Chen, Song, Hsieh, Daniel, Boning, and
  Dhillon]{weng2018towards}
Lily Weng, Huan Zhang, Hongge Chen, Zhao Song, Cho-Jui Hsieh, Luca Daniel,
  Duane Boning, and Inderjit Dhillon.
\newblock Towards fast computation of certified robustness for relu networks.
\newblock In \emph{{International Conference on Machine Learning (ICML)}},
  2018.

\bibitem[Wong et~al.(2020)Wong, Rice, and Kolter]{RS-FGSM}
Eric Wong, Leslie Rice, and J.~Zico Kolter.
\newblock Fast is better than free: Revisiting adversarial training.
\newblock In \emph{{International Conference on Learning Representations
  (ICLR)}}, 2020.

\bibitem[Zagoruyko and Komodakis(2016)]{zagoruyko2016wide}
Sergey Zagoruyko and Nikos Komodakis.
\newblock Wide residual networks.
\newblock In \emph{{BMVC British Machine Vision Conference (BMVC)}}, 2016.

\bibitem[Zhang et~al.(2019)Zhang, Yu, Jiao, Xing, Ghaoui, and Jordan]{TRADES}
Hongyang Zhang, Yaodong Yu, Jiantao Jiao, Eric Xing, Laurent~El Ghaoui, and
  Michael Jordan.
\newblock Theoretically principled trade-off between robustness and accuracy.
\newblock In \emph{{International Conference on Machine Learning (ICML)}},
  2019.

\end{thebibliography}
